\documentclass[english]{article}
\usepackage{geometry}

\usepackage[T1]{fontenc}
\usepackage[latin9]{inputenc}
\usepackage{bm}
\usepackage{amsmath,mathtools}
\usepackage{amssymb}
\usepackage[unicode=true,
 bookmarks=false,
 breaklinks=false,pdfborder={0 0 1},colorlinks=false]
 {hyperref}
\hypersetup{
 colorlinks,citecolor=blue,filecolor=blue,linkcolor=blue,urlcolor=blue}

\makeatletter
\usepackage{amsthm}
\usepackage{comment}
\usepackage{natbib}
\usepackage{booktabs}

\usepackage{graphicx}

\usepackage[linesnumbered,ruled,vlined]{algorithm2e}
\usepackage{algorithmic}

\SetCommentSty{mycommfont}

\usepackage{float}
\usepackage{multirow}
 
\usepackage{dsfont}
\usepackage{tcolorbox}

\usepackage{color}
\definecolor{yxc}{RGB}{255,0,0}
\definecolor{yjc}{RGB}{125,0,0}
\definecolor{ytw}{RGB}{255,69,0}
\definecolor{gen}{RGB}{0,0,200}

\allowdisplaybreaks

\usepackage{makecell}

\title{Minimax-Optimal Multi-Agent Robust Reinforcement Learning}

\author{Yuchen Jiao \thanks{School of Information and Electronics, Beijing Institute of Technology, Beijing, China; Email: \texttt{jiaoyc23@bit.edu.cn}.}\and Gen Li \thanks{Department of Statistics, The Chinese University of Hong Kong, Hong Kong; Email: \texttt{genli@cuhk.edu.hk}.}}
\date{\today}

\makeatother

\begin{document}

\theoremstyle{plain} \newtheorem{lemma}{\textbf{Lemma}}\newtheorem{proposition}{\textbf{Proposition}}\newtheorem{theorem}{\textbf{Theorem}}

\theoremstyle{assumption}\newtheorem{assumption}{\textbf{Assumption}}
\theoremstyle{remark}\newtheorem{remark}{\textbf{Remark}}
\theoremstyle{definition}\newtheorem{definition}{\textbf{Definition}}
\theoremstyle{corollary}\newtheorem{corollary}{\textbf{Corollary}}

\maketitle 

\begin{abstract}
Multi-agent robust reinforcement learning, also known as multi-player robust Markov games (RMGs), is a crucial framework for modeling competitive interactions under environmental uncertainties, with wide applications in multi-agent systems.
However, existing results on sample complexity in RMGs suffer from at least one of three obstacles: restrictive range of uncertainty level or accuracy, the curse of multiple agents, and the barrier of long horizons, all of which cause existing results to significantly exceed the information-theoretic lower bound.
To close this gap, we extend the Q-FTRL algorithm \citep{li2022minimax} to the RMGs in finite-horizon setting, assuming access to a generative model.
We prove that the proposed algorithm achieves an $\varepsilon$-robust coarse correlated equilibrium (CCE) with a sample complexity (up to log factors) of $\widetilde{O}\left(H^3S\sum_{i=1}^mA_i\min\left\{H,1/R\right\}/\varepsilon^2\right)$,
where $S$ denotes the number of states, $A_i$ is the number of actions of the $i$-th agent, $H$ is the finite horizon length, and $R$ is uncertainty level.
We also show that this sample compelxity is minimax optimal by combining an information-theoretic lower bound.
Additionally, in the special case of two-player zero-sum RMGs, the algorithm achieves an $\varepsilon$-robust Nash equilibrium (NE) with the same sample complexity.
%To the best of our knowledge, this is the first result to achieve minimax-optimal sample complexity for RMGs. 
\end{abstract}

\section{Introduction}

The rapidly evolving field of multi-agent reinforcement learning (MARL), also referred to as Markov games (MGs) \citep{littman1994markov, shapley1953stochastic}, explores how a group of agents interacts in a shared, dynamic environment to maximize their individual expected cumulative rewards \citep{zhang2020model, lanctot2019openspiel, silver2017mastering, vinyals2019grandmaster}. This area has found wide applications in fields such as ecosystem management \citep{fang2015security}, strategic decision-making in board games \citep{silver2017mastering}, management science \citep{saloner1991modeling}, and autonomous driving \citep{zhou2020smarts}.
However, in real-world applications, environmental uncertainties---stemming from factors such as system noise, model misalignment, and the sim-to-real gap---can significantly alter both the qualitative outcomes of the game and the cumulatiev rewards that agents receive \citep{slumbers2023game}. It has been demonstrated that when solutions learned in a simulated environment are applied, even a small deviation in the deployed environment from the expected model can result in catastrophic performance drops for one or more agents \citep{shi2024sample, balaji2019deepracer, yeh2021sustainbench, zeng2022resilience, zhang2020robust}.

These challenges motivate the study of robust Markov games (RMGs), which assume that each agent aims to maximize its worst-case cumulative reward in an environment where the transition model is constrained by an uncertainty set centered around an unknown nominal model. Given the competitive nature of the game, the objective of RMGs is to reach an equilibrium where no agent has an incentive to unilaterally change its policy to increase its own payoff. A classical type of equilibrium is the robust Nash equilibrium (NE) \citep{nash1950equilibrium}, where each agent's policy is independent, and no agent can improve its worst-case performance by deviating from its current strategy.
Due to the high computational cost of solving robust NEs, especially in games with more than two agents, this concept is often relaxed to the robust coarse correlated equilibrium (CCE), where agents' policies may be correlated \citep{moulin1978strategically}.

In the context of RMGs, achieving equilibrium with minimal samples is of particular interest, as data is often limited in practical applications. While there has been extensive research on the sample complexity of MGs, studies on RMGs remain limited \citep{kardecs2011discounted},
including the asymptotic convergence analysis \citep{zhang2020robust} and the finite-sample analysis for RMGs \citep{blanchet2024double,ma2023decentralized,shi2024curious,shi2024breaking}. 
Specifically, \citet{ma2023decentralized} showed that robust Correlated Equilibrium (CE) can be achieved with $\widetilde{O}\left(H^5S\max_{1\le i\le m}A_i^2/\varepsilon^2\right)$ samples when the uncertainty level $R$ is sufficiently small, i.e., $R\le\varepsilon/(SH^2)$, where $S$ denotes the number of states, $A_i$ denotes the number of actions for the $i$-th agent, and $H$ denotes the finite horizon length.
Additionally, \citet{blanchet2024double} and \citet{shi2024sample} proved that robust NEs (and CCEs) can be achieved for all uncertainty levels within $[0,1]$ using $\widetilde{O}\left(H^4S^3\prod_{i=1}^mA_i^2/\varepsilon^2\right)$ and $\widetilde{O}\left(H^3S\prod_{i=1}^mA_i\min\{H,1/R\}/\varepsilon^2\right)$ samples, respectively, assuming access to offline dataset or a generative model, where $m$ is the number of agents.
Recently, \citet{shi2024breaking} achieved robust CCEs with sample complexity $\widetilde{O}\left(H^6S\sum_{i=1}^mA_i\min\{H,1/R\}/\varepsilon^4\right)$, under a policy-induced uncertainty set definition.
However, all prior results suffer from at least one of three obstacles: restrictive range of uncertainty level $R$ or accuracy $\varepsilon$, the curse of multiple agents, and the barrier of long horizons $H$. 
As a result, they fall short of achieving minimax-optimal sample complexity. 
% These limitations naturally lead to the following question:
% \begin{quote} 
% Can we design an algorithm for RMGs that achieves an 
% $\varepsilon$-robust NE/CCE with minimax-optimal sample complexity?
% \end{quote}

%Introduction to robust MARL, robust MG, NE/CCE.

%interest in sample complexity. related works, motivation

\subsection{Our contributions}

In this paper, we study the problem of RMGs in the finite-horizon setting with an $R$-contamination model \citep{huber1965robust} as uncertainty set.
This model has been widely adopted in previous works on robust single-agent reinforcement learning \citep{ma2023decentralized, wang2021online} due to its computational efficiency.
For the first time, we achieve minimax-optimal sample complexity for RMGs through establishing  matching upper and lower bounds. Specifically, let $S$ and $m$ denote the number of states and agents, respectively, $A_i$ represent the number of actions available to the $i$-th agent, and $H$ denote the horizon length. We focus on analyzing the minimal number of samples required to reach an $\varepsilon$-robust NE or CCE, where no agent has an incentive to unilaterally change its policy to improve its worst-case payoff, given that the transition model is constrained within an $R$-contamination set, with $R$ representing the uncertainty level.
Our contributions are as follows:

\begin{itemize}
\item 
We extend the Q-FTRL algorithm for standard MGs to RMGs and establish its sample complexity. Specifically, the proposed algorithm achieves an $\varepsilon$-robust CCE with the following sample complexity (up to logarithmic factors):
$$
\widetilde{O}\left(\frac{H^3S\sum_{i=1}^mA_i}{\varepsilon^2}\min\left\{H,\frac{1}{R}\right\}\right)
$$
for a full range of $\varepsilon\in(0,H]$ and uncertainty level $R\in[0,1]$.
 In the special case of two-player zero-sum RMGs, the algorithm also achieves an $\varepsilon$-robust NE with the same sample complexity. Compared to previous results \citep{shi2024breaking,shi2024sample, ma2023decentralized}, our approach effectively addresses the challenges posed by multiple agents and long horizon lengths. Furthermore, our result holds for the entire range of $\varepsilon$ and $R$, implying that the algorithm requires no burn-in cost, and allows for sufficient robustness. Detailed comparisons are in Table \ref{tab:related-works}.
\item
For the $R$-contamination set considered in this work, we derive an information-theoretic lower bound on the sample complexity of RMGs with a generative model. Specifically, we show that to achieve an $\varepsilon$-robust CCE, the minimum required number of samples is at least of the order (up to logarithmic factors):
$$
\widetilde{O}\left(\frac{H^3S\max_{1\le i\le m}A_i}{\varepsilon^2}\min\left\{H,\frac{1}{R}\right\}\right).
$$
Combining this lower bound with our algorithm's sample complexity, we conclude that the proposed algorithm achieves minimax-optimal sample complexity when the number of agents is fixed. 
\end{itemize}

\begin{table}
\begin{small}
\centering
\begin{tabular}{c|c|c|c}
\toprule
Algorithm & Equilibrium & $\varepsilon$ and $R$ range & Sample compl.  \\
\hline
\multirow{2}{*}{\citet{ma2023decentralized}} & \multirow{2}{*}{CE} &  $\varepsilon\in(0,H]$, $R\in\left(0,\frac{\varepsilon}{SH^2}\right]$ & $\frac{H^6S\max_{1\le i\le m}A_i^2}{\varepsilon^2}$\\
\cline{3-4}
 &  &  $\varepsilon\in(0,H]$, $R\in\left(\frac{\varepsilon}{SH^2},\frac{p_{\min}}{H}\right]$ & $\frac{H^6S\max_{1\le i\le m}A_i^2}{\varepsilon^2p_{\min}^2}$\\
\hline
\citet{blanchet2024double} & NE & $\varepsilon\in(0,H]$, $R\in[0,1]$ & $\frac{H^4S^3\prod_{i=1}^mA_i^2}{\varepsilon^2}$\\
\hline
\citet{shi2024sample} & NE & \makecell{$\varepsilon\in\left(0,\sqrt{\min\{H,\frac{1}{R}\}}\right]$,\\ $R\in[0,1]$} & $\frac{H^3S\prod_{i=1}^mA_i}{\varepsilon^2}\min\{H,\frac{1}{R}\}$\\
\hline
\citet{shi2024breaking} & CCE & \makecell{$\varepsilon\in\left(0,\sqrt{\min\{H,\frac{1}{R}\}}\right]$,\\ $R\in[0,1]$} & $\frac{H^6S\sum_{i=1}^mA_i}{\varepsilon^4}\min\{H,\frac{1}{R}\}$\\
\hline
\textbf{this work(Theorem \ref{thm:upper})} & CCE & $\varepsilon\in(0,H]$, $R\in[0,1]$ & $\frac{H^3S\sum_{i=1}^mA_i}{\varepsilon^2}\min\{H,\frac{1}{R}\}$\\
\hline
\textbf{this work(Corollary \ref{cor:main})} & \makecell{NE\\ (two-player)} & $\varepsilon\in(0,H]$, $R\in[0,1]$ & $\frac{H^3S\sum_{i=1}^mA_i}{\varepsilon^2}\min\{H,\frac{1}{R}\}$\\
\bottomrule
\end{tabular}
\caption{Our results and comparisons with prior art (up to log factors) for RMG algorithms in terms of equilibrium type, range of $\varepsilon$ and $R$, and the sample complexity.
Here $p_{\min}$ denotes the minimal number of nominal transition probability.}
\label{tab:related-works}
\end{small}
\end{table}

\subsection{Related works}

In this section, we shall briefly review some works related to robust Markov games.

\paragraph{Standard Markov games.}
Early research on Markov games, also known as Multi-agent reinforcement learning (MARL) \citep{shapley1953stochastic, oroojlooy2023review, zhang2021multi}, primarily focused on asymptotic analysis \citep{littman1996generalized, littman2001friend}, but recent theoretical efforts have increasingly concentrated on finite-sample analysis. A significant portion of the literature has focused on the special case of two-player zero-sum Markov games, examples include \citet{bai2020near,bai2020provable,zhong2022pessimistic,chen2022almost,cui2022offline,chen2019distributionally,dou2022gap,jia2019feature,mao2023provably,tian2021online,wei2017online,yan2022model,yang2022t}. Other studies have aimed to develop computationally efficient algorithms for multi-agent general-sum MGs, including \citet{daskalakis2023complexity,jin2021v,liu2021sharp,mao2023provably,song2021can}.
The minimax-optimal sample complexity for finite-horizon settings has been established as $\widetilde{O}\left(\frac{H^4S\sum_{i=1}^mA_i}{\varepsilon^2}\right)$ \citep{li2022minimax}, and has been achieved by \citet{li2022minimax} under the same generative model assumption employed in this work. 
%In contrast, our study demonstrates that learning RMGs is no more challenging than learning standard MGs. Specifically, we show an improvement of a factor of $HR$ in sample complexity when the uncertainty level $R\gtrsim\frac{1}{H}$.

\paragraph{Robust reinforcement learning.}
In the context of single-agent reinforcement learning (RL), the study of model robustness through distributionally robust dynamic programming and robust Markov Decision Processes (MDPs) has garnered significant attention, both empirically and theoretically \citep{badrinath2021robust,goyal2023robust,ho2018fast,ho2021partial,iyengar2005robust,roy2017reinforcement,tamar2014scaling,xu2010distributionally}. A growing body of work has focused on the finite-sample analysis for robust RL, where the agent aims to optimize its worst-case performance under a transition model constrained by an uncertainty set \citep{panaganti2022sample,yang2022toward,zhou2021finite,wang2023finite,wang2024sample,yang2023robust}.
Among these studies, \citet{wang2021online} explored the robust variant of Q-learning algorithm under the R-contamination set and derived a sample complexity of $\widetilde{O}\left(\frac{SA}{(1-\gamma)^5\varepsilon^2}\right)$ in the discounted infinite-horizon setting, where $\gamma$ represents the discount factor. However, this work does not explicitly provide the impact of robustness on the sample complexity. %In contrast, our result demonstrates that when the uncertainty level $R$ is greater than $\frac{1}{H}$, robustness leads to a reduction in sample complexity by a factor of $HR$. Similar results have also been observed for uncertainty sets based on total variation distance \citep{shi2024curious,shi2024sample}.

\paragraph{Robust Markov games.}

Research on RMGs has explored robustness across various dimensions, including states \citep{han2022solution}, environment (probabilistic transition kernels) \citep{shi2024sample,shi2024breaking}, and policies \citep{kannan2023smart}. This work focuses on the robustness in environment.

The most relevant works to this study include \citet{kardecs2011discounted,zhang2020robust,blanchet2024double,ma2023decentralized,shi2024sample,shi2024breaking}. In particular, \citet{kardecs2011discounted} established the existence of NE in RMGs under mild conditions on the uncertainty set. \citet{zhang2020model} demonstrated the asymptotic convergence of a Q-learning type algorithm under certain assumptions.
\citet{ma2023decentralized} proposed an algorithm for RMGs with a sample complexity of $\widetilde{O}\left(H^6S\max_{1\le i\le m}A_i^2/\varepsilon^2\right)$ when the uncertainty level $R\le \varepsilon/(SH^2)$. This result implies that very little robustness can be tolerated when either $H$ or $S$ is large. In cases of $\varepsilon/(SH^2)<R\le p_{\min}/H$, where $p_{\min}$ represents the minimal nominal transition probability, the sample complexity is proved to be $\widetilde{O}\left(H^6S\max_{1\le i\le m}A_i^2/(\varepsilon^2p_{\min}^2)\right)$.
The algorithm proposed in \citet{blanchet2024double} can reach robust NE with sample complexity $\widetilde{O}\left(H^4S^3\prod_{i=1}^mA_i^2/\varepsilon^2\right)$.
Later, \citet{shi2024sample} achieved a better sample complexity of $\widetilde{O}\left(H^3S\prod_{i=1}^mA_i\min\left\{H,1/R\right\}/\varepsilon^2\right)$ for reaching a robust NE when $\varepsilon \in \left(0,\sqrt{\min\left\{H,1/R\right\}}\right]$, which implies a $\widetilde{O}(H^3S\prod_{i=1}^mA_i)$ burn-in cost. 
Recent work \citet{shi2024breaking} achieved a sample complexity of $\widetilde{O}(H^6S\sum_{i=1}^mA_i\min\{H,1/R\}/\varepsilon^4)$ for any $\varepsilon \in \left(0,\sqrt{\min\left\{H,1/R\right\}}\right]$, which implies a $\widetilde{O}(H^5S\sum_{i=1}^mA_i\max\{1,HR\})$ burn-in cost.
% Compared to previous works, our result improves sample complexity by addressing both the curse of multiagency and the barrier of long horizon.
% As a result, our result achieves the minimax optimality for the first time.
% Additionally, our approach eliminates the burn-in cost and accommodates a full range of uncertainty levels.

\paragraph{Notations and organization.}
The notation $x = [x(s,a_i)]_{(s,a_i)\in\mathcal{S}\times \mathcal{A}_i} \in \mathbb{R}^{SA_i}$ (resp. $x = [x(s)]_{s\in\mathcal{S}} \in \mathbb{R}^S$ ) denotes a vector
assigning values to each state-action pair for the $i$-th agent (resp. each state).
The Hadamard product of two vectors $x$ and $y$ is denoted as $x\circ y = [x(s) \cdot y(s)]_{s\in\mathcal{S}}$ or as $x\circ y = [x(s,a_i) \cdot y(s,a_i)]_{(s,a_i)\in\mathcal{S}\times\mathcal{A}_i}$.
Notation $\mathsf{Var}_{P}(V) = P(V\circ V) - (PV)\circ (PV)$ denotes the variance of $V$ following the distribution of $P$.
Notation $\Delta(\mathcal{S})$ represents the probability simplex over the set $\mathcal{S}$.
Moreover, with $\mathcal{X} := \{S,\{A_i\}_{i=1}^m, H,R, 1/\delta, 1/\varepsilon\}$. 
we also use $f(\mathcal{X}) = O(g(\mathcal{X}))$ or $f(\mathcal{X}) \lesssim g(\mathcal{X})$ to indicate that there exists a universal constant $C > 0$ such that $f(\mathcal{X}) \le Cg(\mathcal{X})$.
Similarly $f(\mathcal{X}) \gtrsim g(\mathcal{X})$ indicates $g(\mathcal{X}) = O(f(\mathcal{X}))$. Additionally, the notation $\widetilde{O}(\cdot)$ is defined in the same way as $O(\cdot)$ except that it hides logarithmic factors.

The remainder of this paper is organized as follows. Section \ref{sec:background} reviews some basic concepts about the standard and robust MGs in finite-horizon setting.
In Section \ref{sec:algandtheory}, we propose an algorithm for RMGs with a generative model.
Moreover, we present our main results, including the smaple complexity analysis of the proposed algorithm, and an information-theoretic lower bound for RMGs. 
Section \ref{sec:analysis} outlines the key steps of our theoretical analysis, and the other proof details are deferred to Appendix.
Finally, Section~\ref{sec:discussion} concludes the main paper with additional discussions.

\section{Background}
\label{sec:background}

In this section, we first introduce some basics for standard MGs in a finite-horizon setting, which facilitates the understanding of RMGs.
Then we formulate the RMGs and describe our goal.

\subsection{Standard Markov games}
\label{subsec:background-MG}

A finite-horizon, non-stationary multi-player general-sum Markov game (MG) with $m$ players competing against each other is represented by $\mathcal{MG} =\{\mathcal{S}, \{\mathcal{A}_i\}_{1\le i\le m},H,P,r\}$, 
where $\mathcal{S} = \{1, \cdots, S\}$ represents the state space of the shared environment, $\mathcal{A}_i = \{1,\cdots,A_i\}$ is the action space for the $i$-th player, $H$ denotes the horizon of the Markov game, $P = \{P_h\}_{1\le h\le H}$ with $P_h : \mathcal{S} \times \mathcal{A} \to \Delta(\mathcal{S})$ describes the transition probability kernel for $\mathcal{MG}$, and $r = \{r_{i,h}\}_{1\le h\le H,1\le i\le m}$ with $r_{i,h} : \mathcal{S} \times \mathcal{A} \to [0, 1]$ denotes the reward received by the $i$-th player.
Here, $\mathcal{A}$ is defined as the Cartesian product $\mathcal{A} := \mathcal{A}_1 \times \cdots \times \mathcal{A}_m$.
Furthermore, for each $i\in[m]$, we denote $\mathcal{A}_{-i} :=\prod_{j\neq i}\mathcal{A}_j$, which represents the action space of all players except the $i$-th one.
We use notation ${\bm a}\in \mathcal{A}$ and ${\bm a}_{-i} \in \mathcal{A}_{-i}$ to denote a joint action of all players (or a joint action excluding the $i$-th player's action).
For any $(s, {\bm a}, h, s') \in \mathcal{S}\times \mathcal{A} \times [H] \times \mathcal{S}$, we let $P_h(s'|s, {\bm a})$ represent the transition probability from state $s$ to state $s'$ at step $h$, given that the joint action ${\bm a}$ is taken by the players. Additionally, $r_{i,h}(s, {\bm a})\in[0,1]$ indicates the (deterministic) immediate normalized reward that the $i$-th player receives at state $s$ and step $h$ after taking the joint action ${\bm a}$.

% As an important special case, a two-player zero-sum Markov game—denoted byMG =
% 
% S, {A1,A2},H, P, r
	
% — satisfies r2,h = −r1,h for all h ∈ [H]. Following the convention, we assume that r1,h ≥ 0 for all h ∈ [H],1
% and refer to the first (resp. second) player as the max-player (resp. the min-player).

\paragraph{Markov policies.} This paper focuses on Markov policies, which select actions based solely on the current state $s$ and step $h$, independent of previous trajectories.
%where the action-selection strategies depend solely on the current state $s$ and the step number $h$, without considering previously visited states. 
Specifically, denote policy of the $i$-th player as $\pi_i = \{\pi_{i,h}\}_{1\le h\le H}$, where $\pi_{i,h}(\cdot | s) \in \Delta(\mathcal{A}_i)$ denotes the probability distribution of the $i$-th player's action over space $\mathcal{A}_i$ for state  $s$ at step $h$. The joint Markov policy of all players is defined as $\pi = (\pi_1, \cdots,\pi_m) =\{\pi_h\}_{1\le h\le H}:
\mathcal{S} \times[H] \to \Delta(\mathcal{A})$, 
where $\pi_h(\cdot | s)= (\pi_{1,h},\cdots, \pi_{m,h})(\cdot | s) \in\Delta(\mathcal{A})$
specifies the joint action distribution for all players at state $s$ and step $h$.
For any joint policy $\pi$, we use $\pi_{-i}$ to denote the policies of all players except the $i$-th player, and use $\pi_{-i,h}$ to represent the policies of all but the $i$-th player at step $h$.

%Additionally, a joint policy π is said to be a product policy if π1,..., πm are executed in a statistically independent fashion (namely, under policy π the players take actions independently), and we shall adopt the notation π = π1 × · · · × πm to indicate that π is a product policy.

\paragraph{Value functions and Q functions.}  
%Consider a Markovian trajectory of states and actions $\{(s_h, {\bm a}_h)\}_{1 \leq h \leq H}$, where states $s_{h+1} \in \mathcal{S}$ transits according to the transition kernel when joint action ${\bm a}_h \in \mathcal{A}$ is taken at step $h$ and state $s_h$. 
For a given joint policy $\pi$ and transition kernel $P$, the long-term cumulative reward of agent $i$ in MG is quantified through the value function $V^{\pi, P}_{i,h}: \mathcal{S} \to \mathbb{R}$ and Q-function $Q^{\pi, P}_{i,h}: \mathcal{S}\times\mathcal{A} \to \mathbb{R}$ for the $i$-th agent, defined as follows: 
\begin{align*}
V^{\pi,P}_{i,h}(s) &:= \mathbb{E}_{\pi,P}\left[ \sum_{t=h}^{H} r_{i,t}(s_t, {\bm a}_t) \mid s_h = s \right], \quad \forall (i,h, s) \in [m]\times [H] \times \mathcal{S},\nonumber\\
Q^{\pi,P}_{i,h}(s,{\bm a}) &:= \mathbb{E}_{\pi,P}\left[ \sum_{t=h}^{H} r_{i,t}(s_t, {\bm a}_t) \mid s_h = s,{\bm a}_h = {\bm a} \right], \quad \forall (i,h, s,{\bm a}) \in [m]\times [H] \times \mathcal{S}\times\mathcal{A}.
\end{align*}
where the expectation is taken over the distribution ${\bm a}_t\sim\pi_t(\cdot \mid s_t), s_{t+1}\sim P_t(\cdot \mid s_t, {\bm a}_t),t\ge h$ for $V^{\pi,P}_{i,h}(s)$, and ${\bm a}_{t+1}\sim\pi_{t+1}(\cdot \mid s_{t+1}), s_{t+1}\sim P_t(\cdot \mid s_t, {\bm a}_t),t\ge h$ for $Q^{\pi,P}_{i,h}(s,{\bm a})$.
%Markovian trajectory $\{(s_h, {\bm a}_h)\}$ under the policy $\pi$; that is, conditional on $s_h$, the joint action ${\bm a}_h$ follows the distribution $\pi_h(\cdot \mid s_h)$, and then $s_{h+1}$ is sampled from $s_{h+1}\sim P_h(\cdot \mid s_h, {\bm a}_h)$.

\subsection{Robust Markov games}

A multi-agent general-sum robust Markov game (RMG) in the finite-horizon setting is represented by
$$
\mathcal{MG}_{\mathsf{rob}} = \left\{ \mathcal{S}, \{\mathcal{A}_i\}_{1 \leq i \leq m}, \mathcal{U}^{R}(P^0), r, H \right\},
$$
where $\mathcal{S}$,  $\{\mathcal{A}_i\}$,  $r$,  $H$ and policies are defined as in standard MGs (cf. Section \ref{subsec:background-MG}). 
The key distinction from standard MGs is that: instead of transitioning according to a fixed probability kernel, players in an RMG aim to maximize their worst-case cumulative reward, assuming the probability kernel resides within an uncertainty set $\mathcal{U}^{R}(P^0)$ centered around the nominal probability kernel $P^0$.
This modification leads to differences in the definitions of the value and Q-functions, which will be detailed later.
%A key distinction in this setting from standard MGs is that, instead of relying on a fixed transition kernel, the $i$-th agent assumes that the transition kernel can be selected from a given uncertainty set $\mathcal{U}^{R}(P^0)$. This uncertainty set is based on a nominal kernel $P^0 = \{P_h^0\}_{h=1}^H$ with ,

In the case of two-player zero-sum RMGs, we have $m=2$ and $r_{2,h} = -r_{1,h}$.
All other definitions remain consistent with those in the multi-agent general-sum RMGs.

\paragraph{Uncertainty set with agent-wise $(s, a)$-rectangularity.}
Now we intend to introduce the uncertainty sets  $\mathcal{U}^R(P^0)$  for RMGs. 
In this work, the uncertainty set is constructed using the R-contamination model \citep{huber1965robust}, which has been extensively used in the context of robust statistics \citep{huber1973use,huber2011robust}, economics \citep{nishimura2004search,nishimura2006axiomatic}, statistical learning \citep{duchi2019distributionally}, and robust reinforcement learning \citep{wang2021online}. 
Following \citet{shi2024sample}, we consider a multi-agent variant of rectangularity in RMGs --- agent-wise $(s, a)$-rectangularity, which enables the robust counterpart of Bellman equations and computational tractability of the problems \citep{iyengar2005robust,shi2024curious,wiesemann2013robust,zhou2021finite}.

Specifically, denote nominal probability kernel $P^0=\{P_h^0\}_{h=1}^H$ with $P_h^0: \mathcal{S} \times \mathcal{A} \to \Delta(\mathcal{S})$.
The uncertainty set  $\mathcal{U}^R(P^0)$, based on R-contamination model and satisfying agent-wise $(s, a)$-rectangularity, is mathematically defined as:
\begin{align}\label{eq:def-Rset}
%\label{eq:def-rect}
\mathcal{U}^R(P^0) = \otimes \mathcal{U}^{R}(P_{h,s,{\bm a}}^0), \quad {\rm with}\quad
 \mathcal{U}^R(P_{h,s,{\bm a}}^0) = \left\{ (1-R)P^0_{h,s,{\bm a}} + RP'|P' \in \Delta(\mathcal{S})\right\},
\end{align}
where  $\otimes$  represents the Cartesian product, and $P_{h,s,{\bm a}}^0$ denotes the transition probability at state-action pair $(s, {\bm a}) \in \mathcal{S} \times \mathcal{A}$, that is
$$
P_{h,s,{\bm a}}^0 := P^0_h(\cdot \mid s,{\bm a}) \in \mathbb{R}^{S}.
$$
Here $R\in[0,1]$ denotes the size of the uncertainty set and is referred to as the uncertainty level.

\paragraph{Robust value functions and robust Bellman equations.} 
For any given joint policy $\pi$, the worst-case performance of the $i$-th agent at time step $h$, over all possible transition kernels within uncertainty set $\mathcal{U}^R(P^0)$, is quantified by the robust value function $ V^{\pi, R}_{i,h} $ and the robust Q-function $ Q^{\pi, R}_{i,h} $, which are defined as follows:
\begin{align*}
V_{i,h}^{\pi,R}(s) &= \inf_{P \in \mathcal{U}^R(P^0)} V^{\pi, P}_{i,h}(s) = \inf_{P\in\mathcal{U}^R(P^0)}\mathbb{E}_{\pi,P}\left[\sum_{t=h}^Hr_{i,t}(s_t,{\bm a}_t)|s_h= s\right],\nonumber\\
Q_{i,h}^{\pi,R}(s,{\bm a}) &= \inf_{P \in \mathcal{U}^R(P^0)} Q^{\pi, P}_{i,h}(s,{\bm a}) = \inf_{P\in\mathcal{U}^R(P^0)}\mathbb{E}_{\pi,P}\left[\sum_{t=h}^Hr_{i,t}(s_t,{\bm a}_t)|s_h= s,{\bm a}_h = {\bm a}\right], 
\end{align*}
for all $ (i, h, s, {\bm a}) \in [m] \times [H] \times \mathcal{S} \times \mathcal{A}$.

In the context of RMGs, each agent aims to compete with other agents and maximize its worst-case performance.
To characterize this, given the joint policy $\pi_{-i}$ of all players except the $i$-th one, we use $V^{\star, \pi_{-i}, R}_{i,h}$ to represent the optimal robust value function of the $i$-th agent by playing an independent policy $ \pi'_i $, which is defined as: 
%The corresponding value function under the joint policy $ \pi'_i \times \pi_{-i} $ is denoted as $ V^{\pi'_i \times \pi_{-i}, R}_{i,h}$. 
%The optimal robust value function which is maximized over all $ \pi'_i $ is defined as
\begin{align*}
V^{\star, \pi_{-i}, R}_{i,h}(s) := \max_{\pi'_i: \mathcal{S} \times [H] \to \Delta(\mathcal{A}_i)} V^{\pi'_i \times \pi_{-i}, R}_{i,h}(s)
= \max_{\pi'_i: \mathcal{S} \times [H] \to \Delta(\mathcal{A}_i)} \inf_{P \in \mathcal{U}^R(P^0)} V^{\pi'_i \times \pi_{-i}, P}_{i,h}(s).
\end{align*}
As in standard MGs, it can be shown that there exists at least one policy --- referred to as the robust best-response policy for the 
$i$-th agent---that achieves $V_{i,1}^{\star,\pi_{-i},R}(s)$ for all $s \in \mathcal{S}$ and $h \in [H]$
(see \citet{blanchet2024double}, Section A.2).
%, which is dependent on $\pi_{-i}$ and referred to as the robust best-response policy for the $i$-th agent, that can simultaneously attain  .

The robust value functions $ \{V^{\pi, R}_{i,h}\} $ of RMGs with any joint policy $ \pi $ satisfy the following robust Bellman consistency equation: for all $ (i, h, s) \in [m] \times [H] \times \mathcal{S} $,
\begin{align*}
V^{\pi, R}_{i,h}(s) &=
\mathbb{E}_{{\bm a} \sim \pi_h(\cdot|s)} \left[ r_{i,h}(s, {\bm a}) + \inf_{P_{h,s,{\bm a}} \in \mathcal{U}^R(P^0_{h,s,{\bm a}})} P_{h,s,{\bm a}} V^{\pi, R}_{i,h+1} \right]\nonumber\\
 &= \mathbb{E}_{{\bm a}\sim\pi_h(\cdot|s)} \left[r_{i,h}(s,{\bm a}) + (1-R)P_{h,s,{\bm a}}^0 V_{i,h+1}^{\pi,R} + R\min V_{i,h+1}^{\pi,R}\right].
\end{align*}
This robust Bellman equation is inherently tied to the agent-wise $ (s, a) $-rectangularity condition and the R-contamination model (cf. \eqref{eq:def-Rset}). 
% \begin{align*}
% V_{i,h}^{\star,\pi_{-i},R}(s) &= \max_{\pi_i':\mathcal{S}\times[H]\to \Delta(\mathcal{A}_i)} V_{i,h}^{\pi_i'\times \pi_{-i},R}(s).
% \end{align*}

%Specifically, this condition decouples the dependency of uncertainty subsets across different agents, each state-action pair, and different time steps, leading to the Bellman recursive equation.

\paragraph{Equilibria in robust Markov games.}
In RMGs, each player seeks to maximize its own value function under the worst-case transition scenarios. Due to the competing objectives of players, finding an equilibrium is central to the study of RMGs. Below, we introduce two robust variants of common solution concepts --- robust Nash Equilibrium (NE) and robust Coarse Correlated Equilibrium (CCE).

\begin{itemize}
\item
Robust NE. A product policy $ \pi = \pi_1 \times \pi_2 \times \cdots \times \pi_m $ is a robust NE if 
\begin{align}\label{eq:def-NE}
\forall(i, s) \in [m] \times \mathcal{S} : \quad V^{\pi, R}_{i,1}(s) = V^{\star, \pi_{-i}, R}_{i,1}(s). 
\end{align}
A robust NE implies that, given the strategies of the other players $\pi_{-i} $, no player can improve its worst-case performance over the uncertainty set $\mathcal{U}^R(P^0)$ by unilaterally deviating from its current strategy.
\item
Robust CCE. A (possibly correlated) joint policy $ \pi \in \mathcal{S} \times [H] \to \Delta(\mathcal{A}) $ is a robust CCE if
\begin{align}\label{eq:def-CCE}
\forall(i, s) \in [m] \times \mathcal{S} : \quad V^{\pi, R}_{i,1}(s) \geq V^{\star, \pi_{-i}, R}_{i,1}(s).
\end{align}
Robust CCE can be regarded as relaxation of robust NE. It also guarantees that no player benefits from unilaterally deviating its current strategy, but it allows for the possibility that the players' policies may be correlated.

% \item
% Robust CE. A joint policy $ \pi \in \Delta(\mathcal{A}) $ is said to be a robust CE if it holds that 
% $$
% \forall(s, i) \in \mathcal{S} \times [m] : V^{\pi, R}_{i,1}(s) \geq \max_{f_i \in \mathcal{F}_i} V^{f_i \circ \pi, R}_{i,1}(s).
% $$
\end{itemize}

It is well-known that computing exact robust equilibria is a challenging task and may not always be necessary.
Therefore, approximate equilibria are often sought which introduces the concept of an $ \varepsilon $-robust NE. 
A product policy $ \pi =\pi_1\times\cdots\times \pi_m$ is said to be an $ \varepsilon $-robust NE if
%a slight relaxation of the definition in \eqref{eq:def-NE}
%As a result, people usually search for approximate equilibria. Toward this, as a slight relaxation from \eqref{eq:def-NE}, a product policy $ \pi \in \Delta(\mathcal{A}_1) \times \cdots \times \Delta(\mathcal{A}_m) $ is said to be an $ \varepsilon $-robust NE if
\begin{align}\label{eq:gap-multi-NE}
\mathsf{gap}_{\mathsf{NE}}(\pi) = \max_{s\in\mathcal{S}}\mathsf{gap}(\pi;s)\le\varepsilon,\quad {\rm with}\quad \mathsf{gap}(\pi;s):=\max_{1\le i\le m}\left\{V_{i,1}^{\star,\pi_{-i},R}(s)-V_{i,1}^{\pi,R}(s)\right\}.
\end{align}

Similarly, a slight relaxation of the definition in \eqref{eq:def-CCE} introduces the concept of $\varepsilon $-robust CCE.
%relaxing \eqref{eq:def-CCE}, 
A possibly correlated joint policy $ \pi\in\mathcal{S}\times[H] \to \Delta(\mathcal{A}) $ is said to be an $\varepsilon $-robust CCE if
\begin{align}\label{eq:gap-multi}
\mathsf{gap}_{\mathsf{CCE}}(\pi) = \max_{s\in\mathcal{S}}\mathsf{gap}(\pi;s)\le\varepsilon,\quad {\rm with} \quad \mathsf{gap}(\pi;s):=\max_{1\le i\le m}\left\{V_{i,1}^{\star,\pi_{-i},R}(s)-V_{i,1}^{\pi,R}(s)\right\}.
\end{align}
The existence of robust NEs has been proved under general conditions for the uncertainty set \citep{blanchet2024double}, and it directly implies the existence of robust CCEs \citep{shi2024sample}.

% \begin{align*}
% \mathsf{gap}_{\mathsf{NE}}(\pi) := \max_{s \in \mathcal{S}, 1 \leq i \leq m} \left( V^{\star, \pi_{-i}, \sigma_i}_{i,1}(s) - V^{\pi, \sigma_i}_{i,1}(s) \right) \leq \epsilon.
% \end{align*}
% or an $ \varepsilon $-robust CE if
% $$
% \mathsf{gap}_{\mathsf{CE}}(\pi) := \max_{s \in \mathcal{S}, 1 \leq i \leq n} \left( \max_{f_i \in \mathcal{F}_i} V^{f_i \circ \pi, R}_{i,1}(s) - V^{\pi, R}_{i,1}(s) \right) \leq \epsilon. 
% $$

\paragraph{Sampling mechanism: a generative model.}
In this paper, we adopt a commonly used sampling mechanism in standard MARL \citep{li2022minimax,zhang2020robust}.
Specifically, we assume access to a generative model (simulator) \citep{kearns1998finite}, which returns an independent sample of state $s'$ based on the nominal transition kernel $P^0$, as follows:
\begin{align*}
s'\sim P^0_h(\cdot | s, {\bm a}),
\end{align*}
for any pair $(s, {\bm a}, h)\in \mathcal{S} \times\mathcal{A} \times [H]$ selected by the learners.
The objective of this paper is to compute an $\varepsilon$-robust equilibrium for the game $\mathcal{MG}_{\mathsf{rob}}$ with as few samples as possible.
%, i.e., minimizing the number of calls to the generative model.

\section{Sample-efficient robust Markov games}
\label{sec:algandtheory}

In this section, we extend the Q-FTRL algorithm proposed by \citet{li2022minimax} to the context of RMG, and show that it retains the minimax sample complexity in this setting.

\subsection{Algorithm description}
\label{sec:alg}

The proposed algorithm leverages the principles of optimism and the Follow-the-Regularized-Leader (FTRL) framework.
It operates by working backwards from the terminal step $H$ to the initial step $h=1$. 
At each step $h$, each agent $i$ samples $K$ times from nominal transition probability for each state-action pair $(s,a_i)\in\mathcal{S}\times \mathcal{A}$, and estimates the worst-case Q-function under the current policy after each sampling. Subsequently, each agent utilizes the FTRL method to update its policy $\widehat{\pi}_{i,h}^{k+1}$, such that the resulting joint correlated policy approximates a robust CCE at step $h$. 
Once the estimation at step $h$ is completed, the algorithm computes the value function under the optimized policy and proceeds to the next step $h-1$.

We now present the mathematical formulation of the proposed algorithm. The algorithm begins by initializing value functions as $\widehat{V}_{i,H+1} = 0\in\mathbb{R}^{S}$ for all agents $1\le i\le m$. It then iterates over each step $h$, utilizing the value function estimates $\{\widehat{V}_{i,h+1}\}_{i=1}^m$ from the previous step.
Specifically, at each step $h$, the procedure proceeds in two phases: policy updates and value function estimation.

\begin{algorithm}[!ht]
\renewcommand{\algorithmicrequire}{\textbf{Input:}}
\renewcommand{\algorithmicensure}{\textbf{Output:}}
\caption{Sample-efficient Robust Q-FTRL}
\label{alg}
\begin{small}
\begin{algorithmic}[1]
\REQUIRE number of rounds $K$ for each step, learning rates $\{\alpha_k\}$ and $\{\eta_{k+1}\}$.
%\ENSURE Q-function $Q_T$.
\STATE \textbf{Initialize}: for any $(s, a_i, h,i) \in \mathcal{S} \times \mathcal{A}_i\times [H]\times[m]$, set $\widehat{V}_{i,H+1}(s) = Q^0_{i,h}(s, a_i) = 0$ and $\pi^1_{i,h}(a_i | s) = 1/A_i$.
\FOR{$h=H,\cdots,1$}
\FOR{$k=1,\cdots,K$}
\FOR{$i=1,\cdots,m$}
%\STATE Draw actions ${\bm a}_{i,h}^k(s,a_i)$ for all $(s,a_i)\in\mathcal{S}\times\mathcal{A}_i$ based on the policy $\{\pi^k_{j,h}\}_{j\in[m]}$ as specified in \eqref{eq:sample-action}.
\STATE Draw independent samples from generative model according to \eqref{eq:sample} for all $(s,a_i)\in\mathcal{S}\times\mathcal{A}_i$, and construct the model estimation $r^k_{i,h}, P^k_{i,h}$ using \eqref{eq:model-estimation-r} and \eqref{eq:model-estimation-P}.
%$(r^k_{i,h}, P^k_{i,h})\leftarrow \mathsf{sampling}(i, h, \pi^k_h= \{\pi^k_{j,h}\}_{j\in[m]})$.
\STATE Compute $Q_{i,h}^k$ according to \eqref{eq:update-Q}.
\STATE Update $\pi^{k+1}_{i,h}(\cdot|s)$ according to \eqref{eq:alg-update-pi} for all $s\in\mathcal{S}$.
\ENDFOR
\ENDFOR
\STATE Compute $\widehat{V}_{i,h}$ according to \eqref{eq:cal-hatV}.
\ENDFOR
\IF{$\mathcal{MG}$ is a two-player zero-sum robust Markov game}
\STATE \textbf{Output}: $\widehat{\pi}_1 \times \widehat{\pi}_2$, where for any $i = 1, 2$, $\widehat{\pi}_i = \{\widehat{\pi}_{i,h}\}_{1\le h\le H}$ with $\widehat{\pi}_{i,h} =\sum_{k=1}^K\alpha_k^K\pi^k_{i,h}$.
\ENDIF
\IF{$\mathcal{MG}$ is a multi-player general-sum robust  Markov game}
\STATE \textbf{Output}: $\widehat{\pi} = \{\widehat{\pi}_{h}\}_{1\le h\le H}$ where $\widehat{\pi}_{h} =\sum_{k=1}^K\alpha_k^K(\pi^k_{1,h}\times \cdots\times \pi^k_{m,h})$.
\ENDIF
\end{algorithmic}
\end{small}
\end{algorithm}

\paragraph{Sampling and policy updates}
At each step, the algorithm draws $K$ samples from the generative model for each state-action pair $(s,a_i)\in\mathcal{S}\times \mathcal{A}_i$ for all $i\in[m]$, using these samples to update policies.
Specifically, in the $k$-th round of sampling, a tuple of samples $\{{\bm a}_h^k(s,a_i)=\{a_{j,h}^k(s,a_i)\}_{j=1}^m, r_{i,h}^k(s,a_i), s_{i,h+1}^k(s,a_i)\}_{(s,a_i)\in\mathcal{S}\times\mathcal{A}_i}$ is drawn for each agent $i$ according to the following rules:
\begin{align}
a_{j,h}^k(s,a_i)&\sim\pi_{j,h}^k(\cdot|s),~ (j\neq i)\quad{\rm and}\quad a_{i,h}^k(s,a_i) = a_i;\label{eq:sample-action}\\
r_{i,h}^k(s,a_i) &= r_{i,h}(s,{\bm a}_h^k(s,a_i));\qquad
s_{i,h+1}^k(s,a_i)\sim P_h^0(\cdot|s,{\bm a}_h^k(s,a_i)),\label{eq:sample}
\end{align}
where $\pi_{j,h}^k(\cdot|s)$ is initialized as $\pi_{j,h}^1(a_j|s) = 1/A_j$ for all $(s,a_j,j)\in\mathcal{S}\times\mathcal{A}_j\times[m]$.
For convenience, we collect samples for all state-action pairs into a reward vector $r_{i,h}^k\in\mathbb{R}^{SA_i}$ and an empirical transition probability matrix $P^{k}_{i,h} \in\mathbb{R}^{SA_i\times S}$, given by
\begin{align}
r^{k}
_{i,h}  &= [r^{k}
_{i,h}(s,a_i)]_{(s,a_i)\in\mathcal{S}\times\mathcal{A}_i}\in \mathbb{R}^{SA_i},\label{eq:model-estimation-r}\\
P^{k}_{i,h}(s'|s,a_i) &= 1,~{\rm for}~s'=s_{h+1}^{k}(s,a_i),\quad {\rm and}\quad P^{k}_{i,h}(s'|s,a_i) = 0,~{\rm otherwise}.\label{eq:model-estimation-P}
\end{align}
Next, we update the Q-function using robust Q-learning  follows:
\begin{align}\label{eq:update-Q}
Q_{i,h}^k = (1-\alpha_k)Q_{i,h}^{k-1} + \alpha_kq_{i,h}^k,
\end{align}
where $0<\alpha_k<1$ is the learning rate, $Q_{i,h}^k\in\mathbb{R}^{SA_i}$ is initialized as ${Q}_{i,h}^0 = 0$, and $q_{i,h}^k$ is computed as
$$
q_{i,h}^k = r_{i,h}^k + (1-R)P_{i,h}^k\widehat{V}_{i,h+1} + R\min_s \widehat{V}_{i,h+1}.
$$
After updating the Q-function estimates, the algorithm uses the FTRL strategy \citep{shalev2012online} to update the policy of the $i$-th agent as follows: 
\begin{align}\label{eq:alg-update-pi}
\pi_{i,h}^{k+1}(\cdot|s) 
&= \frac{\exp\left(\eta_{k+1}Q_{i,h}^k(s,\cdot)\right)}{\sum_{a\in\mathcal{A}_i}\exp\left(\eta_{k+1}Q_{i,h}^k(s,a)\right)},\quad \forall s\in\mathcal{S},
\end{align}
where $\eta_{k+1}>0$ is another learning rate associated with the policy updates.

\paragraph{Value function estimation}

After $K$ rounds of the above procedure, we construct the final policy estimate $\widehat{\pi}_h: \mathcal{S} \rightarrow \Delta(\mathcal{A})$ as
\begin{align*}
\widehat{\pi}_h({\bm a}|s) &= \sum_{k=1}^K\alpha_k^K\prod_{i=1}^m\pi_{i,h}^k(a_i|s),\quad \forall (s,{\bm a}) \in \mathcal{S} \times \mathcal{A},
\end{align*}
where
\begin{align*}
\alpha_{k}^K &= \alpha_k\prod_{j=k+1}^K(1-\alpha_j),\quad 1\le k\le K.
\end{align*}
Moreover, we estimate the
value function $\widehat{V}_{i,h}\in\mathbb{R}^{S}$ for step $h$ under $\widehat{\pi}_h$.
To obtain an optimistic estimation, we introduce a bonus term $\beta_{i,h}\in\mathbb{R}^{S}$ with the $s$-th entry given by
\begin{align}
\beta_{i,h}(s) &= c_{\mathsf{b}}\sqrt{\frac{\log^3\frac{KS\sum_{i=1}^m A_i}{\delta}}{K\min\left\{H,\frac{1}{R}\right\}}}\sum_{k=1}^K\alpha_k^K\left\{\mathsf{Var}_{\pi_{i,h,s}^k}(q_{i,h,s}^k) + \min\left\{H,\frac{1}{R}\right\}\right\},\label{eq:def-beta}
\end{align}
where $q_{i,h,s}^k = q_{i,h}^k(\cdot|s)\in\mathbb{R}^{A_i}$, and $\pi_{i,h,s}^k = \pi_{i,h}^k(\cdot|s)\in\Delta(\mathcal{A}_i)$.
The value function $\widehat{V}_{i,h}$ is then estimated as
\begin{align}
\widehat{V}_{i,h}(s) &= \min\left\{\sum_{k=1}^K\alpha_k^K\mathbb{E}_{a_i\sim\pi_{i,h,s}^k}q_{i,h,s}^k + \beta_{i,h}(s), H-h+1\right\},\quad \forall s \in \mathcal{S}.\label{eq:cal-hatV}
\end{align}
Equipped with $\widehat{\pi}_h$ and $\widehat{V}_{i,h}$, the algorithm proceeds to step $h-1$ and repeats the aforementioned operations.

% \begin{itemize}
% \item[1.]\emph{Sampling.} For each $(s, a_i) \in \mathcal{S} \times \mathcal{A}_i$, draw an independent action ${\bm a}_{h}^k(s,a_i)=[a_{j,h}^k(s,a_i)]_{1\le j\le m}$ follows
% \begin{align}\label{eq:sample-action}
% a_{j,h}^k(s,a_i) \left\{
% \begin{array}{ll}
%  \sim \pi_{j,h}^k(\cdot|s),    & j\neq i, \\
%  a_i,    & j=i,
% \end{array} 
% \right.
% \end{align}
% and generate an independent sample from the nominal transition probability as
% \begin{align}\label{eq:sample}
% s_{h+1}^{k}(s,a_i)\sim P_h^0(\cdot|s,{\bm a}_{h}^k(s,a_i))\quad {\rm and}\quad r_{i,h}^k(s,a_i) = r_{i,h}(s,{\bm a}_{h}^k(s,a_i)).
% \end{align}
% %for all $(s, a_i, s_{h+1}) \in \mathcal{S}\times\mathcal{A}_i\times\mathcal{S}$.

% \item[2.]\emph{$Q$-function estimation.}
% To estimate the Q-function for the $i$-th agent, we apply the robust Q-learning update rule \citep{wang2021online}:

% \item[3.]\emph{Policy updates.} After updating the Q-function estimates, we employ the FTRL strategy \citep{shalev2012online} to update the policy of the $i$-th agent as follows: 
% \begin{align}\label{eq:alg-update-pi}
% \pi_{i,h}^{k+1}(\cdot|s) 
% &= \arg\min_{\mu\in\Delta(\mathcal{A}_i)}\left\{-\sum_{a\in\mathcal{A}_i}\mu(a)Q_{i,h}^k(s,a) + \frac{\sum_{a\in\mathcal{A}_i}\mu(a)\log\mu(a)}{\eta_{k+1}}\right\}\nonumber\\
% %\pi_{i,h}^{k+1}(\cdot|s) 
% &= \frac{\exp\left(\eta_{k+1}Q_{i,h}^k(s,\cdot)\right)}{\sum_{a\in\mathcal{A}_i}\exp\left(\eta_{k+1}Q_{i,h}^k(s,a)\right)},\quad \forall s\in\mathcal{S},
% \end{align}
% where $\eta_{k+1}>0$ is another learning rate associated with the policy updates.
% \end{itemize}

Finally, we remark that the output policy $\widehat{\pi}$ is  a correlated instead of product policy. In the special case of two-player zero-sum robust Markov games, we could alternatively output a product policy as $\widehat{\pi} = \widehat{\pi}_1\times \widehat{\pi}_2$, where 
\begin{align*}
\widehat{\pi}_i = \{\widehat{\pi}_{i,h}\}_{1\le h\le H},\quad {\rm with}\quad \widehat{\pi}_{i,h} = \sum_{k=1}^K \alpha_{k}^K\pi_{i,h}^k,\quad \forall i=1,2.
\end{align*}
The entire procedure is summarized in Algorithm \ref{alg}, with the total number of samples being $KHS\sum_{i=1}^mA_i$.
% It is worth pointing out that the final policy $\widehat{\pi}$ takes the form of a mixture of product policies. In the special case of two-player zero-sum MGs, we can alternatively output a product policy $\widehat{\pi} = \widehat{\pi}_1\times \widehat{\pi}_2$,
% where for each $i = 1, 2$, we take $\widehat{\pi}_i = \{\widehat{\pi}_{i,h}\}_{1\le h\le H}$ with $\widehat{\pi}_{i,h} = \sum_{k=1}^K \alpha_{k}^K\pi_{i,h}^k$.

\subsection{Main results}
\label{sec:thm}

In this section, we present our main theoretical results concerning the sample complexity of RMGs, which include both an upper bound for the sample complexity of the proposed Algorithm \ref{alg} and an information-theoretic lower bound.

The sample complexity of Algorithm \ref{alg} is formalized in the following theorem, the proof of which is postponed to Section \ref{sec:analysis}.

\begin{theorem}\label{thm:upper}
Consider an $m$-player robust general-sum Markov game with uncertainty set $\mathcal{U}^R$ defined in \eqref{eq:def-Rset} and uncertainty level $R$.
Suppose that the step-size
\begin{align}\label{eq:step-size}
    \alpha_k = \frac{c_{\alpha}\log K}{k-1+c_{\alpha}\log K},\quad \eta_k = \sqrt{\frac{\log K}{\alpha_k \min\left\{H,\frac{1}{R}\right\}}},
\end{align}
where $c_{\alpha}>0$ is a constant.
For any $\varepsilon\in (0,H]$, $0\le R<1$, and $0 < \delta < 1$, with probability at least $1 - \delta$, the joint policy $\widehat{\pi}$ returned by Algorithm \ref{alg} achieves an $\varepsilon$-robust CCE (cf. \eqref{eq:gap-multi})
$$
\mathsf{gap}_{\mathsf{CCE}}(\widehat{\pi})\le\varepsilon
$$
as long as the total number of samples
\begin{align}\label{eq:thm-K}
KHS\sum_{i=1}^m A_i \ge \frac{CH^3S\sum_{i=1}^m A_i}{\varepsilon^2}\min\left\{H,\frac{1}{R}\right\}\log^3\left(\frac{KS\sum_{i=1}^mA_i}{\delta}\right),
\end{align}
where $C>0$ is a large enough universal constant. 
\end{theorem}

Let us consider the special case of two-player zero-sum RMGs.
Under the same conditions specified in Theorem \ref{thm:upper}, Algorithm \ref{alg} achieves an $\varepsilon$-robust NE as stated in the following corollary.
Its proof is direct and is postponed to Appendix \ref{subsec:proof-cor}.
\begin{corollary}\label{cor:main}
    Consider a two-player zero-sum robust Markov game with uncertainty set $\mathcal{U}^R$ defined in \eqref{eq:def-Rset} and uncertainty level $R$. Suppose that the step-size is taken as \eqref{eq:step-size}.
    For any $\varepsilon\in (0,H]$, $0\le R<1$, and any $0 < \delta < 1$, with probability at least $1 -\delta$, the product policy $\widehat{\pi}_1\times \widehat{\pi}_2$ outputted by Algorithm \ref{alg} achieves an $\varepsilon$-robust NE (cf. \eqref{eq:gap-multi-NE})
    $$
\mathsf{gap}_{\mathsf{NE}}(\widehat{\pi})\le\varepsilon
$$
as long as the total number of samples
\begin{align*}
KHS(A_1+A_2) \ge \frac{CH^3S(A_1+A_2)}{\varepsilon^2}\min\left\{H,\frac{1}{R}\right\}\log^3\left(\frac{KS(A_1+A_2)}{\delta}\right),
\end{align*}
where $C>0$ is a large enough universal constant. 
\end{corollary}

The following theorem establishes an information-theoretic lower bound for learning RMGs with a generative model.
Its proof is postponed to Appendix \ref{sec:proof-thm-lower}.

\begin{theorem}\label{thm:lower}
Consider any $m\ge 2$, $H\ge 16$, $0<\varepsilon<c_0H$, and any $R\in[0,c_R)$, where $c_0>0$ is a small enough constant and $c_R$ is a constant within $(0,1)$. We can construct a collection of $m$-player general-sum robust Markov games --- denoted by $\{\mathcal{M}\mathcal{G}_\theta\}_{\theta\in\Theta}$, such that for any dataset comprising $N$ independent samples generated from the nominal environment for each state-action pairs $(s,a_i)\in\mathcal{S}\times\mathcal{A}_i$, $1\le i\le m$, one has
\begin{align*}
\inf_{\widehat{\pi}}\max_{\theta\in\Theta}\mathbb{P}_{\theta}\left(\mathsf{gap}_{\mathsf{CCE}}(\widehat{\pi})>\varepsilon\right)\ge\frac{1}{4},
\end{align*}
provided that
$$
NHS\max_{1\le i\le m}A_i \le \frac{C_1H^3S\max_{1\le i\le m}A_i}{\varepsilon^2}\min\left\{H,\frac{1}{R}\right\}.
$$
Here, $C_1$ is some small enough constant, the infimum is taken over all estimators $\widehat{\pi}$, and $\mathbb{P}_{\theta}$ denotes the
probability when the game is $\mathcal{MG}_\theta$ for all $\theta\in\Theta$.
\end{theorem}

%We now highlight some key implications and comparison with previous works.

%\paragraph{Minimax sample complexity}

According to Theorem \ref{thm:upper}, the sample complexity of Algorithm \ref{alg} is given by (up to logarithmic factors)
$$
\widetilde{O}\left(\frac{H^3S\sum_{i=1}^mA_i}{\varepsilon^2}\min\left\{H,\frac{1}{R}\right\}\right).
$$
When compared to the information-theoretic lower bound established in Theorem \ref{thm:lower}, we observe that this complexity is minimax optimal (up to logarithmic terms), provided the number of players $m$ remains fixed or grows logarithmically with respect to the problem parameters. To the best of our knowledge, this is the first algorithm for RMGs to achieve minimax optimal sample complexity.

%\paragraph{Full $\varepsilon$-range}
Furthermore, our sample complexity holds for the entire range of $\varepsilon\in(0,H]$. 
This is particularly favoring for data-limited applications, as it implies that the algorithm incurs no burn-in cost, making it efficient in scenarios where collecting sample is costly.

\section{Analysis}
\label{sec:analysis}

\subsection{Notations and technical lemmas}
\label{subsec:notations}

We first introduce some notations.
Define $\overline{V}$ as the estimation for $V$ by replacing all expectation operation with the empirical estimation calculated according to $(r_{i,h}^k,P_{i,h}^k)$.
Specifically,
define %$\overline{V}^{\star,\widehat{\pi}_{-i},R}$ as the estimation for the value function of the $i$-th agent under policy $\widehat{\pi}_{-i}$, and the optimized policy for the $i$-th agent:
\begin{align}
\overline{V}_{i,h}^{\widehat{\pi},R}(s) 
&=\sum_{k=1}^K\alpha_{k}^K\mathbb{E}_{a_i\sim\pi_{i,h,s}^k}\left[r_{i,h}^k(s,a_i) + (1-R)P_{i,h}^k(s,a_i)\overline{V}_{i,h+1}^{\widehat{\pi},R} + R\min\overline{V}_{i,h+1}^{\widehat{\pi},R}\right],\label{eq:def-bar-V-pi}\\
\overline{V}_{i,h}^{\star,\widehat{\pi}_{-i},R}(s) 
&= \max_{a_i\in\mathcal{A}_i}\sum_{k=1}^K\alpha_{k}^K\left[r_{i,h}^k(s,a_i) + (1-R)P_{i,h}^k(s,a_i)\overline{V}_{i,h+1}^{\star,\widehat{\pi}_{-i},R} + R\min\overline{V}_{i,h+1}^{\star,\widehat{\pi}_{-i},R} \right].\label{eq:def-bar-V-star}
\end{align}
where $P_{i,h}^k(s,a_i) = P_{i,h}^k(\cdot|s,a_i)\in\mathbb{R}^{S}$ and $\overline{V}_{i,H+1}^{\widehat{\pi},R}=\overline{V}_{i,H+1}^{\star.\widehat{\pi}_{-i},R}=0$.
%$\widehat{\pi}_i^{\star} = \{\widehat{\pi}_{i,h}^{\star}\}_{h=1}^H$ with
Moreover, define
\begin{align}\label{eq:def-pistari}
\widehat{\pi}_{i}^{\star} = \{\widehat{\pi}_{i,h}^{\star}\}_{h=1}^H =\arg\max_{\pi_{i}:\mathcal{S}\times[H]\to\Delta(\mathcal{A}_i)} V_{i,1}^{\pi_{i}\times \widehat{\pi}_{-i},R},
\end{align}
and the corresponding value function estimation $\overline{V}_{i,h}^{\widehat{\pi}_i^{\star}\times\widehat{\pi}_{-i},R}$ with $\overline{V}_{i,H+1}^{\widehat{\pi}_i^\star,\widehat{\pi}_{-i},R}=0$ as
\begin{align}
\overline{V}_{i,h}^{\widehat{\pi}_i^{\star}\times\widehat{\pi}_{-i},R}(s) 
&= \sum_{k=1}^K\alpha_{k}^K\mathbb{E}_{a_i\sim\widehat{\pi}_{i,h,s}^{\star}}\left[r_{i,h}^k(s,a_i) + (1-R)P_{i,h}^k(s,a_i)\overline{V}_{i,h+1}^{\widehat{\pi}_i^{\star}\times\widehat{\pi}_{-i},R} + R\min\overline{V}_{i,h+1}^{\widehat{\pi}_i^{\star}\times\widehat{\pi}_{-i},R} \right],\label{eq:def-bar-V-pi-star}
\end{align}
where $\widehat{\pi}_{i,h,s}^{\star} = \widehat{\pi}_{i,h}^{\star}(\cdot|s)\in\Delta(\mathcal{A}_i)$.
% Moreover, define $\overline{V}^{\star,\widehat{\pi}_{-i},R}$ as the estimation with $i$-th agent's policy optimized:
% \begin{align}
% \end{align}
%
Define the expected reward and the nominal transition probability matrix under policy $\widehat{\pi}$ as 
\begin{align}
r_{i,h}^{\widehat{\pi}}(s) 
&= \mathbb{E}_{{\bm a}\sim\widehat{\pi}} r_{i,h}(s,{\bm a}) = \sum_{k=1}^K\alpha_k^K \mathbb{E}_{a_j\sim\pi_{j,h,s}^k} r_{i,h}(s,{\bm a}) \label{eq:def-r-pi}\\
&= \sum_{k=1}^K\alpha_k^K \mathbb{E}_{a_i\sim\pi_{i,h,s}^k}\left[\mathbb{E}_{a_j\sim\pi_{j,h,s}^k,j\neq i} [r_{i,h}(s,{\bm a})|a_i]\right],\label{eq:def-r-pi-2}\\
P_{h,s}^{\widehat{\pi}}
&=\mathbb{E}_{{\bm a}\sim\widehat{\pi}} P_{h,s,{\bm a}}^0=\sum_{k=1}^K\alpha_k^K\mathbb{E}_{a_j\sim\pi_{j,h,s}^k} P_{h,s,{\bm a}}^0\label{eq:def-P-pi}\\
&=\sum_{k=1}^K\alpha_k^K\mathbb{E}_{a_i\sim\pi_{i,h,s}^k}\left[\mathbb{E}_{a_j\sim\pi_{j,h,s}^k,j\neq i}\left[P_{h,s,{\bm a}}^0|a_i\right]\right].\label{eq:def-P-pi-2}
\end{align}
Similarly, the expected reward and the nominal transition probability matrix under policy $\widehat{\pi}_i^{\star}\times\widehat{\pi}_{-i}$ is denoted as
\begin{align}
r_{i,h}^{\widehat{\pi}_i^{\star}\times\widehat{\pi}_{-i}}(s) &= \mathbb{E}_{{\bm a}\sim\widehat{\pi}_i^{\star}\times\widehat{\pi}_{-i}} r_{i,h}(s,{\bm a}),\label{eq:def-r-pi-star}\\
P_{h}^{\widehat{\pi}_i^{\star}\times\widehat{\pi}_{-i}}(s)&=\mathbb{E}_{{\bm a}\sim\widehat{\pi}_i^{\star}\times\widehat{\pi}_{-i}} P_{h,s,{\bm a}}^0.\label{eq:def-P-pi-star}
\end{align}

In addition, we define an empirical probability transition matrix $\widehat{P}_{i,h}^{\widehat{\pi}}\in\mathbb{R}^{S\times S}$ whose $s$-th row is 
\begin{align}\label{eq:def-hatP-pi}
\widehat{P}_{i,h,s}^{\widehat{\pi}}=
\sum_{k=1}^K\alpha_k^K\mathbb{E}_{a_i\sim\pi_{i,h,s}^k}P_{i,h}^k(s,a_i).
\end{align}

Finally, define $P_{h}^{\widehat{\pi},V}\in\mathbb{R}^{S\times S}$(resp. $P_{h}^{\widehat{\pi}_i^{\star}\times\widehat{\pi}_{-i},V}$, $\widehat{P}_{i,h}^{\widehat{\pi},V}$) as the matrix with the $s$-th row
\begin{align*}
P_{h,s}^{\widehat{\pi},V}: &= \arg\min_{\mathcal{P}\in \mathcal{U}^{R}\left(P_{h,s}^{\widehat{\pi}}\right)} \mathcal{P}V,\quad P_{h,s}^{\widehat{\pi}_i^{\star}\times\widehat{\pi}_{-i},V}: = \arg\min_{\mathcal{P}\in \mathcal{U}^{R}(P_{h,s}^{\widehat{\pi}_i^{\star}\times\widehat{\pi}_{-i}})} \mathcal{P}V,\nonumber\\
\widehat{P}_{i,h,s}^{\widehat{\pi},V}: &= \arg\min_{\mathcal{P}\in \mathcal{U}^{R}\left(\widehat{P}_{i,h,s}^{\widehat{\pi}}\right)} \mathcal{P}V,
\end{align*}
which represents the probability transition kernel in the uncertainty set $\mathcal{U}^{R}(P_{h,s}^{\widehat{\pi}})$ (resp. $P_{h,s}^{\widehat{\pi}_i^{\star}\times\widehat{\pi}_{-i}}$, $\widehat{P}_{i,h,s}^{\widehat{\pi}}$) that leads to the worst-case value for vector $V\in\mathbb{R}^S$.
It is obvious that the above definitions are equivalent to
\begin{align}
P_{h,s}^{\widehat{\pi},V} &= (1-R)P_{h,s}^{\widehat{\pi}} + R e_{\arg\min V},\quad P_{h,s}^{\widehat{\pi}^{\star}\times\widehat{\pi}_{-i},V} = (1-R)P_{h,s}^{\widehat{\pi}^{\star}\times\widehat{\pi}_{-i}} + R e_{\arg\min V},\label{eq:def-PV-pi}\\
\widehat{P}_{i,h,s}^{\widehat{\pi},V} &= (1-R)\widehat{P}_{i,h,s}^{\widehat{\pi}} + R e_{\arg\min V},\label{eq:def-hatP-pi-V}
\end{align}
where $e_i\in\mathbb{R}^{S}$  denotes the $i$-th standard basis vector.
Then the estimation $\widehat{V}_{i,h}$ in \eqref{eq:cal-hatV} and the $\overline{V}_{i,h}^{\widehat{\pi},R}$ in \eqref{eq:def-bar-V-pi} can be rewritten as
\begin{align}
\widehat{V}_{i,h}(s) &= \min\left\{ \sum_{k=1}^K\alpha_k^K\mathbb{E}_{a_i\sim\pi_{i,h,s}^k}r_{i,h}^k(s,a_i)+\widehat{P}_{i,h,s}^{\widehat{\pi},\widehat{V}_{i,h+1}}\widehat{V}_{i,h+1} + \beta_{i,h}(s), H-h+1\right\},\label{eq:cal-hatV-2}\\
\overline{V}_{i,h}^{\widehat{\pi},R}(s) 
&=\sum_{k=1}^K\alpha_{k}^K\mathbb{E}_{a_i\sim\pi_{i,h,s}^k}r_{i,h}^k(s,a_i) + \widehat{P}_{i,h,s}^{\widehat{\pi},\overline{V}_{i,h+1}^{\widehat{\pi},R}}\overline{V}_{i,h+1}^{\widehat{\pi},R}.\label{eq:cal-barVhatpi-2}
\end{align}

Before diving into the proof details, we first present a technical lemma that states some properties about $\widehat{V}_{i,h}$.
Its proof is deferred to Appendix~\ref{subsec:proof-lem-qihs}.

% The following lemma states some properties about step-size $\alpha_k$, which is proved in Lemma 1 in \citet{li2022minimax}.
% \begin{lemma}\label{lem:alpha}
% Assume that $c_{\alpha} \ge 24$, and $K\ge c_\alpha \log K+1$. For any $k\ge 1$, we have
% \begin{align*}
% \sum_{k=1}^K\alpha_k^K = 1,\qquad \max_{1\le k\le K}\alpha_{k}^K\le \frac{2c_{\alpha}\log K}{K},\qquad \max_{1\le k \le K/2}\alpha_k^K\le \frac{1}{K^6}.
% \end{align*}
% \end{lemma}

\begin{lemma}\label{lem:qihs-1}
For any $i\in[m]$, $1\le h\le H$, $s\in\mathcal{S}$ and $1\le k\le K$, assume that $K$ satisfies \eqref{eq:thm-K}. We have
\begin{align}
\widehat{V}_{i,h} - \min\widehat{V}_{i,h} &\le 3\sum_{h'=h}^H(1-R)^{h'-h}\le 3\min\left\{H,\frac{1}{R}\right\}.\label{eq:upper-widehatV}
\end{align}
\end{lemma}

\subsection{Proof of Theorem \ref{thm:upper}}

According to the definition of $\mathsf{gap}_{\mathsf{CCE}}(\widehat{\pi})$ in \eqref{eq:gap-multi}, it suffices to prove that for all $i\in[m]$,
\begin{align*}
V_{i,1}^{\star,\widehat{\pi}_{-i},R}(s) - V_{i,1}^{\widehat{\pi},R}(s) \le \varepsilon, \quad \forall s\in\mathcal{S}.
\end{align*}
Towards this, we make the following decomposition:
\begin{align}
&\quad V_{i,h}^{\star,\widehat{\pi}_{-i},R} - V_{i,h}^{\widehat{\pi},R} \nonumber\\
&= V_{i,h}^{\widehat{\pi}_i^{\star}\times\widehat{\pi}_{-i},R} - \overline{V}_{i,h}^{\widehat{\pi}_i^{\star}\times\widehat{\pi}_{-i},R} + \overline{V}_{i,h}^{\widehat{\pi}_i^{\star}\times\widehat{\pi}_{-i},R} - \widehat{V}_{i,h} + \widehat{V}_{i,h} - \overline{V}_{i,h}^{\widehat{\pi},R} + \overline{V}_{i,h}^{\widehat{\pi},R} - V_{i,h}^{\widehat{\pi},R},\label{eq:proof-thm-decom}
\end{align}
where we use the fact that $V_{i,h}^{\star,\widehat{\pi}_{-i},R}=V_{i,h}^{\widehat{\pi}_i^{\star}\times\widehat{\pi}_{-i},R}$ (cf. \eqref{eq:def-pistari}).
% According to the definition of $\widehat{\pi}$, we have that for all $1\le h\le H$,
% \begin{align*}
% V_{i,h}^{\star,\widehat{\pi}_{-i},R}\nonumber\\
% V_{i,h}^{\widehat{\pi},R}
% \end{align*}
% with $V_{i,H+1}^{\star,\widehat{\pi}_{-i},R} =0 $, and $V_{i,H+1}^{\widehat{\pi},R} = 0$.

Below we shall complete the proof of Theorem \ref{thm:upper} in four steps: we first bound each term in the RHS of \eqref{eq:proof-thm-decom} separately, and then combine the results.

\paragraph{Step 1. bounding $V_{i,h}^{\widehat{\pi}_i^{\star}\times\widehat{\pi}_{-i},R} - \overline{V}_{i,h}^{\widehat{\pi}_i^{\star}\times\widehat{\pi}_{-i},R}$ and $\overline{V}_{i,h}^{\widehat{\pi},R} - V_{i,h}^{\widehat{\pi},R}$.}

We start from decomposing the difference term $\overline{V}_{i,h}^{\widehat{\pi},R} - V_{i,h}^{\widehat{\pi},R}$.
According to the definition of $V_{i,h}^{\widehat{\pi},R}$ and policy $\widehat{\pi}$, we have
\begin{align*}
V_{i,h}^{\widehat{\pi},R}(s) &= \sum_{k=1}^K\alpha_{k}^K \mathbb{E}_{a_j\sim\pi_{j,h,s}^k}\left[r_{i,h}(s,{\bm a}) + (1-R) P_{h,s,{\bm a}}^0 V_{i,h+1}^{\widehat{\pi},R} + R\min V_{i,h+1}^{\widehat{\pi},R}\right]\nonumber\\
&=
%\mathbb{E}_{a_i\sim\pi_{i,h,s}^k}
r_{i,h}^{\widehat{\pi}}(s) + (1-R) P_{h,s}^{\widehat{\pi}} V_{i,h+1}^{\widehat{\pi},R} + R\min V_{i,h+1}^{\widehat{\pi},R}\nonumber\\
&=r_{i,h}^{\widehat{\pi}}(s) +P_{h,s}^{\widehat{\pi},V_{i,h+1}^{\widehat{\pi},R}}V_{i,h+1}^{\widehat{\pi},R},
\end{align*}
where $r_{i,h}^{\widehat{\pi}}$ and $P_{h,s}^{\widehat{\pi}}$ are defined in \eqref{eq:def-r-pi} and \eqref{eq:def-P-pi}, respectively, and $P_{h,s}^{\widehat{\pi},V}$ is defined in \eqref{eq:def-PV-pi}.

For convenience, we introduce an error term as 
\begin{align}\label{eq:def-zeta-pi}
\zeta_{i,h}^{\widehat{\pi},R} = \overline{V}_{i,h}^{\widehat{\pi},R} - r_{i,h}^{\widehat{\pi}} - P_{h}^{\widehat{\pi},\overline{V}_{i,h+1}^{\widehat{\pi},R}}\overline{V}_{i,h+1}^{\widehat{\pi},R}.
\end{align}
According to the definition of $\overline{V}_{i,h}^{\widehat{\pi},R}$, we have
\begin{align*}
\overline{V}_{i,h}^{\widehat{\pi},R} - V_{i,h}^{\widehat{\pi},R} &=
\zeta_{i,h}^{\widehat{\pi},R} + r_{i,h}^{\widehat{\pi}} + P_{h}^{\widehat{\pi},\overline{V}_{i,h+1}^{\widehat{\pi},R}}\overline{V}_{i,h+1}^{\widehat{\pi},R}- V_{i,h}^{\widehat{\pi},R}\nonumber\\
&=\zeta_{i,h}^{\widehat{\pi},R} +P_{h}^{\widehat{\pi},\overline{V}_{i,h+1}^{\widehat{\pi},R}}\overline{V}_{i,h+1}^{\widehat{\pi},R}- P_{h}^{\widehat{\pi},V_{i,h+1}^{\widehat{\pi},R}}V_{i,h+1}^{\widehat{\pi},R}\nonumber\\
&\overset{(i)}{\le}\zeta_{i,h}^{\widehat{\pi},R} + P_{h}^{\widehat{\pi},{V}_{i,h+1}^{\widehat{\pi},R}}\left(\overline{V}_{i,h+1}^{\widehat{\pi},R} - V_{i,h+1}^{\widehat{\pi},R}\right)\le\sum_{h'=h}^H\prod_{j=h}^{h'-1}P_{j}^{\widehat{\pi},{V}_{i,j+1}^{\widehat{\pi},R}}\zeta_{i,h'}^{\widehat{\pi},R},
%\sum_{k=1}^K\alpha_{k}^K\mathbb{E}_{a_i\sim\pi_{i,h,s}^k}\left[(r_{i,h}^k-r_{i,h}^{\widehat{\pi}_{-i}})(s,a_i) + (1-R)(P_{i,h}^k-P_{h}^{\widehat{\pi}_{-i}})(s,a_i)\overline{V}_{i,h+1}^{\widehat{\pi},R} + R\min\overline{V}_{i,h+1}^{\widehat{\pi},R}\right]
\end{align*}
where (i) uses the fact that $P_{h}^{\widehat{\pi},V}V\le P_{h}^{\widehat{\pi},V'}V$ for any pair of value functions $V$ and $V'$.
Similarly, the term $V_{i,h}^{\widehat{\pi}_i^{\star}\times\widehat{\pi}_{-i},R} - \overline{V}_{i,h}^{\widehat{\pi}_i^{\star}\times\widehat{\pi}_{-i},R}$ can be decomposed by introducing the error term
\begin{align}\label{eq:def-zeta-pi-star}
\zeta_{i,h}^{\widehat{\pi}_i^{\star}\times\widehat{\pi}_{-i},R} = \overline{V}_{i,h}^{\widehat{\pi}_i^{\star}\times\widehat{\pi}_{-i},R} - r_{i,h}^{\widehat{\pi}_i^{\star}\times\widehat{\pi}_{-i}} - P_{h}^{\widehat{\pi}_i^{\star}\times\widehat{\pi}_{-i},\overline{V}_{i,h+1}^{\widehat{\pi}_i^{\star}\times\widehat{\pi}_{-i},R}}\overline{V}_{i,h+1}^{\widehat{\pi}_i^{\star}\times\widehat{\pi}_{-i},R},
\end{align}
where $r_{i,h}^{\widehat{\pi}_i^{\star}\times\widehat{\pi}_{-i}}$ and $P_{h}^{\widehat{\pi}_i^{\star}\times\widehat{\pi}_{-i},V}$ are defined in \eqref{eq:def-r-pi-star} and \eqref{eq:def-P-pi-star}, respectively, and $P_{h}^{\widehat{\pi}_i^{\star}\times\widehat{\pi}_{-i},V}$ is defined in \eqref{eq:def-PV-pi}.
The the term $V_{i,h}^{\widehat{\pi}_i^{\star}\times\widehat{\pi}_{-i},R} - \overline{V}_{i,h}^{\widehat{\pi}_i^{\star}\times\widehat{\pi}_{-i},R}$ is decomposed as 
\begin{align*}
V_{i,h}^{\widehat{\pi}_i^{\star}\times\widehat{\pi}_{-i},R} - \overline{V}_{i,h}^{\widehat{\pi}_i^{\star}\times\widehat{\pi}_{-i},R}
&= -\zeta_{i,h}^{\widehat{\pi}_i^{\star}\times\widehat{\pi}_{-i},R} +  V_{i,h}^{\widehat{\pi}_i^{\star}\times\widehat{\pi}_{-i},R} - r_{i,h}^{\widehat{\pi}_i^{\star}\times\widehat{\pi}_{-i}} - P_{h}^{\widehat{\pi}_i^{\star}\times\widehat{\pi}_{-i},\overline{V}_{i,h+1}^{\widehat{\pi}_i^{\star}\times\widehat{\pi}_{-i},R}}\overline{V}_{i,h+1}^{\widehat{\pi}_i^{\star}\times\widehat{\pi}_{-i},R}\nonumber\\
&=-\zeta_{i,h}^{\widehat{\pi}_i^{\star}\times\widehat{\pi}_{-i},R} + P_{h}^{\widehat{\pi}_i^{\star}\times\widehat{\pi}_{-i},{V}_{i,h+1}^{\widehat{\pi}_i^{\star}\times\widehat{\pi}_{-i},R}}V_{i,h+1}^{\widehat{\pi}_i^{\star}\times\widehat{\pi}_{-i},R} - P_{h}^{\widehat{\pi}_i^{\star}\times\widehat{\pi}_{-i},\overline{V}_{i,h+1}^{\widehat{\pi}_i^{\star}\times\widehat{\pi}_{-i},R}}\overline{V}_{i,h+1}^{\widehat{\pi}_i^{\star}\times\widehat{\pi}_{-i},R}\nonumber\\
&\le -\zeta_{i,h}^{\widehat{\pi}_i^{\star}\times\widehat{\pi}_{-i},R} + P_{h}^{\widehat{\pi}_i^{\star}\times\widehat{\pi}_{-i},\overline{V}_{i,h+1}^{\widehat{\pi}_i^{\star}\times\widehat{\pi}_{-i},R}}\left({V}_{i,h+1}^{\widehat{\pi}_i^{\star}\times\widehat{\pi}_{-i},R} - \overline{V}_{i,h+1}^{\widehat{\pi}_i^{\star}\times\widehat{\pi}_{-i},R}\right)\nonumber\\
& =-\sum_{h'=h}^H\prod_{j=h}^{h'-1}P_{j}^{\widehat{\pi}_i^{\star}\times\widehat{\pi}_{-i},\overline{V}_{i,j+1}^{\widehat{\pi}_i^{\star}\times\widehat{\pi}_{-i},R}}\zeta_{i,h'}^{\widehat{\pi}_i^{\star}\times\widehat{\pi}_{-i},R}. 
\end{align*}

The following lemma provides an upper bound for our decomposed terms.
Its proof is deferred to Appendix~\ref{subsec:proof-lem-bound-error-term}.
\begin{lemma}\label{lem:bound-error-term}
Assume that $K$ satisfies \eqref{eq:thm-K}.
With probability at least $1-\delta$, we have
\begin{align}
\left\|\sum_{h=1}^H\prod_{j=1}^{h-1}P_{j}^{\widehat{\pi},{V}_{i,j+1}^{\widehat{\pi},R}}\zeta_{i,h}^{\widehat{\pi},R}\right\|_{\infty}
&\lesssim  H\sqrt{\frac{\min\{H,\frac{1}{R}\}\log^3\frac{KS\sum_{i=1}^mA_i}{\delta}}{K}},
\label{eq:lem-bound-error-term-pi}\\
\left\|\sum_{h=1}^H\prod_{j=1}^{h-1}P_{j}^{\widehat{\pi}_i^{\star}\times\widehat{\pi}_{-i},\overline{V}_{i,j+1}^{\widehat{\pi}_i^{\star}\times\widehat{\pi}_{-i},R}}\zeta_{i,h}^{\widehat{\pi}_i^{\star}\times\widehat{\pi}_{-i},R}\right\|_{\infty} 
&\lesssim  H\sqrt{\frac{\min\{H,\frac{1}{R}\}\log^3\frac{KS\sum_{i=1}^mA_i}{\delta}}{K}}.\label{eq:lem-bound-error-term-pi-star}
\end{align}
\end{lemma}

According to Lemma \ref{lem:bound-error-term}, we have
\begin{align}
\overline{V}_{i,1}^{\widehat{\pi},R} - {V}_{i,1}^{\widehat{\pi},R}
&\lesssim H\sqrt{\frac{\min\{H,\frac{1}{R}\}\log^3\frac{KS\sum_{i=1}^mA_i}{\delta}}{K}}, \label{eq:proof-thm-term-1}\\
{V}_{i,1}^{\widehat{\pi}_i^{\star}\times\widehat{\pi}_{-i},R} - \overline{V}_{i,1}^{\widehat{\pi}_i^{\star}\times\widehat{\pi}_{-i},R}
&\lesssim H\sqrt{\frac{\min\{H,\frac{1}{R}\}\log^3\frac{KS\sum_{i=1}^mA_i}{\delta}}{K}}.\label{eq:proof-thm-term-2}
\end{align}

\paragraph{Step 2. bounding $\widehat{V}_{i,h} - \overline{V}_{i,h}^{\widehat{\pi},R}$.}

According to the expression of $\widehat{V}_{i,h}$ and $\overline{V}_{i,h}^{\widehat{\pi},R}$ in \eqref{eq:cal-hatV-2} and \eqref{eq:cal-barVhatpi-2}, we have
\begin{align}
\widehat{V}_{i,h} - \overline{V}_{i,h}^{\widehat{\pi},R}
&\le  \widehat{P}_{i,h}^{\widehat{\pi},\widehat{V}_{i,h+1}}\widehat{V}_{i,h+1}-\widehat{P}_{i,h}^{\widehat{\pi},\overline{V}_{i,h+1}^{\widehat{\pi},R}}\overline{V}_{i,h+1}^{\widehat{\pi},R}+ \beta_{i,h}\nonumber\\
&\le\widehat{P}_{i,h}^{\widehat{\pi},\overline{V}_{i,h+1}^{\widehat{\pi},R}}\left(\widehat{V}_{i,h+1}-\overline{V}_{i,h+1}^{\widehat{\pi},R}\right)+ \beta_{i,h}\le \sum_{h'=h}^H\prod_{j = h}^{h'-1}\widehat{P}_{i,j}^{\widehat{\pi},\overline{V}_{i,j+1}^{\widehat{\pi},R}} \beta_{i,h'},\label{eq:proof-thm-step2-temp-8}
\end{align}
where $\widehat{P}_{i,h}^{\widehat{\pi},V}$ is defined in \eqref{eq:def-hatP-pi-V}.

The following lemma shows how to decompose $\beta_{i,h}$.
Its proof is deferred to Appendix \ref{subsec:proof-lem-beta}.
\begin{lemma}\label{lem:beta}
Define $\widetilde{\widehat{V}}_{i,h+1} = {\widehat{V}}_{i,h+1} - \min{\widehat{V}}_{i,h+1}$.
For $\beta_{i,h}$ defined in \eqref{eq:def-beta}, we have
\begin{align}\label{eq:lem-beta-1}
\beta_{i,h}\le 2c_{\mathsf{b}}\sqrt{\frac{\log^3\frac{KS\sum_{i=1}^m A_i}{\delta}}{K\min\left\{H,\frac{1}{R}\right\}}}\left(\widehat{P}_{i,h}^{\widehat{\pi},\widehat{V}_{i,h+1}}\left(\widetilde{\widehat{V}}_{i,h+1}\circ \widetilde{\widehat{V}}_{i,h+1}\right) + \frac{3}{2}\min\left\{H,\frac{1}{R}\right\}\right).
\end{align}
Moreover, if $\widetilde{\widehat{V}}_{i,h+1}\le 3\min\{H,\frac{1}{R}\}$, then we have
\begin{align}\label{eq:lem-beta-2}
\beta_{i,h}
&\le 4c_{\mathsf{b}}\sqrt{\frac{\log^3\frac{KS\sum_{i=1}^m A_i}{\delta}}{K\min\left\{H,\frac{1}{R}\right\}}}\left((1-R)\widehat{P}_{i,h}^{\widehat{\pi}}\left(\widetilde{\widehat{V}}_{i,h+1}\circ \widetilde{\widehat{V}}_{i,h+1}\right) - \widetilde{\widehat{V}}_{i,h}\circ \widetilde{\widehat{V}}_{i,h}\right)\nonumber\\
&\quad+ 14c_{\mathsf{b}}\sqrt{\frac{\min\left\{H,\frac{1}{R}\right\}\log^3\frac{KS\sum_{i=1}^m A_i}{\delta}}{K}}.
\end{align}
\end{lemma}

By using Lemma \ref{lem:qihs-1}, we have \eqref{eq:lem-beta-2} holds. 
Recalling the definition of $\widehat{P}_{i,h}^{\widehat{\pi},V}$, for any $V,V'\in\mathbb{R}^S$, we have $\widehat{P}_{i,h}^{\widehat{\pi},V}V \le \widehat{P}_{i,h}^{\widehat{\pi},V'}V$.
Thus we have
\begin{align*}
\beta_{i,h}
&\le 4c_{\mathsf{b}}\sqrt{\frac{\log^3\frac{KS\sum_{i=1}^m A_i}{\delta}}{K\min\left\{H,\frac{1}{R}\right\}}}\left[\left(\widehat{P}_{i,h}^{\widehat{\pi},\overline{V}_{i,h+1}^{\widehat{\pi},R}}\left(\widetilde{\widehat{V}}_{i,h+1}\circ \widetilde{\widehat{V}}_{i,h+1}\right) - \widetilde{\widehat{V}}_{i,h}\circ \widetilde{\widehat{V}}_{i,h}\right) +\frac{7}{2}\min\left\{H,\frac{1}{R}\right\}\right].
% \nonumber\\
% &\quad + 14c_{\mathsf{b}}\sqrt{\frac{\min\left\{H,\frac{1}{R}\right\}\log^3\frac{KS\sum_{i=1}^m A_i}{\delta}}{K}}.
\end{align*}
Inserting it into \eqref{eq:proof-thm-step2-temp-8}, we have
\begin{align}
\widehat{V}_{i,h} - \overline{V}_{i,h}^{\widehat{\pi},R}
&\le  4c_{\mathsf{b}}\sqrt{\frac{\log^3\frac{KS\sum_{i=1}^m A_i}{\delta}}{K\min\left\{H,\frac{1}{R}\right\}}}\sum_{h'=h}^H\prod_{j = h}^{h'-1}\widehat{P}_{i,j}^{\widehat{\pi},\overline{V}_{i,j+1}^{\widehat{\pi},R}}\Bigg(\widehat{P}_{i,h'}^{\widehat{\pi},\overline{V}_{i,h'+1}^{\widehat{\pi},R}}\left(\widetilde{\widehat{V}}_{i,h'+1}\circ \widetilde{\widehat{V}}_{i,h'+1}\right)\nonumber\\
&\qquad\qquad - \widetilde{\widehat{V}}_{i,h'}\circ \widetilde{\widehat{V}}_{i,h'}\Bigg)
%+ \frac{7}{2}H\min\left\{H,\frac{1}{R}\right\}\right]
+ 14c_{\mathsf{b}}H\sqrt{\frac{\min\left\{H,\frac{1}{R}\right\}\log^3\frac{KS\sum_{i=1}^m A_i}{\delta}}{K}}\nonumber\\
 &\overset{(i)}{\lesssim} H\sqrt{\frac{\min\left\{H,\frac{1}{R}\right\}\log^3\frac{KS\sum_{i=1}^m A_i}{\delta}}{K}},\label{eq:proof-thm-term-3}
\end{align}
where (i) holds because 
\begin{align*}
&\quad\sum_{h'=h}^H\prod_{j = h}^{h'-1}\widehat{P}_{i,j}^{\widehat{\pi},\overline{V}_{i,j+1}^{\widehat{\pi},R}}\left(\widehat{P}_{i,h'}^{\widehat{\pi},\overline{V}_{i,h'+1}^{\widehat{\pi},R}}\left(\widetilde{\widehat{V}}_{i,h'+1}\circ \widetilde{\widehat{V}}_{i,h'+1}\right) - \widetilde{\widehat{V}}_{i,h'}\circ \widetilde{\widehat{V}}_{i,h'}\right)\nonumber\\
&=\sum_{h'=h}^H\prod_{j = h}^{h'}\widehat{P}_{i,j}^{\widehat{\pi},\overline{V}_{i,j+1}^{\widehat{\pi},R}}\left(\widetilde{\widehat{V}}_{i,h'+1}\circ \widetilde{\widehat{V}}_{i,h'+1}\right) - \sum_{h'=h}^H\prod_{j = h}^{h'-1}\widehat{P}_{i,j}^{\widehat{\pi},\overline{V}_{i,j+1}^{\widehat{\pi},R}}\widetilde{\widehat{V}}_{i,h'}\circ \widetilde{\widehat{V}}_{i,h'}\nonumber\\
&=-\widetilde{\widehat{V}}_{i,h}\circ \widetilde{\widehat{V}}_{i,h}\le 0.
\end{align*}

\paragraph{Step 3. bounding $\overline{V}_{i,h}^{\widehat{\pi}_i^{\star}\times\widehat{\pi}_{-i},R} - \widehat{V}_{i,h}$.}

This step uses the following lemma,
which states that $\widehat{V}_{i,h}$ is an optimism estimation of $\overline{V}_{i,h}^{\star,\widehat{\pi}_{-i},R}$.
Its proof is deferred to Appendix \ref{sec:proof-lem-hatV-ge-barV}.

\begin{lemma}\label{lem:hatV-ge-barV}
For $\overline{V}_{i,h}^{\star,\widehat{\pi}_{-i},R}$ defined in \eqref{eq:def-bar-V-star}, and $K$ satisfying \eqref{eq:thm-K}, with probability at least $1-\delta$, we have for all $i\in [m]$ and $1\le h\le H$,
\begin{align}\label{eq:lem-hatV-ge-barV}
\widehat{V}_{i,h}\ge \overline{V}_{i,h}^{\star,\widehat{\pi}_{-i},R}.
\end{align}
\end{lemma}

According to definition of $\overline{V}_{i,h}^{\widehat{\pi}_i^{\star}\times\widehat{\pi}_{-i},R}$, $\overline{V}_{i,h}^{\star,\widehat{\pi}_{-i},R}$, and $ \widehat{V}_{i,h}$, and combining with Lemma \ref{lem:hatV-ge-barV}, we have
\begin{align}\label{eq:proof-thm-term-4}
\overline{V}_{i,h}^{\widehat{\pi}_i^{\star}\times\widehat{\pi}_{-i},R} - \widehat{V}_{i,h}\le\overline{V}_{i,h}^{\star,\widehat{\pi}_{-i},R} - \widehat{V}_{i,h} \le 0.
\end{align}

\paragraph{Step 4. combining the above results.}

Inserting \eqref{eq:proof-thm-term-1}, \eqref{eq:proof-thm-term-2}, \eqref{eq:proof-thm-term-3}, and \eqref{eq:proof-thm-term-4} into \eqref{eq:proof-thm-decom}, we have
\begin{align}
V_{i,h}^{\star,\widehat{\pi}_{-i},R} - V_{i,h}^{\widehat{\pi},R} \le CH\sqrt{\frac{\min\left\{H,\frac{1}{R}\right\}\log^3\frac{KS\sum_{i=1}^m A_i}{\delta}}{K}}\le \varepsilon,
\end{align}
as long as $K$ satisfying \eqref{eq:thm-K},
where $C$ is a sufficiently large constant.
Thus we complete the proof.

\section{Discussion}
\label{sec:discussion}

This paper investigates multi-player robust Markov games in a finite-horizon setting, and achieve minimax sample complexity through establishing matching upper and lower bounds.
Specifically, we use an $R$-contamination set with uncertainty level $R$ as uncertainty set for the transition probabilities. We then propose an algorithm for RMGs based on Q-FTRL, and demonstrate that it achieves an $\varepsilon$-robust CCE with a sample complexity of $H^3S\sum_{i=1}^mA_i\min\{H,1/R\}/\varepsilon^2$ (up to log factors), where $S$ denotes the number of states, $A_i$ is the number of actions for the $i$-th player, and $H$ represents the horizon length.
In addition, an $\varepsilon$-robust NE is achieved in the special case of two-player zero-sum RMGs with the same sample complexity.
Moreover, we establish an information-theoretic lower bound that matches the sample complexity of the proposed algorithm, demonstrating its minimax optimality.
%To the best of our knowledge, this is the first algorithm that achieves minimax optimality for RMGs. 
Compared to previous results, our result addresses the problems of curse of multiagency and the barrier of long horizon.
%improves the sample complexity by reducing the dependence on the product $\prod_{i=1}^m A_i$ to the sum $\sum_{i=1}^m A_i$ and eliminates the burn-in cost.

Future work will focus on extending this algorithm to handle other types of uncertainty sets, such as those defined by total variation distance or Kullback-Leibler divergence. Additionally, we will investigate the possibility of developing algorithms that can achieve minimax-optimal sample complexity, as demonstrated in this paper, and meanwhile avoid relying on flexible generative models. This could include approaches based on local access models \citep{li2021sampleefficient,yin2022efficient} or online sampling protocols \citep{jin2018q,azar2017minimax}.

\appendix
\section{Proof of technical lemmas}
\label{sec:proof}

\subsection{Preliminary lemmas}

Before diving into the proof details, we first present some technical lemmas that play an important role in the proof of Theorem \ref{thm:upper}.

The following lemma states some properties about step-size $\alpha_k$, which is proved in Lemma 1 in \citet{li2022minimax}.
\begin{lemma}\label{lem:alpha}
Assume that $c_{\alpha} \ge 24$, and $K\ge c_\alpha \log K+1$. For any $k\ge 1$, we have
\begin{align*}
\sum_{k=1}^K\alpha_k^K = 1,\qquad \max_{1\le k\le K}\alpha_{k}^K\le \frac{2c_{\alpha}\log K}{K},\qquad \max_{1\le k \le K/2}\alpha_k^K\le \frac{1}{K^6}.
\end{align*}
\end{lemma}

The following lemma restates the properties of $\widehat{V}_{i,h}$ from Lemma \ref{lem:qihs-1}, and adds some properties about $q_{i,h,s}^k$.
Its proof is deferred to Appendix~\ref{subsec:proof-lem-qihs}.
\begin{lemma}\label{lem:qihs}
For any $i\in[m]$, $1\le h\le H$, $s\in\mathcal{S}$ and $1\le k\le K$, assume that $K$ satisfies \eqref{eq:thm-K}. We have
\begin{align}
\widehat{V}_{i,h} - \min\widehat{V}_{i,h} &\le 3\sum_{h'=h}^H(1-R)^{h'-h}\le 3\min\left\{H,\frac{1}{R}\right\},\nonumber\\
\mathsf{Var}_{\pi_{i,h,s}^k}(q_{i,h,s}^k) &\lesssim \left(\min\left\{H,\frac{1}{R}\right\}\right)^2,\quad
 q_{i,h,s}^k-\min q_{i,h,s}^k \lesssim \min\left\{H,\frac{1}{R}\right\}.\label{eq:upperbound-varq}
\end{align}
\end{lemma}

\subsection{Proof of Lemma \ref{lem:qihs-1} and Lemma \ref{lem:qihs}}
\label{subsec:proof-lem-qihs}

\paragraph{Proof of \eqref{eq:upper-widehatV}.}
The key of this proof is to establish \eqref{eq:upper-widehatV} by mathematical induction.
When $h=H+1$, \eqref{eq:upper-widehatV} holds obviously.

Assume that \eqref{eq:upper-widehatV} holds for all $h'\ge h+1$.
Now we intend to prove that \eqref{eq:upper-widehatV} holds for $h$.
To this end, we make the following observation.
\begin{align}\label{eq:poof-lem-qihs-temp-1}
\widetilde{\widehat{V}}_{i,h} 
&= {\widehat{V}}_{i,h} - \min{\widehat{V}}_{i,h} \le 1 + (1-R)\widehat{P}_{i,h}^{\widehat{\pi}}{\widehat{V}}_{i,h+1} + R\min{\widehat{V}}_{i,h+1}-\min{\widehat{V}}_{i,h} + \beta_{i,h}\nonumber\\
&\overset{(i)}{\le}1 + (1-R)\widehat{P}_{i,h}^{\widehat{\pi}}\widetilde{\widehat{V}}_{i,h+1} + \beta_{i,h}\nonumber\\
&\le \sum_{h'=h}^H(1-R)^{h'-h}\prod_{j=h}^{h'-1}\widehat{P}_{i,j}^{\widehat{\pi}}\left(1+\beta_{i,h'}\right)\nonumber\\
&=\sum_{h'=h}^H(1-R)^{h'-h} + \sum_{h'=h}^H(1-R)^{h'-h}\prod_{j=h}^{h'-1}\widehat{P}_{i,j}^{\widehat{\pi}}\beta_{i,h'},
\end{align}
where (i) uses the fact that
\begin{align*}
\min\widehat{V}_{i,h} 
&\ge \min\left\{(1-R)\min\widehat{V}_{i,h+1} + R\min\widehat{V}_{i,h+1} + \beta_{i,h},H-h+1\right\} \nonumber\\
&\ge \min\left\{\min\widehat{V}_{i,h+1} +\beta_{i,h},H-h+1\right\} \ge\min\widehat{V}_{i,h+1}.
\end{align*}
Define
$$
\rho_0 = c_{\mathsf{b}}\sqrt{\frac{\log^3\frac{KS\sum_{i=1}^m A_i}{\delta}}{K\min\left\{H,\frac{1}{R}\right\}}}.
$$
According to Lemma \ref{lem:beta}, we have
\begin{align*}
\beta_{i,h}
&\le 2\rho_0\widehat{P}_{i,h}^{\widehat{\pi},\widehat{V}_{i,h+1}}\left(\widetilde{\widehat{V}}_{i,h+1}\circ \widetilde{\widehat{V}}_{i,h+1}\right) + 3\rho_0\min\left\{H,\frac{1}{R}\right\}\nonumber\\
&= 2\rho_0(1-R)\widehat{P}_{i,h}^{\widehat{\pi}}\left(\widetilde{\widehat{V}}_{i,h+1}\circ \widetilde{\widehat{V}}_{i,h+1}\right) + 3\rho_0\min\left\{H,\frac{1}{R}\right\},
\end{align*}
Recalling our assumption that $\widetilde{\widehat{V}}_{i,h'} \le 3\min\{H,\frac{1}{R}\}$ for $h'\ge h+1$, and by using \eqref{eq:lem-beta-2}, we have
\begin{align*}
\beta_{i,h'}
&\le 4\rho_0\left((1-R)\widehat{P}_{i,h'}^{\widehat{\pi}}\left(\widetilde{\widehat{V}}_{i,h'+1}\circ \widetilde{\widehat{V}}_{i,h'+1}\right) - \widetilde{\widehat{V}}_{i,h'}\circ \widetilde{\widehat{V}}_{i,h'}\right)+ 14\rho_0\min\left\{H,\frac{1}{R}\right\}, \quad h'\ge h+1,
\end{align*}

Collecting the above inequalities, we have
\begin{align}
&\quad\sum_{h'=h}^H(1-R)^{h'-h}\prod_{j=h}^{h'-1}\widehat{P}_{i,j}^{\widehat{\pi}}\beta_{i,h'}\nonumber\\
&\le4\rho_0\sum_{h'=h+1}^H(1-R)^{h'-h}\prod_{j=h}^{h'-1}\widehat{P}_{i,j}^{\widehat{\pi}}\left((1-R)\widehat{P}_{i,h'}^{\widehat{\pi}}\left(\widetilde{\widehat{V}}_{i,h'+1}\circ \widetilde{\widehat{V}}_{i,h'+1}\right) - \widetilde{\widehat{V}}_{i,h'}\circ \widetilde{\widehat{V}}_{i,h'}\right)\nonumber\\
&\quad +2\rho_0(1-R)\widehat{P}_{i,h}^{\widehat{\pi}}\left(\widetilde{\widehat{V}}_{i,h+1}\circ \widetilde{\widehat{V}}_{i,h+1}\right) + 14\rho_0\min\left\{H,\frac{1}{R}\right\}\sum_{h'=h+1}^H(1-R)^{h'-h}+3\rho_0\min\left\{H,\frac{1}{R}\right\}\nonumber\\
&\overset{(i)}{\le} -4\rho_0(1-R)\widehat{P}_{i,h}^{\widehat{\pi}}\widetilde{\widehat{V}}_{i,h+1}\circ \widetilde{\widehat{V}}_{i,h+1}
+2\rho_0(1-R)\widehat{P}_{i,h}^{\widehat{\pi}}\left(\widetilde{\widehat{V}}_{i,h+1}\circ \widetilde{\widehat{V}}_{i,h+1}\right)\nonumber\\
&\quad + 14\rho_0\min\left\{H,\frac{1}{R}\right\}\sum_{h'=h}^H(1-R)^{h'-h}\nonumber\\
&\le 14\rho_0\min\left\{H,\frac{1}{R}\right\}\sum_{h'=h}^H(1-R)^{h'-h},\label{eq:poof-lem-qihs-temp-2}
\end{align}
where (i) uses the fact that
\begin{align*}
&\quad \sum_{h'=h+1}^H(1-R)^{h'-h}\prod_{j=h}^{h'-1}\widehat{P}_{i,j}^{\widehat{\pi}}\left((1-R)\widehat{P}_{i,h'}^{\widehat{\pi}}\left(\widetilde{\widehat{V}}_{i,h'+1}\circ \widetilde{\widehat{V}}_{i,h'+1}\right) - \widetilde{\widehat{V}}_{i,h'}\circ \widetilde{\widehat{V}}_{i,h'}\right)\nonumber\\
&=\sum_{h'=h+1}^H(1-R)^{h'-h+1}\prod_{j=h}^{h'}\widehat{P}_{i,j}^{\widehat{\pi}}\widetilde{\widehat{V}}_{i,h'+1}\circ \widetilde{\widehat{V}}_{i,h'+1} - \sum_{h'=h+1}^H(1-R)^{h'-h}\prod_{j=h}^{h'-1}\widehat{P}_{i,j}^{\widehat{\pi}}\widetilde{\widehat{V}}_{i,h'}\circ \widetilde{\widehat{V}}_{i,h'}\nonumber\\
&=\sum_{h'=h+2}^{H+1}(1-R)^{h'-h}\prod_{j=h}^{h'-1}\widehat{P}_{i,j}^{\widehat{\pi}}\widetilde{\widehat{V}}_{i,h'}\circ \widetilde{\widehat{V}}_{i,h'} - \sum_{h'=h+1}^H(1-R)^{h'-h}\prod_{j=h}^{h'-1}\widehat{P}_{i,j}^{\widehat{\pi}}\widetilde{\widehat{V}}_{i,h'}\circ \widetilde{\widehat{V}}_{i,h'}\nonumber\\
&=-(1-R)\widehat{P}_{i,h+1}^{\widehat{\pi}}\widetilde{\widehat{V}}_{i,h+1}\circ \widetilde{\widehat{V}}_{i,h+1}.
\end{align*}

Inserting \eqref{eq:poof-lem-qihs-temp-2} into \eqref{eq:poof-lem-qihs-temp-1}, we have
\begin{align*}
\widetilde{\widehat{V}}_{i,h} 
&\le\sum_{h'=h}^H(1-R)^{h'-h} +  14\rho_0\min\left\{H,\frac{1}{R}\right\}\sum_{h'=h}^H(1-R)^{h'-h}\le 3\sum_{h'=h}^H(1-R)^{h'-h},
\end{align*}
as long as $14\rho_0\min\left\{H,\frac{1}{R}\right\}\le 2$ for $K$ in \eqref{eq:thm-K}.
Thus we complete the proof of \eqref{eq:upper-widehatV}.

\paragraph{Proof of \eqref{eq:upperbound-varq}.}
With \eqref{eq:upper-widehatV} established, now we are ready to prove \eqref{eq:upperbound-varq}.
According to the definition of $q_{i,h,s}^k$, we have
\begin{align*}
\min q_{i,h,s}^k \ge \min q_{i,h}^k \ge (1-R)P_{i,h}^k\widehat{V}_{i,h+1} + R\min\widehat{V}_{i,h+1}\ge \min\widehat{V}_{i,h+1}.
\end{align*}
Thus we have
\begin{align*}
q_{i,h,s}^k - \min q_{i,h,s}^k 
&\le q_{i,h,s}^k - \min q_{i,h}^k \le 1 + (1-R)P_{i,h}^k\left(\widehat{V}_{i,h+1}-\min\widehat{V}_{i,h+1}\right)\nonumber\\
&= 1 + (1-R)P_{i,h}^k\widetilde{\widehat{V}}_{i,h+1}
\le 1 + 3(1-R)\sum_{h'=h+1}^H(1-R)^{h'-h-1}\nonumber\\
&\le3\sum_{h'=h}^H(1-R)^{h'-h}\le 3\min\left\{H,\frac{1}{R}\right\}.
\end{align*}
Moreover, it is obviously that 
%\begin{align*}
$
\mathsf{Var}_{\pi_{i,h,s}^k}(q_{i,h,s}^k)\le 9\min\left\{H,\frac{1}{R}\right\}^2.
$
%\end{align*}
Thus we complete the proof.

\subsection{Proof of Lemma \ref{lem:bound-error-term}}
\label{subsec:proof-lem-bound-error-term}

Before refined analysis, we first demonstrate that due to the introduction of robustness, the $\overline{V}_{i,h+1}^{\widehat{\pi},R}$ and $\overline{V}_{i,h+1}^{\widehat{\pi}_i^{\star}\times \widehat{\pi}_{-i},R}$ has an upper bound related to $R$.
Notice that
\begin{align}
\overline{V}_{i,h}^{\widehat{\pi},R}(s) - \min \overline{V}_{i,h}^{\widehat{\pi},R} &\le 1 + (1-R)\sum_{k=1}^K\alpha_k^K\mathbb{E}_{a_i\sim\pi_{i,h,s}^k}P_{i,h}^k(s,a_i)\overline{V}_{i,h+1}^{\widehat{\pi},R} + R\min \overline{V}_{i,h+1}^{\widehat{\pi},R} - \overline{V}_{i,h}^{\widehat{\pi},R}\nonumber\\
&\le 1 + (1-R)\sum_{k=1}^K\alpha_k^K\mathbb{E}_{a_i\sim\pi_{i,h,s}^k}P_{i,h}^k(s,a_i)\left(\overline{V}_{i,h+1}^{\widehat{\pi},R} - \min\overline{V}_{i,h+1}^{\widehat{\pi},R}\right),\label{eq:proof-lem-bound-error-temp-1}
\end{align}
where the last inequality holds because
\begin{align}\label{eq:proof-lem-bound-error-temp-5}
\min\overline{V}_{i,h}^{\widehat{\pi},R} \ge \sum_{k=1}^K\alpha_k^K\mathbb{E}_{a_i\sim\pi_{i,h,s}^k}(1-R)P_{i,h}^k(s,a_i)\overline{V}_{i,h+1}^{\widehat{\pi},R} + R\min\overline{V}_{i,h+1}^{\widehat{\pi},R} \ge \min\overline{V}_{i,h+1}^{\widehat{\pi},R}.
\end{align}
Recalling the definition of $\widehat{P}_{i,h}^{\widehat{\pi}}$ in \eqref{eq:def-hatP-pi},
inequality \eqref{eq:proof-lem-bound-error-temp-1} can be rewritten as
\begin{align}
\overline{V}_{i,h}^{\widehat{\pi},R} - \min \overline{V}_{i,h}^{\widehat{\pi},R} &\le 1 + (1-R)\widehat{P}_{i,h}^{\widehat{\pi}}\left(\overline{V}_{i,h+1}^{\widehat{\pi},R} - \min\overline{V}_{i,h+1}^{\widehat{\pi},R}\right)\le\sum_{h'=h}^H(1-R)^{h'-h}.\label{eq:barV-pi-upper}
\end{align}

Similarly, by replacing $\widehat{\pi}$ with $\widehat{\pi}_i^{\star}\times\widehat{\pi}_{-i}$, we have
\begin{align}\label{eq:barV-pi-star-upper}
\overline{V}_{i,h}^{\widehat{\pi}_i^{\star}\times\widehat{\pi}_{-i},R} - \min \overline{V}_{i,h}^{\widehat{\pi}_i^{\star}\times\widehat{\pi}_{-i},R} &\le 1 + (1-R)\widehat{P}_{i,h}^{\widehat{\pi}_i^{\star}\times\widehat{\pi}_{-i}}\left(\overline{V}_{i,h+1}^{\widehat{\pi}_i^{\star}\times\widehat{\pi}_{-i},R} - \min\overline{V}_{i,h+1}^{\widehat{\pi}_i^{\star}\times\widehat{\pi}_{-i},R}\right)\nonumber\\
&\le\sum_{h'=h}^H(1-R)^{h'-h},
\end{align}
where the $s$-th row of matrix $\widehat{P}_{i,h}^{\widehat{\pi}_i^{\star}\times\widehat{\pi}_{-i}}\in\mathbb{R}^{S\times S}$ is defined as
\begin{align*}
\widehat{P}_{i,h,s}^{\widehat{\pi}_i^{\star}\times\widehat{\pi}_{-i}}=
\sum_{k=1}^K\alpha_k^K\mathbb{E}_{a_i\sim\widehat{\pi}_{i,h}^{\star}}P_{i,h}^k(s,a_i).
\end{align*}

Next, we shall present the proof for \eqref{eq:lem-bound-error-term-pi}.
%For convenience, in the following proof, we shall replace $\widehat{\pi}$ or $\widehat{\pi}_{i}^{\star}\times\widehat{\pi}_{-i}$ with ${\pi}$.
%\paragraph{Proof of \eqref{eq:lem-bound-error-term-pi}}
We start the proof by looking at a single error term $\zeta_{i,h}^{\widehat{\pi},R}$.
According to its definition in \eqref{eq:def-zeta-pi}, we have
\begin{align}
\zeta_{i,h}^{\widehat{\pi},R}(s) &= \sum_{k=1}^K{\alpha}_k^K\mathbb{E}_{a_i\sim\pi_{i,h,s}^k}\left[r_{i,h}^k(s,a_i) + P_{i,h}^{k,\overline{V}^{\widehat{\pi},R}_{i,h+1}}(s,a_i)\overline{V}^{\widehat{\pi},R}_{i,h+1}\right] - r_{i,h}^{\widehat{\pi}} - P_{h,s}^{\widehat{\pi},\overline{V}_{i,h+1}^{\widehat{\pi},R}}\overline{V}_{i,h+1}^{\widehat{\pi},R}\nonumber\\
&\overset{(i)}{=}\sum_{k=1}^K{\alpha}_k^K\mathbb{E}_{a_i\sim\pi_{i,h,s}^k}\left[r_{i,h}^k(s,a_i)-\mathbb{E}_{k-1}r_{i,h}^k(s,a_i) + (1-R)\left(P_{i,h}^{k}-\mathbb{E}_{k-1}P_{i,h}^{k}\right)(s,a_i)\overline{V}^{\widehat{\pi},R}_{i,h+1}\right]
\end{align}
where $P_{i,h}^{k,V}(s,a_i)\in\mathbb{R}^{S}$ is defined as
\begin{align*}
P_{i,h}^{k,V}(s,a_i) = (1-R)P_{i,h}^{k}(s,a_i) + Re_{\arg\min V},
\end{align*}
(i) uses the definition of $r_{i,h}^{\widehat{\pi}}$ and $P_{h}^{\widehat{\pi}}$ in \eqref{eq:def-r-pi-2} and \eqref{eq:def-P-pi-2}, and $\mathbb{E}_{k-1}[\cdot]$ denotes the expectation with history policies $\pi_{i,h}^j$ for $j\le k$ conditioned. 

Next we intend to use Freedman's inequality to bound this error term, which is stated as below (c.f. Theorem 5 in \citet{li2022minimax}).
\begin{lemma}\label{lem:freedman}
Suppose that $Y_n =
\sum_{k=1}^n X_k\in\mathbb{R}$, where $\{X_k\}$ is a real-valued scalar sequence obeying
\begin{align*}
|X_k|\le R,\quad {\rm and}\quad \mathbb{E}[X_k|\{X_j\}_{j\le k}] = 0\quad \forall k\ge 1
\end{align*}
for some quantity $R > 0$. Define
$$
W_n:=\sum_{k=1}^n\mathbb{E}_{k-1}[X_k^2],
$$
where $\mathbb{E}_{k-1}[\cdot]$ stands for the expectation conditional on $\{X_j\}_{j<k}$. Consider any arbitrary quantity $\kappa > 0$.
With probability at least $1 - \delta$, one has
$$
|Y_n| \le \sqrt{8W_n\log\frac{3n}{\delta}} + 5R\log\frac{3n}{\delta}
\le \kappa W_n + \left(\frac{2}{\kappa}+5R\right)\log\frac{3n}{\delta}.
$$
\end{lemma}

To apply Freedman's inequality, we define
\begin{align*}
X_k =& \alpha_k^K\mathbb{E}_{a_i\sim\pi_{i,h,s}^k}\left[r_{i,h}^k(s,a_i)-\mathbb{E}_{k-1}r_{i,h}^k(s,a_i) + (1-R)\left(P_{i,h}^{k}-\mathbb{E}_{k-1}P_{i,h}^{k}\right)(s,a_i)\overline{V}^{\widehat{\pi},R}_{i,h+1}\right]\nonumber\\
\overset{(i)}{=}& \alpha_k^K\mathbb{E}_{a_i\sim\pi_{i,h,s}^k}\left[r_{i,h}^k(s,a_i)-\mathbb{E}_{k-1}r_{i,h}^k(s,a_i) + (1-R)\left(P_{i,h}^{k}-\mathbb{E}_{k-1}P_{i,h}^{k}\right)(s,a_i)\widetilde{\overline{V}}^{\widehat{\pi},R}_{i,h+1}\right],
\end{align*}
where
\begin{align*}
\widetilde{\overline{V}}^{\widehat{\pi},R}_{i,h+1} = \overline{V}^{\widehat{\pi},R}_{i,h+1} - \min\overline{V}^{\widehat{\pi},R}_{i,h+1},
\end{align*}
and (i) holds because
%\begin{align*}
$
\left(P_{i,h}^{k}-\mathbb{E}_{k-1}P_{i,h}^{k}\right)\min\overline{V}^{\widehat{\pi},R}_{i,h+1} = 0.
$
%\end{align*}
It is obvious that
\begin{align*}
|X_k|&\le \alpha_k^K\left(1 + (1-R)\left(\overline{V}^{\widehat{\pi},R}_{i,h+1}-\min\overline{V}^{\widehat{\pi},R}_{i,h+1}\right)\right) \nonumber\\
&\le \max_k \alpha_k^K \min\left\{H,\frac{1}{R}\right\} \overset{(i)}{\le} \frac{2c_\alpha\log K}{K}\min\left\{H,\frac{1}{R}\right\},\nonumber\\
%1 + (1-R)\sum_{h'=h}^H(1-R)^{h'-h} \le \sum_{h=1}^H(1-R)^{h} \le \min\left\{H,\frac{1}{R}\right\},\nonumber\\
\mathbb{E}_{k-1}[X_k] &= 0.
\end{align*}
where (i) uses Lemma \ref{lem:alpha}.
Moreover, we have
\begin{align*}
W_K &= \sum_{k=1}^K\mathbb{E}_{k-1}X_k^2
= \sum_{k=1}^K(\alpha_k^K)^2\mathsf{Var}_{k-1}(\widetilde{q}_{i,h,s}^k)\nonumber\\
&\le \max_k \alpha_k^K\sum_{k=1}^K\alpha_k^K\mathsf{Var}_{k-1}(\widetilde{q}_{i,h,s}^k) \le \frac{2c_\alpha\log K}{K}\sum_{k=1}^K\alpha_k^K\mathsf{Var}_{k-1}(\widetilde{q}_{i,h,s}^k),
\end{align*}
where $\mathsf{Var}_{k-1}(\widetilde{q}_{i,h,s}^k)$ denotes the variance of random variable $\widetilde{q}_{i,h,s}^k$ conditioned on $\pi_{i,h}^j$ with $j\le k$, and
$$
\widetilde{q}_{i,h,s}^k := \mathbb{E}_{a_i\sim{\pi}_{i,h,s}^k}\left[r_{i,h}^k(s,a_i)+ (1-R)P_{i,h}^{k}(s,a_i)\widetilde{\overline{V}}^{\widehat{\pi},R}_{i,h+1}\right].
$$
According to Freedman's inequality (c.f. Lemma \ref{lem:freedman}), by taking $\kappa = \sqrt{\frac{K\log\frac{KS\sum_{i=1}^mA_i}{\delta}}{\min\{H,\frac{1}{R}\}}}$, we have
\begin{align}\label{eq:proof-lem-bound-error-temp-8}
\left|\zeta_{i,h}^{\widehat{\pi},R}\right| 
&\le 2c_\alpha\sqrt{\frac{\log^3 \frac{KS\sum_{i=1}^mA_i}{\delta}}{K\min\{H,\frac{1}{R}\}}}\sum_{k=1}^K\alpha_k^K\mathsf{Var}_{k-1}(\widetilde{q}_{i,h,s}^k) \nonumber\\
&\quad + \left(2\sqrt{\frac{\min\{H,\frac{1}{R}\}}{K\log\frac{KS\sum_{i=1}^mA_i}{\delta}}} + \frac{10 c_{\alpha}\min\{H,\frac{1}{R}\}\log K}{K}\right)\log\frac{3KS\sum_{i=1}^mA_i}{\delta}\nonumber\\
&\le 2c_\alpha\sqrt{\frac{\log^3\frac{KS\sum_{i=1}^mA_i}{\delta}}{K\min\{H,\frac{1}{R}\}}}\sum_{k=1}^K\alpha_k^K\mathsf{Var}_{k-1}(\widetilde{q}_{i,h,s}^k) 
+ c\sqrt{\frac{\min\{H,\frac{1}{R}\}\log\frac{KS\sum_{i=1}^mA_i}{\delta}}{K}},
\end{align}
where $c$ is a constant. 

% We introduce an auxiliary random vector $\widehat{q}_{i,h}^k\in\mathbb{R}^{S}$ with the $s$-th entry defined as
% $$
% \widehat{q}_{i,h,s}^k = \widetilde{q}_{i,h,s}^k ()
% $$
Now we focus on the term $\sum_{k=1}^K\alpha_k^K\mathsf{Var}_{k-1}(\widetilde{q}_{i,h,s}^k)$.
According to its definition, we have
\begin{align}\label{eq:proof-lem-bound-error-temp-10}
\sum_{k=1}^K\alpha_k^K\mathsf{Var}_{k-1}(\widetilde{q}_{i,h,s}^k) 
&\le 2\sum_{k=1}^K\alpha_k^K\mathsf{Var}_{k-1}(r_{i,h,s}^k) + 2\sum_{k=1}^K\alpha_k^K\mathsf{Var}_{k-1}(\widetilde{q}_{i,h,s}^k-r_{i,h,s}^k)\nonumber\\
&\overset{(i)}{\le} 2 + 2\sum_{k=1}^K\alpha_k^K\mathsf{Var}_{k-1}(\widetilde{q}_{i,h,s}^k-r_{i,h,s}^k),
\end{align}
where $r_{i,h,s}^k:=\mathbb{E}_{a_i\sim{\pi}_{i,h,s}^k}r_{i,h}^k(s,a_i)$,
and (i) holds because $r_{i,h,s}^k$ is bounded with $|r_{i,h,s}^k|\le 1$.
Moreover, we have
\begin{align}
&\quad \sum_{k=1}^K\alpha_k^K\mathsf{Var}_{k-1}(\widetilde{q}_{i,h,s}^k-r_{i,h,s}^k) \nonumber\\
&= \sum_{k=1}^K\alpha_k^K\mathbb{E}_{k-1}(\widetilde{q}_{i,h,s}^k-r_{i,h,s}^k)^2 - \sum_{k=1}^K\alpha_k^K\left(\mathbb{E}_{k-1}(\widetilde{q}_{i,h,s}^k-r_{i,h,s}^k)\right)^2.\label{eq:proof-lem-bound-error-temp-4}
\end{align}
For the expectation of square, we have
\begin{align}
\sum_{k=1}^K\alpha_k^K\mathbb{E}_{k-1}(\widetilde{q}_{i,h,s}^k-r_{i,h,s}^k)^2 
&= \sum_{k=1}^K\alpha_k^K\mathbb{E}_{k-1}\left[\mathbb{E}_{a_i\sim{\pi}_{i,h,s}^k}\left[(1-R)P_{i,h}^{k}(s,a_i)\widetilde{\overline{V}}^{\widehat{\pi},R}_{i,h+1}\right]\right]^2 \nonumber\\
&\le \sum_{k=1}^K\alpha_k^K\mathbb{E}_{k-1,a_i\sim{\pi}_{i,h,s}^k}\left[(1-R)P_{i,h}^{k}(s,a_i)\widetilde{\overline{V}}^{\widehat{\pi},R}_{i,h+1}\right]^2\nonumber\\
&\overset{(i)}{\le} (1-R)^2{P}_{h,s}^{\widehat{\pi}}\left(\widetilde{\overline{V}}^{\widehat{\pi},R}_{i,h+1}\circ \widetilde{\overline{V}}^{\widehat{\pi},R}_{i,h+1}\right)\nonumber\\
&\le {P}_{h,s}^{\widehat{\pi},{\overline{V}}^{\widehat{\pi},R}_{i,h+1}}\left(\widetilde{\overline{V}}^{\widehat{\pi},R}_{i,h+1}\circ \widetilde{\overline{V}}^{\widehat{\pi},R}_{i,h+1}\right),\label{eq:proof-lem-bound-error-temp-2}
\end{align}
where (i) uses the fact that each row of $P_{i,h}^{k}(s,a_i)$ has only a non-zero value, and the vector ${P}_{h,s}^{\widehat{\pi}}\in\mathbb{R}^{S}$ is defined in \eqref{eq:def-P-pi}.
% \begin{align*}
% \overline{P}_{h,s}^{\widehat{\pi}}=
% \sum_{k=1}^K\alpha_k^K\mathbb{E}_{a_j\sim\pi_{j,h,s}^k}P_{h}(s,{\bm a}).
% \end{align*}
For square of the expectation, we have
\begin{align}
\sum_{k=1}^K\alpha_{k}^K\left(\mathbb{E}_{k-1}(\widetilde{q}_{i,h,s}^k-r_{i,h,s}^k)\right)^2 
&\ge (1-R)^2\sum_{k=1}^K\alpha_{k}^K\left(\mathbb{E}_{k-1,a_i\sim{\pi}_{i,h,s}^k}P_{i,h}^{k}(s,a_i)\widetilde{\overline{V}}^{\widehat{\pi},R}_{i,h+1}\right)^2 \nonumber\\
&\overset{(i)}{\ge}(1-R)^2\left(\sum_{k=1}^K\alpha_{k}^K\mathbb{E}_{k-1,a_i\sim{\pi}_{i,h,s}^k}P_{i,h}^{k}(s,a_i)\widetilde{\overline{V}}^{\widehat{\pi},R}_{i,h+1}\right)^2\nonumber\\
&= (1-R)^2\left({P}_{h,s}^{\widehat{\pi}}\widetilde{\overline{V}}^{\widehat{\pi},R}_{i,h+1}\right)^2=\left({P}_{h,s}^{\widehat{\pi},{\overline{V}}^{\widehat{\pi},R}_{i,h+1}}\widetilde{\overline{V}}^{\widehat{\pi},R}_{i,h+1}\right)^2,\label{eq:proof-lem-bound-error-temp-3}
\end{align}
where (i) arises from Jensen's inequality.
Inserting \eqref{eq:proof-lem-bound-error-temp-2} and \eqref{eq:proof-lem-bound-error-temp-3} into \eqref{eq:proof-lem-bound-error-temp-4}, we have
\begin{align}\label{eq:proof-lem-bound-error-temp-6}
\sum_{k=1}^K\alpha_k^K\mathsf{Var}_{k-1}(\widetilde{q}_{i,h,s}^k-r_{i,h,s}^k) \le {P}_{h,s}^{\widehat{\pi},{\overline{V}}^{\widehat{\pi},R}_{i,h+1}}\left(\widetilde{\overline{V}}^{\widehat{\pi},R}_{i,h+1}\circ \widetilde{\overline{V}}^{\widehat{\pi},R}_{i,h+1}\right)-\left({P}_{h,s}^{\widehat{\pi},{\overline{V}}^{\widehat{\pi},R}_{i,h+1}}\widetilde{\overline{V}}^{\widehat{\pi},R}_{i,h+1}\right)^2.
\end{align}

Define the error term
\begin{align}\label{eq:def-vartheta}
\vartheta_{i,h}^{\widehat{\pi}} := \widetilde{\overline{V}}^{\widehat{\pi},R}_{i,h} - {P}_{h}^{\widehat{\pi},\overline{V}^{\widehat{\pi},R}_{i,h+1}}\widetilde{\overline{V}}^{\widehat{\pi},R}_{i,h+1}.
\end{align}
% where $\overline{P}_{h}^{\widehat{\pi}}\in\mathbb{R}^{S\times S}$ collects $\overline{P}_{h,s}^{\widehat{\pi}}$ as the $s$-th row, and $\overline{P}_{h}^{\widehat{\pi},\overline{V}^{\widehat{\pi},R}_{i,h+1}}$ is the counterpart of $\overline{P}_{h}^{\widehat{\pi}}$ as defined in Section \ref{subsec:notations}.
According to its definition, this error term is upper bounded by
\begin{align*}
\vartheta_{i,h}^{\widehat{\pi}} 
&={\overline{V}}^{\widehat{\pi},R}_{i,h} - {P}_{h}^{\widehat{\pi},\overline{V}^{\widehat{\pi},R}_{i,h+1}}{\overline{V}}^{\widehat{\pi},R}_{i,h+1} - \min{\overline{V}}^{\widehat{\pi},R}_{i,h} +\min{\overline{V}}^{\widehat{\pi},R}_{i,h+1}\nonumber\\
&\overset{(i)}{\le}{\overline{V}}^{\widehat{\pi},R}_{i,h} - {P}_{h}^{\widehat{\pi},\overline{V}^{\widehat{\pi},R}_{i,h+1}}{\overline{V}}^{\widehat{\pi},R}_{i,h+1}\nonumber\\
&\overset{(ii)}{\le} r_{i,h}^{\widehat{\pi}} + \zeta_{i,h}^{\widehat{\pi},R}\le 1 + \left|\zeta_{i,h}^{\widehat{\pi},R}\right|.
\end{align*}
where (i) arises from the fact that $\min{\overline{V}}^{\widehat{\pi},R}_{i,h}\ge \min{\overline{V}}^{\widehat{\pi},R}_{i,h+1}$ as shown in \eqref{eq:proof-lem-bound-error-temp-5}, and (ii) arises from the definition of $\zeta_{i,h}^{\widehat{\pi},R}$ in \eqref{eq:def-zeta-pi}.

By applying the definition \eqref{eq:def-vartheta}, we have
\begin{align}
&\quad {P}_{h}^{\widehat{\pi},{\overline{V}}^{\widehat{\pi},R}_{i,h+1}}\left(\widetilde{\overline{V}}^{\widehat{\pi},R}_{i,h+1}\circ \widetilde{\overline{V}}^{\widehat{\pi},R}_{i,h+1}\right)-\left({P}_{h}^{\widehat{\pi},{\overline{V}}^{\widehat{\pi},R}_{i,h+1}}\widetilde{\overline{V}}^{\widehat{\pi},R}_{i,h+1}\right)\circ\left({P}_{h}^{\widehat{\pi},{\overline{V}}^{\widehat{\pi},R}_{i,h+1}}\widetilde{\overline{V}}^{\widehat{\pi},R}_{i,h+1}\right) \nonumber\\
&= {P}_{h}^{\widehat{\pi},{\overline{V}}^{\widehat{\pi},R}_{i,h+1}}\left(\widetilde{\overline{V}}^{\widehat{\pi},R}_{i,h+1}\circ \widetilde{\overline{V}}^{\widehat{\pi},R}_{i,h+1}\right)-\left(\widetilde{\overline{V}}^{\widehat{\pi},R}_{i,h}-\vartheta_{i,h}^{\widehat{\pi}}\right)\circ\left(\widetilde{\overline{V}}^{\widehat{\pi},R}_{i,h}-\vartheta_{i,h}^{\widehat{\pi}}\right)\nonumber\\
&\le {P}_{h}^{\widehat{\pi},{\overline{V}}^{\widehat{\pi},R}_{i,h+1}}\left(\widetilde{\overline{V}}^{\widehat{\pi},R}_{i,h+1}\circ \widetilde{\overline{V}}^{\widehat{\pi},R}_{i,h+1}\right)-\widetilde{\overline{V}}^{\widehat{\pi},R}_{i,h}\circ \widetilde{\overline{V}}^{\widehat{\pi},R}_{i,h}+2\vartheta_{i,h}^{\widehat{\pi}}\circ\widetilde{\overline{V}}^{\widehat{\pi},R}_{i,h}\nonumber\\
&\le {P}_{h}^{\widehat{\pi},{\overline{V}}^{\widehat{\pi},R}_{i,h+1}}\left(\widetilde{\overline{V}}^{\widehat{\pi},R}_{i,h+1}\circ \widetilde{\overline{V}}^{\widehat{\pi},R}_{i,h+1}\right)-\widetilde{\overline{V}}^{\widehat{\pi},R}_{i,h}\circ \widetilde{\overline{V}}^{\widehat{\pi},R}_{i,h}+2\min\left\{H,\frac{1}{R}\right\}\left(1 + \left|\zeta_{i,h}^{\widehat{\pi},R}\right|\right)\nonumber\\
&\le{P}_{h}^{\widehat{\pi},{{V}}^{\widehat{\pi},R}_{i,h+1}}\left(\widetilde{\overline{V}}^{\widehat{\pi},R}_{i,h+1}\circ \widetilde{\overline{V}}^{\widehat{\pi},R}_{i,h+1}\right)-\widetilde{\overline{V}}^{\widehat{\pi},R}_{i,h}\circ \widetilde{\overline{V}}^{\widehat{\pi},R}_{i,h}+2\min\left\{H,\frac{1}{R}\right\}\left(1 + \left|\zeta_{i,h}^{\widehat{\pi},R}\right|\right).\label{eq:proof-lem-bound-error-temp-7}
\end{align}

Inserting \eqref{eq:proof-lem-bound-error-temp-7} into \eqref{eq:proof-lem-bound-error-temp-6} and \eqref{eq:proof-lem-bound-error-temp-10}, we have
% \begin{align*}
% &\quad\sum_{k=1}^K\alpha_k^K\mathsf{Var}_{k-1}(\widetilde{q}_{i,h,s}^k-r_{i,h,s}^k) \nonumber\\
% &\le {P}_{h}^{\widehat{\pi},{{V}}^{\widehat{\pi},R}_{i,h+1}}\left(\widetilde{\overline{V}}^{\widehat{\pi},R}_{i,h+1}\circ \widetilde{\overline{V}}^{\widehat{\pi},R}_{i,h+1}\right)-\widetilde{\overline{V}}^{\widehat{\pi},R}_{i,h}\circ \widetilde{\overline{V}}^{\widehat{\pi},R}_{i,h}+2
% \min\left\{H,\frac{1}{R}\right\} + 2\left|\zeta_{i,h}^{\widehat{\pi},R}\right|\min\left\{H,\frac{1}{R}\right\}.
% \end{align*}
%
% Inserting \eqref{eq:proof-lem-bound-error-temp-6} into \eqref{eq:proof-lem-bound-error-temp-10}, we have
\begin{align*}
&\quad\sum_{k=1}^K\alpha_k^K\mathsf{Var}_{k-1}(\widetilde{q}_{i,h,s}^k) \le 2{P}_{h}^{\widehat{\pi},{{V}}^{\widehat{\pi},R}_{i,h+1}}\left(\widetilde{\overline{V}}^{\widehat{\pi},R}_{i,h+1}\circ \widetilde{\overline{V}}^{\widehat{\pi},R}_{i,h+1}\right)-2\widetilde{\overline{V}}^{\widehat{\pi},R}_{i,h}\circ \widetilde{\overline{V}}^{\widehat{\pi},R}_{i,h}\nonumber\\
&\qquad +6
\min\left\{H,\frac{1}{R}\right\}\left(1 + \frac23
\left|\zeta_{i,h}^{\widehat{\pi},R}\right|\right).
\end{align*}

Combining with \eqref{eq:proof-lem-bound-error-temp-8}, we have
\begin{align*}
\left|\zeta_{i,h}^{\widehat{\pi},R}\right| 
&\le 4c_\alpha\sqrt{\frac{\log^3\frac{KS\sum_{i=1}^mA_i}{\delta}}{K\min\{H,\frac{1}{R}\}}}\left({P}_{h}^{\widehat{\pi},{{V}}^{\widehat{\pi},R}_{i,h+1}}\left(\widetilde{\overline{V}}^{\widehat{\pi},R}_{i,h+1}\circ \widetilde{\overline{V}}^{\widehat{\pi},R}_{i,h+1}\right)-\widetilde{\overline{V}}^{\widehat{\pi},R}_{i,h}\circ \widetilde{\overline{V}}^{\widehat{\pi},R}_{i,h}\right) \nonumber\\
&\quad + 4c_\alpha\sqrt{\frac{\log^3\frac{KS\sum_{i=1}^mA_i}{\delta}}{K\min\{H,\frac{1}{R}\}}}\left(\left(3+\frac{c}{4c_\alpha}\right)
\min\left\{H,\frac{1}{R}\right\} + 2\left|\zeta_{i,h}^{\widehat{\pi},R}\right|\min\left\{H,\frac{1}{R}\right\}\right) 
%+ c\sqrt{\frac{\min\{H,\frac{1}{R}\}\log\frac{KS\sum_{i=1}^mA_i}{\delta}}{K}}
\nonumber\\
&\le4c_\alpha\sqrt{\frac{\log^3\frac{KS\sum_{i=1}^mA_i}{\delta}}{K\min\{H,\frac{1}{R}\}}}\left({P}_{h}^{\widehat{\pi},{{V}}^{\widehat{\pi},R}_{i,h+1}}\left(\widetilde{\overline{V}}^{\widehat{\pi},R}_{i,h+1}\circ \widetilde{\overline{V}}^{\widehat{\pi},R}_{i,h+1}\right)-\widetilde{\overline{V}}^{\widehat{\pi},R}_{i,h}\circ \widetilde{\overline{V}}^{\widehat{\pi},R}_{i,h}\right)\nonumber\\
&\quad + c'\sqrt{\frac{\min\{H,\frac{1}{R}\}\log^3\frac{KS\sum_{i=1}^mA_i}{\delta}}{K}} + \frac12\left|\zeta_{i,h}^{\widehat{\pi},R}\right|,
\end{align*}
where $c' = 12c_\alpha + c$, and the last inequality holds because $4c_\alpha\sqrt{\frac{\min\{H,\frac{1}{R}\}\log^3\frac{KS\sum_{i=1}^mA_i}{\delta}}{K}}\le \frac14$ as long as $K$ satisfying \eqref{eq:thm-K}.
Thus we have
\begin{align}\label{eq:proof-lem-bound-error-temp-9}
\left|\zeta_{i,h}^{\widehat{\pi},R}\right| 
&\le8c_\alpha\sqrt{\frac{\log^3\frac{KS\sum_{i=1}^mA_i}{\delta}}{K\min\{H,\frac{1}{R}\}}}\left({P}_{h}^{\widehat{\pi},{{V}}^{\widehat{\pi},R}_{i,h+1}}\left(\widetilde{\overline{V}}^{\widehat{\pi},R}_{i,h+1}\circ \widetilde{\overline{V}}^{\widehat{\pi},R}_{i,h+1}\right)-\widetilde{\overline{V}}^{\widehat{\pi},R}_{i,h}\circ \widetilde{\overline{V}}^{\widehat{\pi},R}_{i,h}\right)\nonumber\\
&\quad+ 2c'\sqrt{\frac{\min\{H,\frac{1}{R}\}\log^3\frac{KS\sum_{i=1}^mA_i}{\delta}}{K}}.
\end{align}

By applying \eqref{eq:proof-lem-bound-error-temp-9}, we have
\begin{align*}
&\quad \left\|\sum_{h=1}^H\prod_{j=1}^{h-1}P_{j}^{\widehat{\pi},{V}_{i,j+1}^{\widehat{\pi},R}}\zeta_{i,h}^{\widehat{\pi},R}\right\|_{\infty}
\le \sum_{h=1}^H\prod_{j=1}^{h-1}P_{j}^{\widehat{\pi},{V}_{i,j+1}^{\widehat{\pi},R}}\left|\zeta_{i,h}^{\widehat{\pi},R}\right|\nonumber\\
&\lesssim\sum_{h=1}^H\prod_{j=1}^{h-1}P_{j}^{\widehat{\pi},{V}_{i,j+1}^{\widehat{\pi},R}}\Bigg(\sqrt{\frac{\log^3\frac{KS\sum_{i=1}^mA_i}{\delta}}{K\min\{H,\frac{1}{R}\}}}\left({P}_{h}^{\widehat{\pi},{{V}}^{\widehat{\pi},R}_{i,h+1}}\left(\widetilde{\overline{V}}^{\widehat{\pi},R}_{i,h+1}\circ \widetilde{\overline{V}}^{\widehat{\pi},R}_{i,h+1}\right)-\widetilde{\overline{V}}^{\widehat{\pi},R}_{i,h}\circ \widetilde{\overline{V}}^{\widehat{\pi},R}_{i,h}\right)\nonumber\\
&\qquad\qquad\qquad\qquad\quad+ \sqrt{\frac{\min\{H,\frac{1}{R}\}\log^3\frac{KS\sum_{i=1}^mA_i}{\delta}}{K}}\Bigg) \nonumber\\
&\lesssim \sqrt{\frac{\log^3\frac{KS\sum_{i=1}^mA_i}{\delta}}{K\min\{H,\frac{1}{R}\}}}\sum_{h=1}^H\prod_{j=1}^{h-1}P_{j}^{\widehat{\pi},{V}_{i,j+1}^{\widehat{\pi},R}}\left({P}_{h}^{\widehat{\pi},{{V}}^{\widehat{\pi},R}_{i,h+1}}\left(\widetilde{\overline{V}}^{\widehat{\pi},R}_{i,h+1}\circ \widetilde{\overline{V}}^{\widehat{\pi},R}_{i,h+1}\right)-\widetilde{\overline{V}}^{\widehat{\pi},R}_{i,h}\circ \widetilde{\overline{V}}^{\widehat{\pi},R}_{i,h}\right) \nonumber\\
&\quad+ H\sqrt{\frac{\min\{H,\frac{1}{R}\}\log^3\frac{KS\sum_{i=1}^mA_i}{\delta}}{K}}\nonumber\\
&\lesssim H\sqrt{\frac{\min\{H,\frac{1}{R}\}\log^3\frac{KS\sum_{i=1}^mA_i}{\delta}}{K}},
\end{align*}
where the last inequality holds because 
\begin{align*}
&\quad \sum_{h=1}^H\prod_{j=1}^{h-1}P_{j}^{\widehat{\pi},{V}_{i,j+1}^{\widehat{\pi},R}}\left({P}_{h}^{\widehat{\pi},{{V}}^{\widehat{\pi},R}_{i,h+1}}\left(\widetilde{\overline{V}}^{\widehat{\pi},R}_{i,h+1}\circ \widetilde{\overline{V}}^{\widehat{\pi},R}_{i,h+1}\right)-\widetilde{\overline{V}}^{\widehat{\pi},R}_{i,h}\circ \widetilde{\overline{V}}^{\widehat{\pi},R}_{i,h}\right)\nonumber\\
&=\sum_{h=1}^H\prod_{j=1}^{h}P_{j}^{\widehat{\pi},{V}_{i,j+1}^{\widehat{\pi},R}}\left(\widetilde{\overline{V}}^{\widehat{\pi},R}_{i,h+1}\circ \widetilde{\overline{V}}^{\widehat{\pi},R}_{i,h+1}\right)-\sum_{h=1}^H\prod_{j=1}^{h-1}P_{j}^{\widehat{\pi},{V}_{i,j+1}^{\widehat{\pi},R}}\widetilde{\overline{V}}^{\widehat{\pi},R}_{i,h}\circ \widetilde{\overline{V}}^{\widehat{\pi},R}_{i,h}\nonumber\\
&=\sum_{h=2}^{H+1}\prod_{j=1}^{h-1}P_{j}^{\widehat{\pi},{V}_{i,j+1}^{\widehat{\pi},R}}\left(\widetilde{\overline{V}}^{\widehat{\pi},R}_{i,h}\circ \widetilde{\overline{V}}^{\widehat{\pi},R}_{i,h}\right)-\sum_{h=1}^H\prod_{j=1}^{h-1}P_{j}^{\widehat{\pi},{V}_{i,j+1}^{\widehat{\pi},R}}\widetilde{\overline{V}}^{\widehat{\pi},R}_{i,h}\circ \widetilde{\overline{V}}^{\widehat{\pi},R}_{i,h}\nonumber\\
&=-\widetilde{\overline{V}}^{\widehat{\pi},R}_{i,1}\circ \widetilde{\overline{V}}^{\widehat{\pi},R}_{i,1}\le 0.
\end{align*}
Now we complete the proof.
The proof of \eqref{eq:lem-bound-error-term-pi-star} is similar by replacing ${\pi}_{i,h}^k$ with $\widehat{\pi}_i^{\star}$, and is omitted here.

\subsection{Proof of Lemma \ref{lem:beta}}
\label{subsec:proof-lem-beta}

In the definition of $\beta_{i,h}$ (c.f. \eqref{eq:def-beta}), for convenience, we denote
\begin{align*}
\rho_0 : =c_{\mathsf{b}}\sqrt{\frac{\log^3\frac{KS\sum_{i=1}^m A_i}{\delta}}{K\min\left\{H,\frac{1}{R}\right\}}},\qquad \rho_1 = \rho_0\min\left\{H,\frac{1}{R}\right\}.
\end{align*}
By using these new notations, we could rewrite the definition of $\beta_{i,h}$ as
\begin{align}\label{eq:proof-thm-step2-temp-7}
\beta_{i,h}  
&= \rho_0\sum_{k=1}^K\alpha_{k}^K \mathsf{Var}_{\pi_{i,h,s}^k}(q_{i,h,s}^k) + \rho_1\nonumber\\
&\overset{(i)}{\le} 2\rho_0\sum_{k=1}^K\alpha_{k}^K \mathsf{Var}_{\pi_{i,h,s}^k}(q_{i,h,s}^k-r_{i,h,s}^k) +2\rho_0 + \rho_1,
\end{align}
where (i) uses the fact that
$$
\mathsf{Var}_{\pi_{i,h,s}^k}(q_{i,h,s}^k)\le 2\mathsf{Var}_{\pi_{i,h,s}^k}(r_{i,h,s}^k) + 2\mathsf{Var}_{\pi_{i,h,s}^k}(q_{i,h,s}^k-r_{i,h,s}^k)\le 2+2\mathsf{Var}_{\pi_{i,h,s}^k}(q_{i,h,s}^k-r_{i,h,s}^k),
$$
and $\sum_{k=1}^K\alpha_k^K = 1$.
%\sum_{k=1}^K\alpha_k^K\left\{ + \min\left\{H,\frac{1}{R}\right\}\right\}.
The variance is further decomposed as
\begin{align}\label{eq:proof-thm-step2-temp-1}
&\quad \sum_{k=1}^K\alpha_k^K\mathsf{Var}_{\pi_{i,h,s}^k}\left(q_{i,h,s}^k-r_{i,h,s}^k\right) \nonumber\\
&= \sum_{k=1}^K\alpha_k^K\mathbb{E}_{\pi_{i,h,s}^k}\left(q_{i,h,s}^k-r_{i,h,s}^k\right)^2 - \sum_{k=1}^K\alpha_k^K\left(\mathbb{E}_{\pi_{i,h,s}^k}\left(q_{i,h,s}^k-r_{i,h,s}^k\right)\right)^2.
\end{align}
The expectation of the square is bounded by
\begin{align}\label{eq:proof-thm-step2-temp-2}
\sum_{k=1}^K\alpha_k^K\mathbb{E}_{\pi_{i,h,s}^k}(q_{i,h,s}^k-r_{i,h,s}^k)^2 
&\le  \sum_{k=1}^K\alpha_k^K\mathbb{E}_{\pi_{i,h,s}^k}\left((1-R)P_{i,h}^k\widehat{V}_{i,h+1} + R\min_s \widehat{V}_{i,h+1}\right)^2\nonumber\\
&=  \widehat{P}_{i,h}^{\widehat{\pi},\widehat{V}_{i,h+1}}\left(\widehat{V}_{i,h+1}\circ \widehat{V}_{i,h+1}\right),
\end{align}
and the square of the expectation is bounded by 
\begin{align}\label{eq:proof-thm-step2-temp-3}
\sum_{k=1}^K\alpha_k^K\left(\mathbb{E}_{\pi_{i,h,s}^k}\left(q_{i,h,s}^k-r_{i,h,s}^k\right)\right)^2
&\overset{(i)}{\ge} \left(\sum_{k=1}^K\alpha_k^K\mathbb{E}_{\pi_{i,h,s}^k}\left(q_{i,h,s}^k-r_{i,h,s}^k\right)\right)^2 \nonumber\\
&\ge \left(\sum_{k=1}^K\alpha_k^K\mathbb{E}_{\pi_{i,h,s}^k}\left((1-R)P_{i,h}^k\widehat{V}_{i,h+1} + R\min_s \widehat{V}_{i,h+1}\right)\right)^2\nonumber\\
&=\left(\widehat{P}_{i,h}^{\widehat{\pi},\widehat{V}_{i,h+1}}\widehat{V}_{i,h+1}\right)\circ\left(\widehat{P}_{i,h}^{\widehat{\pi},\widehat{V}_{i,h+1}}\widehat{V}_{i,h+1}\right),
\end{align}
where (i) uses the Jensen's inequality. 
Inserting \eqref{eq:proof-thm-step2-temp-1} and \eqref{eq:proof-thm-step2-temp-2} into \eqref{eq:proof-thm-step2-temp-3}, we have
\begin{align}\label{eq:proof-thm-step2-temp-5}
&\quad\sum_{k=1}^K\alpha_k^K\mathsf{Var}_{\pi_{i,h,s}^k}\left(q_{i,h,s}^k-r_{i,h,s}^k\right)\nonumber\\
&\le \widehat{P}_{i,h}^{\widehat{\pi},\widehat{V}_{i,h+1}}\left(\widehat{V}_{i,h+1}\circ \widehat{V}_{i,h+1}\right) - \left(\widehat{P}_{i,h}^{\widehat{\pi},\widehat{V}_{i,h+1}}\widehat{V}_{i,h+1}\right)\circ\left(\widehat{P}_{i,h}^{\widehat{\pi},\widehat{V}_{i,h+1}}\widehat{V}_{i,h+1}\right)\nonumber\\
&=\widehat{P}_{i,h}^{\widehat{\pi},\widehat{V}_{i,h+1}}\left(\widetilde{\widehat{V}}_{i,h+1}\circ \widetilde{\widehat{V}}_{i,h+1}\right) - \left(\widehat{P}_{i,h}^{\widehat{\pi},\widehat{V}_{i,h+1}}\widetilde{\widehat{V}}_{i,h+1}\right)\circ\left(\widehat{P}_{i,h}^{\widehat{\pi},\widehat{V}_{i,h+1}}\widetilde{\widehat{V}}_{i,h+1}\right),
\end{align}
where
$$
\widetilde{\widehat{V}}_{i,h} := {\widehat{V}}_{i,h} - \min{\widehat{V}}_{i,h}.
$$
Inserting it into \eqref{eq:proof-thm-step2-temp-7}, we have
\begin{align*}
\beta_{i,h}
&\le 2\rho_0\widehat{P}_{i,h}^{\widehat{\pi},\widehat{V}_{i,h+1}}\left(\widetilde{\widehat{V}}_{i,h+1}\circ \widetilde{\widehat{V}}_{i,h+1}\right) + 2\rho_0 + \rho_1\nonumber\\
&\le 2\rho_0\widehat{P}_{i,h}^{\widehat{\pi},\widehat{V}_{i,h+1}}\left(\widetilde{\widehat{V}}_{i,h+1}\circ \widetilde{\widehat{V}}_{i,h+1}\right) + 3\rho_1.
\end{align*}
Recalling the definition of $\rho_0$ and $\rho_1$, we could complete the proof of \eqref{eq:lem-beta-1}.

Next, we shall prove \eqref{eq:lem-beta-2}.
Define 
$$
\widehat{\vartheta}_{i,h} := \widetilde{\widehat{V}}_{i,h} - \widehat{P}_{i,h}^{\widehat{\pi},\widehat{V}_{i,h+1}}\widetilde{\widehat{V}}_{i,h+1}.
$$
Noticing that 
\begin{align*}
\min\widehat{V}_{i,h} 
&\ge \min\left\{(1-R)\min\widehat{V}_{i,h+1} + R\min\widehat{V}_{i,h+1} + \beta_{i,h},H-h+1\right\} \nonumber\\
&\ge \min\left\{\min\widehat{V}_{i,h+1} +\beta_{i,h},H-h+1\right\} \overset{(i)}{\ge} \min\widehat{V}_{i,h+1},
\end{align*}
where (i) uses $\widehat{V}_{i,h+1}\le H-h\le H-h+1$,
we have 
\begin{align}\label{eq:proof-thm-step2-temp-4}
\widehat{\vartheta}_{i,h} = {\widehat{V}}_{i,h} - \widehat{P}_{i,h}^{\widehat{\pi},\widehat{V}_{i,h+1}}{\widehat{V}}_{i,h+1} - \min{\widehat{V}}_{i,h} + \min {\widehat{V}}_{i,h+1} \le {\widehat{V}}_{i,h} - \widehat{P}_{i,h}^{\widehat{\pi},\widehat{V}_{i,h+1}}{\widehat{V}}_{i,h+1}\le 1 + \beta_{i,h}.
\end{align}
By using this, we have
\begin{align*}
&\quad \widehat{P}_{i,h}^{\widehat{\pi},\widehat{V}_{i,h+1}}\left(\widetilde{\widehat{V}}_{i,h+1}\circ \widetilde{\widehat{V}}_{i,h+1}\right) - \left(\widehat{P}_{i,h}^{\widehat{\pi},\widehat{V}_{i,h+1}}\widetilde{\widehat{V}}_{i,h+1}\right)\circ\left(\widehat{P}_{i,h}^{\widehat{\pi},\widehat{V}_{i,h+1}}\widetilde{\widehat{V}}_{i,h+1}\right)\nonumber\\
&=\widehat{P}_{i,h}^{\widehat{\pi},\widehat{V}_{i,h+1}}\left(\widetilde{\widehat{V}}_{i,h+1}\circ \widetilde{\widehat{V}}_{i,h+1}\right) - \left(\widetilde{\widehat{V}}_{i,h}-\widehat{\vartheta}_{i,h}\right)\circ\left(\widetilde{\widehat{V}}_{i,h}-\widehat{\vartheta}_{i,h}\right)\nonumber\\
&\le\widehat{P}_{i,h}^{\widehat{\pi},\widehat{V}_{i,h+1}}\left(\widetilde{\widehat{V}}_{i,h+1}\circ \widetilde{\widehat{V}}_{i,h+1}\right) - \widetilde{\widehat{V}}_{i,h}\circ \widetilde{\widehat{V}}_{i,h}+2\widehat{\vartheta}_{i,h}\circ \widetilde{\widehat{V}}_{i,h}\nonumber\\
&\overset{(i)}{\le}(1-R)\widehat{P}_{i,h}^{\widehat{\pi}}\left(\widetilde{\widehat{V}}_{i,h+1}\circ \widetilde{\widehat{V}}_{i,h+1}\right) - \widetilde{\widehat{V}}_{i,h}\circ \widetilde{\widehat{V}}_{i,h}+2(1+\beta_{i,h})\circ \widetilde{\widehat{V}}_{i,h},
\end{align*}
where (i) uses the definition of $\widehat{P}_{i,h}^{\widehat{\pi}}$ and \eqref{eq:proof-thm-step2-temp-4}.
Inserting it into \eqref{eq:proof-thm-step2-temp-5}, we have
\begin{align}\label{eq:proof-thm-step2-temp-6}
&\quad \sum_{k=1}^K\alpha_k^K\mathsf{Var}_{\pi_{i,h,s}^k}\left(q_{i,h,s}^k-r_{i,h,s}^k\right)\nonumber\\
&\le(1-R)\widehat{P}_{i,h}^{\widehat{\pi}}\left(\widetilde{\widehat{V}}_{i,h+1}\circ \widetilde{\widehat{V}}_{i,h+1}\right) - \widetilde{\widehat{V}}_{i,h}\circ \widetilde{\widehat{V}}_{i,h}+2(1+\beta_{i,h})\circ \widetilde{\widehat{V}}_{i,h}.
\end{align}
Inserting \eqref{eq:proof-thm-step2-temp-6} into \eqref{eq:proof-thm-step2-temp-7}, and noticing that $\widetilde{\widehat{V}}_{i,h}\le 3\min\{H,\frac{1}{R}\}$, we have
\begin{align*}
\beta_{i,h}
&\le 2\rho_0\left((1-R)\widehat{P}_{i,h}^{\widehat{\pi}}\left(\widetilde{\widehat{V}}_{i,h+1}\circ \widetilde{\widehat{V}}_{i,h+1}\right) - \widetilde{\widehat{V}}_{i,h}\circ \widetilde{\widehat{V}}_{i,h}+2\right)\nonumber\\
&\quad + 6\rho_0\beta_{i,h}\min\left\{H,\frac{1}{R}\right\} + 2\rho_0 + \rho_1.
\end{align*}
Recalling that for $K$ satisfying \eqref{eq:thm-K}, we have $6\rho_0\min\left\{H,\frac{1}{R}\right\} \le \frac12$.
Thus we have
\begin{align*}
\beta_{i,h}
&\le 4\rho_0\left((1-R)\widehat{P}_{i,h}^{\widehat{\pi}}\left(\widetilde{\widehat{V}}_{i,h+1}\circ \widetilde{\widehat{V}}_{i,h+1}\right) - \widetilde{\widehat{V}}_{i,h}\circ \widetilde{\widehat{V}}_{i,h}\right) + 12\rho_0 + 2\rho_1\nonumber\\
&\le
4\rho_0\left((1-R)\widehat{P}_{i,h}^{\widehat{\pi}}\left(\widetilde{\widehat{V}}_{i,h+1}\circ \widetilde{\widehat{V}}_{i,h+1}\right) - \widetilde{\widehat{V}}_{i,h}\circ \widetilde{\widehat{V}}_{i,h}\right) + 14\rho_1.
\end{align*}
Then we complete the proof of \eqref{eq:lem-beta-2}.

\subsection{Proof of Lemma \ref{lem:hatV-ge-barV}}
\label{sec:proof-lem-hatV-ge-barV}

We shall prove Lemma \ref{lem:hatV-ge-barV} by mathematical induction.
For $h=H+1$, it is obviously that $\widehat{V}_{i,H+1} = \overline{V}_{i,H+1}^{\star,\widehat{\pi}_{-i},R} = 0$.
Assume that \eqref{eq:lem-hatV-ge-barV} holds for $h+1$, that is, $\widehat{V}_{i,h+1}\ge \overline{V}_{i,h+1}^{\star,\widehat{\pi}_{-i},R}$.
It suffices to prove that \eqref{eq:lem-hatV-ge-barV} holds for $h$.

According to the definition of $\overline{V}_{i,h}^{\star,\widehat{\pi}_{-i},R}$, we have
\begin{align}\label{eq:proof-upper-barV-star}
\overline{V}_{i,h}^{\star,\widehat{\pi}_{-i},R}(s)&=\max_{a_i\in\mathcal{A}_i}\sum_{k=1}^K\alpha_{k}^K\left[r_{i,h}^k(s,a_i) + (1-R)P_{i,h}^k(s,a_i)\overline{V}_{i,h}^{\star,\widehat{\pi}_{-i},R} + R\min\overline{V}_{i,h}^{\star,\widehat{\pi}_{-i},R}\right] \nonumber\\
&\le \max_{a_i\in\mathcal{A}_i}\sum_{k=1}^K\alpha_{k}^K\left[r_{i,h}^k(s,a_i) + (1-R)P_{i,h}^k(s,a_i)\widehat{V}_{i,h+1} + R\min\widehat{V}_{i,h+1}\right] \nonumber\\
&= \max_{a_i\in\mathcal{A}_i}\sum_{k=1}^K\alpha_{k}^K q_{i,h,s}^k =\max_{a_i\in\mathcal{A}_i}Q_{i,h,s}^K,
\end{align}
where $Q_{i,h,s}^K = Q_{i,h}^K(s,\cdot)\in\mathbb{R}^{A_i}$, and the last equality uses the fact that $Q_{i,h,s}^K = \sum_{k=1}^K\alpha_{k}^K q_{i,h,s}^k$ because of $Q_{i,h,s}^0 = 0$ for all $i,h,s$.

Below we intend to find the upper bound of $Q_{i,h,s}^k$. 
Towards this, we introduce the following auxiliary lemma (c.f. Theorem 3 in \citet{li2022minimax}).
\begin{lemma}\label{lem:FTRL}
Suppose that $0 < \alpha_1 \le 1$ and $\eta_1 = \eta_2(1 - \alpha_1)$. 
Also assume that $0 < \alpha_k < 1$ and $0 <
\eta_{k+1}(1 - \alpha_k) \le \eta_k$ for all $k \ge 2$. In addition, define
$\widehat{\eta}_1=\eta_2$ and $\widehat{\eta}_k=\frac{\eta_k}{1-\alpha_k}$ for $k\ge 2$.
Then we have
\begin{align*}
Q_{i,h,s}^K
&\le \sum_{k=1}^K\alpha_k^K\mathbb{E}_{\pi_{i,h,s}^k}q_{i,h,s}^k + \frac{5}{3}\sum_{k=1}^K\alpha_k^K\widehat{\eta}_k\alpha_k\mathsf{Var}_{\pi_{i,h,s}^k}(q_{i,h,s}^k)\nonumber\\
&\quad + \frac{\log A_i}{\eta_{K+1}} + 3\sum_{k=1}^K\alpha_{k}^K\widehat{\eta}_k^2\alpha_k^2\left\|q_{i,h,s}^k-\min q_{i,h,s}^k\right\|_{\infty}^3\mathds{1}\left(\widehat{\eta}_k\alpha_k\left\|q_{i,h,s}^k-\min q_{i,h,s}^k\right\|_{\infty}>\frac13\right).
\end{align*}
\end{lemma}
\begin{proof}
    Careful readers may find that this lemma is slightly different from Theorem 3 in \citet{li2022minimax} by replacing $\left\|q_{i,h,s}^k\right\|_{\infty}$ with $\left\|q_{i,h,s}^k-\min q_{i,h,s}^k\right\|_{\infty}$.
    This can be easily verified by replacing all $\ell_k$ with $\ell_k-\min\ell_k$ when bounding $\left\langle\pi_k-\pi_{k+1}^-,\ell_k\right\rangle$ in the proof of \citet{li2022minimax}. 
\end{proof}

To verify that the condition in Lemma \ref{lem:FTRL} holds true, it suffices to prove that
\begin{align*}
\frac{\eta_{k}}{\eta_{k+1}} = \sqrt{\frac{\alpha_{k+1}}{\alpha_k}} = \sqrt{\frac{k-1+c_\alpha\log K}{k+c_\alpha\log K}} = \sqrt{1-\frac{1}{k+c_\alpha\log K}} \overset{(i)}{\ge} \sqrt{1-\alpha_k}\ge 1-\alpha_k,
\end{align*}
where (i) holds as long as $c_{\alpha} \ge 1$ and $K\ge 3$.
Then applying Lemma \ref{lem:FTRL}, we have
\begin{align*}
Q_{i,h,s}^K
&\le \sum_{k=1}^K\alpha_k^K\mathbb{E}_{\pi_{i,h,s}^k}q_{i,h,s}^k +\xi_1 + \xi_2 + \xi_3,
\end{align*}
where
\begin{align*}
\xi_1 &=  \frac{5}{3}\sum_{k=1}^K\alpha_k^K\widehat{\eta}_k\alpha_k\mathsf{Var}_{\pi_{i,h,s}^k}(q_{i,h,s}^k),\quad
\xi_2 =  \frac{\log A_i}{\eta_{K+1}},\nonumber\\
\xi_3 &= 3\sum_{k=1}^K\alpha_{k}^K\widehat{\eta}_k^2\alpha_k^2\left\|q_{i,h,s}^k-\min q_{i,h,s}^k\right\|_{\infty}^3\mathds{1}\left(\widehat{\eta}_k\alpha_k\left\|q_{i,h,s}^k-\min q_{i,h,s}^k\right\|_{\infty}>\frac13\right).
\end{align*}

Below we shall bound $\xi_1$, $\xi_2$, and $\xi_3$ separately.
\begin{itemize}
    \item For $\xi_1$, we have
    \begin{align}
    \xi_1 &= \frac{5}{3}\alpha_1^K\alpha_1\eta_2\mathsf{Var}_{\pi_{i,h,s}^1}(q_{i,h,s}^1) + \frac{5}{3}\sum_{k=2}^K\alpha_k^K\frac{\alpha_k\eta_k}{1-\alpha_k}\mathsf{Var}_{\pi_{i,h,s}^k}(q_{i,h,s}^k)\nonumber\\
    &\overset{(i)}{\lesssim} \frac{\min\left\{H,\frac{1}{R}\right\}^2}{K^6} 
    + \sum_{k=2}^{K/2}\frac{(\min\{H,\frac{1}{R}\}\log K)^{\frac32}}{K^6} 
    + \sum_{k=K/2+1}^K\sqrt{\frac{\log^2 K}{K\min\{H,\frac{1}{R}\}}}\alpha_k^K\mathsf{Var}_{\pi_{i,h,s}^k}(q_{i,h,s}^k)\nonumber\\
    &\overset{(ii)}{\lesssim}\sqrt{\frac{\log^2 K}{K\min\{H,\frac{1}{R}\}}}\sum_{k=1}^K\alpha_k^K\mathsf{Var}_{\pi_{i,h,s}^k}(q_{i,h,s}^k) + \frac{1}{K^3},\label{eq:proof-xi-1}
    \end{align}
    where (i) holds since $\alpha_k^K\le \frac{1}{K^6}$ according to Lemma \ref{lem:alpha}, and $\mathsf{Var}_{\pi_{i,h,s}^k}(q_{i,h,s}^k)\lesssim \min\{H,\frac{1}{R}\}$, and 
    \begin{align}
    \frac{1}{1-\alpha_k} 
    &\le \frac{1}{1-\alpha_2} = 1+c_{\alpha}\log K\lesssim \log K, \quad k\ge 2,\nonumber\\
    \frac{1}{1-\alpha_k} 
    &\le \frac{1}{1-\alpha_{K/2}} \le 2, \quad k\ge \frac{K}{2},\nonumber\\
    \frac{\eta_k\alpha_k}{1-\alpha_k} 
    & = \frac{\sqrt{\alpha_k}}{1-\alpha_k}\sqrt{\frac{\log K}{\min\{H,\frac{1}{R}\}}} 
    \le 2\sqrt{\frac{\log^2 K}{k\min\{H,\frac{1}{R}\}}}.\label{eq:proof-temp-1}
    \end{align}
    as long as $K\ge 2c_{\alpha}\log K+2$,
    and (ii) uses the fact that $K\gtrsim \log K\min\{H,\frac{1}{R}\}$ for any $\varepsilon \in(0,H]$.
    \item For $\xi_2$, we have
    \begin{align}
    \xi_2 = \frac{\log A_i}{\eta_{K+1}} = \log A_i\sqrt{\frac{\min\{H,\frac{1}{R}\}}{\log K}}\sqrt{\frac{c_{\alpha}\log K}{K+c_{\alpha}\log K}} \le \log A_i\sqrt{\frac{c_{\alpha}\min\{H,\frac{1}{R}\}}{K}}.\label{eq:proof-xi-2}
    \end{align}
    \item For term $\xi_3$, we have
    \begin{align}
    \xi_3 &\le 3\sum_{k=2}^K\alpha_{k}^K\widehat{\eta}_k^2\alpha_k^2\left\|q_{i,h,s}^k-\min q_{i,h,s}^k\right\|_{\infty}^3\mathds{1}\left(\widehat{\eta}_k\alpha_k\left\|q_{i,h,s}^k-\min q_{i,h,s}^k\right\|_{\infty}>\frac13\right) \nonumber\\
    &\quad + 3\alpha_{1}^K{\eta}_2^2\alpha_1^2\left\|q_{i,h,s}^1-\min q_{i,h,s}^1\right\|_{\infty}^3.\label{eq:proof-xi-3}
    \end{align}
    Here by using Lemma \ref{lem:qihs}, the second term in the right-hand-side of \eqref{eq:proof-xi-3} satisfy
    \begin{align}
    3\alpha_{1}^K{\eta}_2^2\alpha_1^2\left\|q_{i,h,s}^1-\min q_{i,h,s}^1\right\|_{\infty}^3\lesssim \frac{\min\left\{H,\frac{1}{R}\right\}^3}{K^6} \le \frac{1}{K^3}.\label{eq:proof-xi-temp-1}
    \end{align}
    To bound the first term, notice \eqref{eq:proof-temp-1}, we have
    \begin{align*}
    \widehat{\eta}_k\alpha_k\left\|q_{i,h,s}^k-\min q_{i,h,s}^k\right\|_{\infty} &= \frac{\eta_k\alpha_k}{1-\alpha_k}\left\|q_{i,h,s}^k-\min q_{i,h,s}^k\right\|_{\infty} \nonumber\\
    &\le
    c\log K\sqrt{\frac{\min\{H,\frac{1}{R}\}}{k}} \le \frac{1}{3},\quad k\ge 9c^2\log^2 K\min\{H,\frac{1}{R}\},
    \end{align*}
    where $c$ is a constant. Then the first term is bound by
    \begin{align}
    &3\sum_{k=2}^K\alpha_{k}^K\widehat{\eta}_k^2\alpha_k^2\left\|q_{i,h,s}^k-\min q_{i,h,s}^k\right\|_{\infty}^3\mathds{1}\left(\widehat{\eta}_k\alpha_k\left\|q_{i,h,s}^k-\min q_{i,h,s}^k\right\|_{\infty}>\frac13\right) \nonumber\\
    &\quad \le 3\sum_{k=2}^{9c^2\log^2 K\min\{H,\frac{1}{R}\}}\alpha_{k}^K\widehat{\eta}_k^2\alpha_k^2\left\|q_{i,h,s}^k-\min q_{i,h,s}^k\right\|_{\infty}^3\nonumber\\
    &\overset{(i)}{\lesssim} \frac{\log^2 K\min\{H,\frac{1}{R}\}^4}{K^6} \overset{(i)}{\lesssim} \frac{1}{K^2},\label{eq:proof-xi-temp-2}
    \end{align}
    where (i) holds because $\alpha_k^K\le \frac{1}{K^6}$ for $K/2\ge 9c^2\log^2 K\min\{H,\frac{1}{R}\}$, and (ii) holds because $K\gtrsim \log K\min\{H,\frac{1}{R}\}$ for any $\varepsilon\in(0,H]$.
    Inserting \eqref{eq:proof-xi-temp-1} and \eqref{eq:proof-xi-temp-2} into \eqref{eq:proof-xi-3}, we have
    \begin{align}\label{eq:proof-xi-33}
    \xi_3 &\lesssim \frac{1}{K^2}.
    \end{align}
\end{itemize}
    Combining \eqref{eq:proof-xi-1}, \eqref{eq:proof-xi-2}, and \eqref{eq:proof-xi-33}, and notice that $K\gtrsim \min\{H,\frac{1}{R}\}$, we have
    \begin{align*}
    Q_{i,h,s}^K&\le \sum_{k=1}^K\alpha_k^K\mathbb{E}_{\pi_{i,h,s}^k}q_{i,h,s}^k + c\sqrt{\frac{\log^2 K}{K\min\{H,\frac{1}{R}\}}}\sum_{k=1}^K\alpha_k^K\mathsf{Var}_{\pi_{i,h,s}^k}(q_{i,h,s}^k) + \log A_i\sqrt{\frac{(c_{\alpha}+1)\min\{H,\frac{1}{R}\}}{K}}.
    \end{align*}
According to the definition of $\widehat{V}_{i,h}$ in \eqref{eq:cal-hatV}, if $\widehat{V}_{i,h} = H-h+1$, then \eqref{eq:lem-hatV-ge-barV} holds obviously.
Otherwise, we have
\begin{align*}
\widehat{V}_{i,h}(s) &=\sum_{k=1}^K\alpha_k^K\mathbb{E}_{a_i\sim\pi_{i,h,s}^k}q_{i,h,s} + \beta_{i,h}(s)\nonumber\\
&\ge \max Q_{i,h,s}^K - c\sqrt{\frac{\log^2 K}{K\min\{H,\frac{1}{R}\}}}\sum_{k=1}^K\alpha_k^K\mathsf{Var}_{\pi_{i,h,s}^k}(q_{i,h,s}^k) \nonumber\\
&\quad- \log A_i\sqrt{\frac{(c_{\alpha}+1)\min\{H,\frac{1}{R}\}}{K}} + \beta_{i,h}(s)\nonumber\\
&\overset{(i)}{\ge} \max Q_{i,h,s}^K \overset{(ii)}{\ge} \overline{V}_{i,h}^{\star,\widehat{\pi}_{-i},R},
\end{align*}
where (i) holds as long as $c_{\mathsf{b}}\ge c + \sqrt{c_{\alpha} + 1}$, and (ii) arises from $\overline{V}_{i,h}^{\star,\widehat{\pi}_{-i},R}\le \max_{a_i}Q_{i,h,s}^k$ as shown in \eqref{eq:proof-upper-barV-star}.
Thus we complete the proof.

\section{Proof of Corollary \ref{cor:main}}
\label{subsec:proof-cor}

%Throughout this section, we use superscript to denote the index of agent $i$, and the subscript to denote the step $h$ for clear.
% In addition, for ease of using Theorem \ref{thm:upper}, we denote $\widehat{\pi} = \{\widehat{\pi}_h\}_{1\le h\le H}$ as 
% $$
% \widehat{\pi}_h= \sum_{k=1}^K\alpha_k^K(\pi_{1,h}^k\times\pi_{2,h}^k).
% $$
In a two-player zero-sum RMG, it is evident that $\widehat{\pi}_{-1} = \widehat{\pi}_2$ and $\widehat{\pi}_{-2} = \widehat{\pi}_1$, where the subscript $1$ and $2$ indicates the index of agent.
Then the conclusion in Theorem \ref{thm:upper} can be restated as
\begin{align}\label{eq:proof-cor-temp-1}
\mathsf{gap}(\widehat{\pi};s) 
&= \max\left\{V_{1,1}^{\star,\widehat{\pi}_{-1},R}(s)-V_{1,1}^{\widehat{\pi},R}(s),V_{2,1}^{\star,\widehat{\pi}_{-2},R}(s)-V_{2,1}^{\widehat{\pi},R}(s)\right\} \nonumber\\
&= \max\left\{V_{1,1}^{\star,\widehat{\pi}_{2},R}(s)-V_{1,1}^{\widehat{\pi},R}(s),V_{2,1}^{\star,\widehat{\pi}_{1},R}(s)-V_{2,1}^{\widehat{\pi},R}(s)\right\} \le \varepsilon,\quad \forall s\in\mathcal{S}.
\end{align}
Noticing that $r_{1,h} = -r_{2,h}$, we have $V_{1,1}^{\pi} = -V_{2,1}^{\pi}$ for any policy $\pi$.
Thus for any state $s\in\mathcal{S}$, we have 
\begin{align*}
V^{\star,\widehat{\pi}_2,R}_{1,1}(s) - V^{\widehat{\pi}_1\times\widehat{\pi}_2,R}_{1,1}(s) 
&= V^{\star,\widehat{\pi}_2,R}_{1,1}(s) + V^{\widehat{\pi}_1\times\widehat{\pi}_2,R}_{2,1}(s)
\le  V^{\star,\widehat{\pi}_2,R}_{1,1}(s) +  V^{\star,\widehat{\pi}_1,R}_{2,1}(s)\nonumber\\
&\overset{(i)}{=} V^{\star,\widehat{\pi}_2,R}_{1,1}(s) - V^{\widehat{\pi},R}_{1,1}(s) +  V^{\star,\widehat{\pi}_1,R}_{2,1}(s) - V^{\widehat{\pi},R}_{2,1}(s)
\overset{(ii)}{\le} \varepsilon + \varepsilon = 2\varepsilon,
\end{align*}
where (i) arises from $V^{\widehat{\pi},R}_{1,1}(s) = -V^{\widehat{\pi},R}_{2,1}(s)$, and (ii) uses \eqref{eq:proof-cor-temp-1}.
Similarly, we can prove that $V^{\star,\widehat{\pi}_1,R}_{2,1}(s) - V^{\widehat{\pi}_1\times\widehat{\pi}_2,R}_{2,1}(s) \le 2\varepsilon$ for all $s\in\mathcal{S}$.
By replacing $\varepsilon$ with $\varepsilon/2$, we could prove Corollary \ref{cor:main}.

\section{Proof of Theorem \ref{thm:lower}}
\label{sec:proof-thm-lower}

According to the proof of Theorem 2 in \citet{shi2024sample}, we could construct a collection of robust Markov games, such that finding a robust NE/CE/CCE of these RMGs degrades to finding the optimal policy of the first agent over a corresponding robust Markov Decision Processes.
Thus the proof of Theorem \ref{thm:lower} degrades to prove that: one can construct a collection of robust Markov Decision Processes $\{\mathcal{M}_\theta | \theta \in\Theta\}$ with $S$ states, $A$ actions, horizon $H\ge 16$, and the uncertanty level $R\in[0,1-c_0)$, such that
\begin{align*}
\inf_{\pi}\max
_{\theta\in\Theta}\mathbb{P}_{\theta}\left(\max_s\left[V_1^{\star,R}(s) - V_1^{\widehat{\pi},R}(s)\right] \ge \varepsilon\right)\ge \frac14,
\end{align*}
provided that the sample size for each state-action pair over the nominal
transition kernel obeys
$$
N\le \frac{cH^3}{\varepsilon^2}\min\left\{H,\frac{1}{R}\right\}
$$
for some sufficiently small constant $c > 0$. Here, the infimum is over all policy estimators $\widehat{\pi}$, and $\mathbb{P}_{\theta}$ denotes the probability when the RMDP is $\mathcal{M}_\theta$.
Below we shall construct a hard instance.

\subsection{Construction of the hard problem instances}

\paragraph{Construction of the hard MDPs.} 
To begin with, define the state space and the action space as $
\mathcal{S} = \{0,1,\cdots,S-1\}
$ and $
\mathcal{A} = \{0,1,\cdots,A-1\},
$ respectively.
To construct the nominal transition probability matrix, for any integer $H \ge 16$, let us consider a set $\Theta \subset
\{0, 1\}^H$ of $H$-dimensional vectors, which we shall construct shortly. We then generate a collection of MDPs
$$
\mathsf{MDP}(\Theta) = \left\{\mathcal{M}_\theta=\left(\mathcal{S},\mathcal{A},P^\theta=\{P_h^\theta\}_{h=1}^H,\{r_h\}_{h=1}^H,H\right)|\theta = [\theta_h]_{1\le h\le H}\in\Theta\right\}.
$$
Given any state $s\in\{2,\cdots,S-1\}$, the corresponding action space is $\mathcal{A}=\{0,\cdots,A-1\}$.
While for state $s = 0$,
the action space is $\mathcal{A}' = \{0,1\}$.
For any $\theta = \{\theta_h\}_{h=1}^H\in \{0, 1\}^H$, the nominal transition kernel $P^{\theta}$ of the constructed robust MDP $\mathcal{M}_\phi$ is defined as
\begin{align*}
P^{\theta}_h(s_{h+1}|s_h,a_h)=\left\{
\begin{array}{ll}
p\mathds{1}(s_{h+1}=1) + (1-p)\mathds{1}(s_{h+1}=0),&{\rm if}~(s_h,a_h) = (0,\theta_h),\\
q\mathds{1}(s_{h+1}=1) + (1-q)\mathds{1}(s_{h+1}=0),&{\rm if}~(s_h,a_h) = (0,1-\theta_h),\\
\mathds{1}(s_{h+1}=1),&{\rm if}~s_h\ge 1,\\
\end{array}
\right.
\end{align*}
where $0<q<p<\frac{3}{4}$ are defined as
\begin{align*}
p = c\max\left\{\frac{1}{H},R\right\},\quad q = p-\Delta,
\end{align*}
where 
\begin{align}\label{eq:def-Delta}
\Delta := \frac{c_1\varepsilon} {H\min\left\{H,\frac{1}{R}\right\}},
\end{align}
and $c\le 3/4$.
Moreover, we assume that $\varepsilon\le c_0H$.
Here $c_1$ is a sufficiently large constant, $c_0$ is a sufficiently small constant, and meanwhile $c_0c_1\le c/2$.
With this definition, one can easily check that $q\ge p-\frac{c_1c_0}{c}\ge\frac{p}{2}$ for any $\varepsilon\in(0,H]$.

Then, we define the reward function as
\begin{align*}
r_h(s,a) = \left\{
\begin{array}{ll}
1,&{\rm if}~s=1,\\
0,&{\rm otherwise}.
\end{array}
\right.
\end{align*}

Finally, let us choose the set $\Theta \subset \{0, 1\}^H$. By virtue of the Gilbert-Varshamov lemma \citep{Gilbert1952Acomparison}, one can construct $\Theta \subset \{0, 1\}^H$ in a way that
\begin{align}\label{eq:proof-thm-lower-construct-Theta}
|\Theta|\ge {\rm e}^{H/8}\quad {\rm and}\quad \left\|\theta - \tilde{\theta}\right\|_1\ge\frac{H}{8}\quad{\rm for~any~\theta,\tilde\theta\in\Theta~obeying~\theta\neq\tilde\theta}.
\end{align}

\paragraph{Uncertainty set of the transition kernels.}
Recalling the uncertainty set assumed is 
\begin{align*}
\mathcal{U}^R(P_h^\theta) = \otimes\mathcal{U}^R(P_{h}^\theta(\cdot|s,a)),\quad \mathcal{U}^R(P_{h}^\theta(\cdot|s,a)):=\left\{(1-R)P_h^\theta(\cdot|s,a)+RP'|P'\in\Delta(\mathcal{S})\right\}.
\end{align*}
For any policy $\pi = \{\pi_{h}\}_{h=1}^H$, its corresponding transition probability matrix is
\begin{align*}
P^{\theta,\pi}_h(s_{h+1}|s_h)=\left\{
\begin{array}{ll}
z_h^\pi\mathds{1}(s_{h+1}=1) + \left(1-z_h^\pi\right) \mathds{1}(s_{h+1}=0),&{\rm if}~s_h = 0,\\
\mathds{1}(s_{h+1}=1),&{\rm if}~s_h\ge 1,\\
\end{array}
\right.
\end{align*}
where
$$
z_h^\pi = \pi_h(\theta_h|0)p+\pi_h(1-\theta_h|0)q.
$$

We pause to show by mathematical induction that for any policy $\pi$, we always have 
\begin{align}\label{eq:proof-thm-lower-temp-1}
V_h^{\pi,R}(0) = \min_s V_h^{\pi,R}(s),\quad \forall 1\le h\le H.
\end{align}
It is obvious that $V_{H+1}^{\pi,R}(0)=\min_s V_{H+1}^{\pi,R}(s) = 0$.
Suppose that $V_{h+1}^{\pi,R}(0)= \min_s V_{h+1}^{\pi,R}(s)$.
Then we have
\begin{align*}
V_{h}^{\pi,R}(0) 
&= (1-R)z_h^\pi V_{h+1}^{\pi,R}(1) + (1-R)(1-z_h^\pi) V_{h+1}^{\pi,R}(0) + RV_{h+1}^{\pi,R}(0)\nonumber\\
&\le (1-R)V_{h+1}^{\pi,R}(1) +RV_{h+1}^{\pi,R}(0) \nonumber\\
&\le (1-R)V_{h+1}^{\pi,R}(1)+RV_{h+1}^{\pi,R}(0) = \min_{s\ge 1}V_{h}^{\pi,R}(s)\nonumber\\
&\le 1 + (1-R)V_{h+1}^{\pi,R}(1)+RV_{h+1}^{\pi,R}(0) = V_{h}^{\pi,R}(1).
\end{align*}
Thus we complete the proof.

Therefore, for any policy $\pi$, the perturbed transition kernel that minimizes the expected reward is
\begin{align*}
    P_h^{\theta,\pi,V_{h+1}^{\pi,R}}(\cdot|s_h) 
    &= \arg\min_{P\in\mathcal{U}_h^R(P_h^{\theta,\pi}(\cdot|s_h))} PV_{h+1}^{\pi,R}\nonumber\\
    &=\left\{
    \begin{array}{ll}
    (1-R)z_h^\pi\mathds{1}(s_{h+1}=1) + \left(1-(1-R)z_h^\pi\right) \mathds{1}(s_{h+1}=0),&{\rm if}~s_h = 0,\\
(1-R)\mathds{1}(s_{h+1}=1) + R\mathds{1}(s_{h+1}=0),&{\rm if}~s_h\ge 1.\\
    \end{array}
    \right.
\end{align*}

\paragraph{Optimal robust policy and value function.}
To proceed, we are ready to identify the the optimal robust policies, and derive the corresponding robust value functions.
According to \eqref{eq:proof-thm-lower-temp-1}, the optimal robust policy $\pi^{\star,\theta}=\{\pi_h^{\star,\theta}\}_{h=1}^H$ for $\mathcal{M}_\theta$ is 
\begin{align}\label{eq:proof-thm-lower-pistar}
\pi_h^{\star,\theta}(\theta_h|0) = 1,\quad \pi_h^{\star,\theta}(1-\theta_h|0) = 0.
\end{align}
According to robust Bellman equation, the corresponding value function is
\begin{align}\label{eq:proof-thm-lower-Vstar}
V^{\star,R}_h(0) &= \frac{\widetilde{p}}{R+\widetilde{p}}\left(H-h+1 - \frac{1-(1-R-\widetilde{p})^{H-h+1}}{R+\widetilde{p}}\right)\nonumber\\
V^{\star,R}_h(1) &= \frac{1}{R+\widetilde{p}}\left(\widetilde{p}(H-h+1) + R\frac{1-(1-R-\widetilde{p})^{H-h+1}}{R+\widetilde{p}}\right),
\end{align}
where
\begin{align}\label{eq:def-tilde-p}
\widetilde{p} = (1-R)p.
\end{align}
The derivation is deferred to the end of this section.
Here we remark that for all $\theta$, RMPDs $\mathcal{M}_\theta$ share the same optimal value function.
Thus we omit the superscript $\theta$ in $V^{\star,R}_h$. 

\subsection{Establishing the minimax lower bound}

We are now positioned to establish our sample complexity lower bounds.
It suffices to prove that 
\begin{align}\label{eq:proof-thm-lower-temp-5}
\inf_{\widehat{\pi}}\max
_{\theta\in\Theta}\mathbb{P}_{\theta}\left(V_1^{\star,R}(0) - V_1^{\widehat{\pi},R}(0) > \varepsilon\right)\ge \frac14,
\end{align}
where $\widehat{\pi}$ is any policy estimator computed based on the independent samples.

\paragraph{Step 1: converting the goal to estimate $\theta$.}

For any policy estimator $\widehat{\pi}=\{\widehat{\pi}_h\}_{h=1}^H$, We can construct a corresponding estimator $\widehat{\theta}=\{\widehat{\theta}_h\}_{h=1}^H$ for $\theta$ as
\begin{align}\label{eq:def-hattheta}
\widehat\theta := \arg\min_{\bar{\theta}\in\Theta}\sum_{h=1}^H\left\|\widehat{\pi}_{h}(\cdot|0)-\pi_h^{\star,\bar\theta}(\cdot|0)\right\|_1,
\end{align}
where $\pi_h^{\star,\bar\theta}$ is defined in \eqref{eq:proof-thm-lower-pistar}.
According to the construction of $\Theta$ in \eqref{eq:proof-thm-lower-construct-Theta}, a good policy estimator $\widehat{\pi}$ implies a correct estimator $\widehat{\theta}$.
Mathematically, with samples from $\mathcal{M}_\theta$ for $\theta\in\Theta$, we have
\begin{align}\label{eq:proof-thm-lower-estimator-prob}
\mathbb{P}_\theta\left(V_1^{\star,R}(0) - V_1^{\widehat{\pi},R}(0) <\varepsilon\right)\le \mathbb{P}_\theta\left(\widehat{\theta} = \theta\right),
\end{align}
where $\mathbb{P}_\theta$ denotes the probability distribution when the MDP is $\mathcal{M}_\theta$.
The proof of this claim is postponed to the end of this section.
% implies a correct estimator for $\theta$ satisfying 
% $$
% \widehat{\theta} = \theta^{\star}.
% $$

\paragraph{Step 2: probability of error in testing multiple hypotheses.}

Next, we turn attention to the hypothesis testing problem on $\Theta$.
Specifically, we focus on the minimax probability of error defined as follows:
\begin{align}\label{eq:def-pe}
p_e:=\inf_\psi\max_{\theta\in\Theta}\mathbb{P}_\theta(\psi\neq\theta),
\end{align}
where the infimum is taken over all possible tests $\psi$ constructed based on the samples from $\mathcal{M}_\theta$.

% We pause to remark that the dataset should have $\frac{N}{H(S-1)(A-1)}$ samples for each state-action pair.
% This is because we can replace the special state $0$ with any state in space $\mathcal{S}$, and replace action $0$ with any subset $\mathcal{A}_1\subset\mathcal{A}$ with $|\mathcal{A}_1|=A/2$, in a similar way as in \citet{shi2024sample}.
% The specific construction is postponed to the end of this section.

%$\mu^\theta = \frac{1}{H}\sum_{h=1}^H\mu_h^\theta$
Let $\mu_h^\theta$ represent the distribution of a sample tuple $(s_h, a_h, s_{h+1})$ for step $h$ under the nominal transition kernel ${P}_h^\theta$ associated with $\mathcal{M}_\theta$ for $1\le h\le H$.
Assume that we collect $N$ independent samples for each state-action pair and each time step $H$.
Combined with Fano's inequality from \citet{Tsybakov2009Introduction} (c.f.Theorem 2.2) and the additivity of the KL divergence (cf. \citet{Tsybakov2009Introduction}, Page 85), we have
\begin{align}\label{eq:proof-thm-lower-temp-3}
p_e &\ge 1 - \frac{N\max_{\theta,\tilde\theta\in\Theta}\sum_{h=1}^H\mathsf{KL}(\mu_h^\theta||\mu_h^{\tilde{\theta}}) + \log 2}{\log|\Theta|}\nonumber\\
&\ge \frac12 - \frac{8N\max_{\theta,\tilde\theta\in\Theta}\sum_{h=1}^H\mathsf{KL}(\mu_h^\theta||\mu_h^{\tilde{\theta}}) }{H},
\end{align}
where we use the assumption that $H\ge 16\log 2 $ and $|\Theta|\ge {\rm e}^{H/8}$.
Notice that 
\begin{align*}
\mathsf{KL}(\mu_h^\theta||\mu_h^{\tilde{\theta}}) &= \sum_{a=\{0,1\}}\mathsf{KL}\left(P_h^\theta(\cdot|0,a)||P_h^{\tilde{\theta}}(\cdot|0,a)\right) \nonumber\\
&=\left(p\log\frac{p}{q} + (1-p)\log\frac{1-p}{1-q}+q\log\frac{q}{p} + (1-q)\log\frac{1-q}{1-p}\right) \mathds{1}(\theta_h\neq\tilde{\theta}_h)\nonumber\\
&\le \left(\frac{p^2}{q}-p + \frac{(1-p)^2}{1-q}-1+p + \frac{q^2}{p}-1 +\frac{(1-q)^2}{1-p}-1+q\right) \mathds{1}(\theta_h\neq\tilde{\theta}_h)\nonumber\\
&= \left(\frac{(p-q)^2}{p(1-p)} + \frac{(p-q)^2}{q(1-q)}\right) \mathds{1}(\theta_h\neq\tilde{\theta}_h)\nonumber\\
&\overset{(i)}{\le} \frac{8(p-q)^2}{q} \mathds{1}(\theta_h\neq\tilde{\theta}_h),
\end{align*}
where (i) uses the fact that $p(1-p)\ge\frac{p}{4}$ and $ q(1-q)\ge \frac{q}{4}$ for $q<p<\frac34$.
Inserting into \eqref{eq:proof-thm-lower-temp-3}, we have
\begin{align*}
p_e &\ge \frac12 - \frac{8N\max_{\theta,\tilde\theta\in\Theta}\left\|\theta - \tilde{\theta}\right\|_1 }{H}\frac{8(p-q)^2}{q}\nonumber\\
&\overset{(i)}{\ge} \frac12 - \frac{64N(p-q)^2}{q}\overset{(ii)}{\ge} \frac12 - \frac{128Nc_1^2\varepsilon^2}{cH^2\min\left\{H,\frac{1}{R}\right\}},
\end{align*}
where (i) uses the value of $p-q$ in \eqref{eq:def-Delta}, and the fact that $\max_{\theta,\tilde\theta\in\Theta}\left\|\theta - \tilde{\theta}\right\|_1 \le H$, (ii) uses the fact that $q\ge \frac{p}{2}$.
Provided that the total number of samples
\begin{align*}
NHSA \le \frac{cH^3SA}{512c_1^2\varepsilon^2}\min\left\{H,\frac{1}{R}\right\},
\end{align*}
we have
\begin{align}\label{eq:proof-thm-lower-pe}
p_e \ge \frac12 - \frac14 = \frac14.
\end{align}

\paragraph{Step 3: combining the above results.}
Combining \eqref{eq:def-pe} and \eqref{eq:proof-thm-lower-pe}, we have that for any estimator $\psi$ for $\theta$, the error probability
\begin{align*}
\max_{\theta\in\Theta}\mathbb{P}_\theta(\psi\neq \theta) \ge \frac14.
\end{align*}

Then for the estimator $\widehat\theta$ constructed in \eqref{eq:def-hattheta}, we have
\begin{align*}
\max_{\theta\in\Theta}\mathbb{P}_\theta(\widehat{\theta}\neq \theta) \ge \frac14.
\end{align*}
According to \eqref{eq:proof-thm-lower-estimator-prob}, we have for any policy estimator $\widehat{\pi}$,
\begin{align*}
\max_{\theta\in\Theta}\mathbb{P}_\theta\left(V_1^{\star,R}(0) - V_1^{\widehat{\pi},R}(0) \ge\varepsilon\right)\ge \max_{\theta\in\Theta}\mathbb{P}_\theta\left(\widehat{\theta} \neq \theta\right) \ge\frac14.
\end{align*}

Considering that this holds for arbitrary policy estimator $\widehat{\pi}$, we could prove that \eqref{eq:proof-thm-lower-temp-5} holds true,
and complete the proof.

\paragraph{Calculation of \eqref{eq:proof-thm-lower-Vstar}.}
Considering that we only consider about the value of $V^{\star,R}_h(0)$ and $V^{\star,R}_h(1)$, we rewrite vector $V^{\star,R}_h = [V^{\star,R}_h(0),V^{\star,R}_h(1)]^{\top}$ for convenience.
According to robust Bellman's equation, we have
\begin{align}\label{eq:proof-thm-lower-temp-4}
V^{\star,R}_h = 
\left[
\begin{array}{c}
0\\
1
\end{array}
\right]
+ \left[
\begin{array}{cc}
1-\widetilde{p}&\widetilde{p}\\
R & 1-R
\end{array}
\right]
V^{\star,R}_{h+1}=:r + P^{\star,R}V^{\star,R}_{h+1},
\end{align}
with $V^{\star,R}_{H+1} = 0$,
where we redefine
$$
P^{\star,R} = P_h^{\star,V^{\star,R}_{h+1}},\quad r = r_h.
$$

By eigenvalue decomposition, $P^{\star,R}$ can be decomposed as
\begin{align}\label{eq:proof-thm-lower-eigendecom-P}
\left[
\begin{array}{cc}
1-\widetilde{p}&\widetilde{p}\\
R & 1-R
\end{array}
\right] 
&= \left[
\begin{array}{cc}
1&\widetilde{p}\\
1 & -R
\end{array}
\right]\left[
\begin{array}{cc}
1&0\\
0 & 1-R-\widetilde{p}
\end{array}
\right]\left[
\begin{array}{cc}
1&\widetilde{p}\\
1 & -R
\end{array}
\right]^{-1}\nonumber\\
&=\frac{1}{R+\widetilde{p}}\left[
\begin{array}{cc}
1&\widetilde{p}\\
1 & -R
\end{array}
\right]\left[
\begin{array}{cc}
1&0\\
0 & 1-R-\widetilde{p}
\end{array}
\right]\left[
\begin{array}{cc}
R&\widetilde{p}\\
1 & -1
\end{array}
\right].
\end{align}
Thus we have
\begin{align*}
V^{\star,R}_{h} 
&= \sum_{h'=h}^H (P^{\star,R})^{h'-h}r=\frac{1}{R+\widetilde{p}}\left[
\begin{array}{cc}
1&\widetilde{p}\\
1 & -R
\end{array}
\right]\left[
\begin{array}{cc}
\sum_{h'=h}^H 1&0\\
0 & \sum_{h'=h}^H (1-R-\widetilde{p})^{h'-h}
\end{array}
\right]\left[
\begin{array}{cc}
R&\widetilde{p}\\
1 & -1
\end{array}
\right]r\nonumber\\
&=\frac{1}{R+\widetilde{p}}\left[
\begin{array}{c}
\widetilde{p}\left(H-h+1 - \frac{1-(1-R-\widetilde{p})^{H-h+1}}{R+\widetilde{p}}\right)\\
\widetilde{p}(H-h+1) + R\frac{1-(1-R-\widetilde{p})^{H-h+1}}{R+\widetilde{p}}
\end{array}
\right].
\end{align*}

\paragraph{Proof of \eqref{eq:proof-thm-lower-estimator-prob}.}

It suffices to prove that for policy estimation $\widehat{\pi}$  obtained from samples of $\mathcal{M}_\theta$,
if $V_1^{\star,R}(0) - V_1^{\widehat{\pi},R}(0) \le \varepsilon$, then we have $\widehat{\theta} = \theta$,
where $\widehat{\theta}$ is defined by \eqref{eq:def-hattheta}.

Towards this, we assume that $\widehat{\theta} \neq \theta$.
Then according to the definition of $\Theta$ in \eqref{eq:proof-thm-lower-construct-Theta}, and recalling $\widehat{\theta}\in\Theta$, we have $\left\|\theta - \widehat{\theta}\right\|_1\ge H/8$, and
\begin{align}\label{eq:proof-thm-lower-temp-2}
\sum_{h=1}^H\left\|\widehat{\pi}_h(\cdot|0)-\pi_h^{\star,\theta}(\cdot|0)\right\|_1
&\ge\sum_{h=1}^H\left\|{\pi}_h^{\star,\theta}(\cdot|0)-\pi_h^{\star,\widehat{\theta}}(\cdot|0)\right\|_1 - \sum_{h=1}^H\left\|\widehat{\pi}_h(\cdot|0)-\pi_h^{\star,\widehat{\theta}}(\cdot|0)\right\|_1\nonumber\\
&=2\left\|\theta - \widehat{\theta}\right\|_1 - \sum_{h=1}^H\left\|\widehat{\pi}_h(\cdot|0)-\pi_h^{\star,\widehat{\theta}}(\cdot|0)\right\|_1\nonumber\\
&\ge \frac{H}{4}-\sum_{h=1}^H\left\|\widehat{\pi}_h(\cdot|0)-\pi_h^{\star,\widehat{\theta}}(\cdot|0)\right\|_1.
\end{align}
To proceed, we claim that $V_1^{\star,R}(0) - V_1^{\widehat{\pi},R}(0)<\varepsilon$ implies
\begin{align}\label{eq:proof-thm-lower-diff-pi}
\sum_{h=1}^H\left\|\widehat{\pi}_h(\cdot|0)-\pi_h^{\star,\theta}(\cdot|0)\right\|_1 <\frac{H}{8}.
\end{align}
Then according to \eqref{eq:proof-thm-lower-temp-2}, we have
\begin{align*}
\sum_{h=1}^H\left\|\widehat{\pi}_h(\cdot|0)-\pi_h^{\star,\theta}(\cdot|0)\right\|_1
&> \frac{H}{4}-\frac{H}{8} = \frac{H}{8},
\end{align*}
which contradicts \eqref{eq:proof-thm-lower-diff-pi}.
Thus we prove that $\widehat\theta = \theta$.

Now it suffices to prove that the claim \eqref{eq:proof-thm-lower-diff-pi} holds true.
Specifically, we intend to prove that if
\begin{align*}
\sum_{h=1}^H\left\|\widehat{\pi}_h(\cdot|0)-\pi_h^{\star,\theta}(\cdot|0)\right\|_1\ge\frac{H}{8},
\end{align*}
then one has 
\begin{align*}
V_1^{\star,R}(0) - V_1^{\widehat{\pi},R}(0) \ge\varepsilon.
\end{align*}

%We shall postpone the proof of this claim to Appendix \ref{subsec:proof-lem-diffV}
Towards this, for any policy $\widehat{\pi}$, notice that $V_1^{\widehat{\pi},R} = [V_1^{\widehat{\pi},R}(0),V_1^{\widehat{\pi},R}(1)]^{\top}$ is computed as
\begin{align*}
V_1^{\widehat{\pi},R} = \sum_{h=1}^H\prod_{j=1}^{h-1} A_j^{\widehat{\pi}} r,
\end{align*}
where $A_j^{\widehat{\pi}}$ is defined as 
\begin{align*}
A_j^{\widehat{\pi}} = \left[
\begin{array}{cc}
1-\tilde{z}_j^{\widehat{\pi}}&\tilde{z}_j^{\widehat{\pi}}\\
R & 1-R
\end{array}
\right],\quad
r = \left[
\begin{array}{c}
0\\
1
\end{array}
\right]
\end{align*}
with
\begin{align}\label{eq:def-tilde-z}
\tilde{z}_j^{\widehat{\pi}} = (1-R)\left(\widehat{\pi}_j(\theta_j|0)p + \widehat{\pi}_j(1-\theta_j|0)q\right).
\end{align}
It is obvious that $(1-R)q\le \tilde{z}_j^{\widehat{\pi}}\le (1-R)p$.
Recall that the optimal value function $V_1^{\star,R} = [V_1^{\star,R}(0),V_1^{\star,R}(1)]^{\top}$ is calculated as (cf. \eqref{eq:proof-thm-lower-temp-4})
\begin{align*}
V_1^{\star,R} = \sum_{h=1}^H\prod_{j=1}^{h-1} A_j^{\star} r,
\end{align*}
where 
\begin{align*}
A_j^{\star} = \left[
\begin{array}{cc}
1-\widetilde{p}&\widetilde{p}\\
R & 1-R
\end{array}
\right],
\end{align*}
and $\widetilde{p}$ is defined in \eqref{eq:def-tilde-p}.
To decompose the difference between $V_1^{\star,R}$ and $V_1^{\widehat{\pi},R}$, we introduce a sequence of vectors $\widetilde{V}_0,\widetilde{V}_1,\cdots,\widetilde{V}_{H}$ defined as
\begin{align*}
\widetilde{V}_h = \sum_{h'=1}^{h+1}\prod_{j=1}^{h'-1} A_j^{\star} r + \sum_{h'=h+2}^{H}\prod_{j=1}^{h} A_j^{\star}\prod_{j=h+1}^{h'-1}A_j^{\widehat{\pi}} r,
\end{align*}
We remark that $\widetilde{V}_h$ is actually the value function at step $1$ for the policy $\widetilde{\pi}=\{\widetilde{\pi}_h\}_{h=1}^H$ satisfying $\widetilde{\pi}_j = \pi_j^{\star,\theta}$ for $j\le h$ and $\widetilde{\pi}_j = \widehat{\pi}_j$ for $j> h$.
Moreover, we have $\widetilde{V}_H = V_1^{\star,R}$ and $\widetilde{V}_{0} = V_1^{\widehat{\pi},R}$.
Now we can decompose the difference between $V_1^{\star,R}$ and $V_1^{\widehat{\pi},R}$ as
\begin{align}\label{eq:proof-thm-lower-temp-10}
V_1^{\star,R} - V_1^{\widehat{\pi},R} 
= \sum_{h=1}^H\left(\widetilde{V}_{h} - \widetilde{V}_{h-1}\right).
\end{align}

For the $h$-th term, we have
\begin{align*}
\widetilde{V}_{h} - \widetilde{V}_{h-1} 
&= \sum_{h'=1}^{h+1}\prod_{j=1}^{h'-1} A_j^{\star} r + \sum_{h'=h+2}^{H}\prod_{j=1}^{h} A_j^{\star}\prod_{j=h+1}^{h'-1}A_j^{\widehat{\pi}} r -\sum_{h'=1}^{h}\prod_{j=1}^{h'-1} A_j^{\star} r - \sum_{h'=h+1}^{H}\prod_{j=1}^{h-1} A_j^{\star}\prod_{j=h}^{h'-1}A_j^{\widehat{\pi}} r\nonumber\\
&= \prod_{j=1}^{h-1} A_j^{\star}\left(A_h^{\star}-A_h^{\widehat{\pi}}\right)\left(\sum_{h'=h+1}^{H}\prod_{j=h+1}^{h'-1}A_j^{\widehat{\pi}}r\right),
\end{align*}
where $\prod_{j=h+1}^{h}A_j^{\widehat{\pi}}r$ is defined as $r$.
We claim that 
\begin{align}\label{eq:proof-thm-lower-temp-6}
\widetilde{V}_{h} - \widetilde{V}_{h-1} 
&\ge \prod_{j=1}^{h-1} A_j^{\star}\left(A_h^{\star}-A_h^{\widehat{\pi}}\right)\left(\sum_{h'=h+1}^{H}\prod_{j=h+1}^{h'-1}A_j^{\star}r\right) \\
&\overset{(i)}{=} \left(A^{\star}\right)^{h-1}\left(A_h^{\star}-A_h^{\widehat{\pi}}\right)\sum_{h'=0}^{H-h-1}\left(A^{\star}\right)^{h'}r,\nonumber
\end{align}
where in (i) we rewrite $A_j^{\star}$ as $A^{\star}$ since it does not vary with $j$.
The proof of \eqref{eq:proof-thm-lower-temp-6} is postponed to the end of this section.
According to the decomposition in \eqref{eq:proof-thm-lower-eigendecom-P}, we have
\begin{align*}
\sum_{h'=0}^{H-h-1}\left(A^{\star}\right)^{h'}r &= \frac{1}{R+\widetilde{p}}\left[
\begin{array}{c}
\widetilde{p}\left(H-h - \frac{1-(1-R-\widetilde{p})^{H-h}}{R+\widetilde{p}}\right)\\
\widetilde{p}(H-h) + R\frac{1-(1-R-\widetilde{p})^{H-h}}{R+\widetilde{p}}
\end{array}
\right],\nonumber\\
%\end{align*}
%\begin{align*}
\left(A^{\star}\right)^{h-1} &=\frac{1}{R+\widetilde{p}}\left[
\begin{array}{cc}
1&\widetilde{p}\\
1 & -R
\end{array}
\right]\left[
\begin{array}{cc}
1&0\\
0 & \left(1-R-\widetilde{p}\right)^{h-1}
\end{array}
\right]\left[
\begin{array}{cc}
R&\widetilde{p}\\
1 & -1
\end{array}
\right].
\end{align*}
by inserting them, we have
\begin{align}\label{eq:proof-thm-lower-temp-11}
\widetilde{V}_{h}(0) - \widetilde{V}_{h-1}(0) 
&= \frac{\widetilde{p}-\widetilde{z}_h^{\widehat{\pi}}}{(R+\widetilde{p})^2}\left(R+\widetilde{p}\left(1-R-\widetilde{p}\right)^{h-1}\right)\left(1-\left(1-R-\widetilde{p}\right)^{H-h}\right)\nonumber\\
&\overset{(i)}{=} \frac{(1-R)(p-q)}{2(R+\widetilde{p})^2}\left(R+\widetilde{p}\left(1-R-\widetilde{p}\right)^{h-1}\right)\left(1-\left(1-R-\widetilde{p}\right)^{H-h}\right)\left\|\widehat{\pi}_h(\cdot|0)-{\pi}_h^{\star,\theta}(\cdot|0)\right\|_1\nonumber\\
&\ge 0,
\end{align}
where (i) uses the definition of $\widetilde{p}$ and $\widetilde{z}_h^{\widehat{\pi}}$ in \eqref{eq:def-tilde-p} and \eqref{eq:def-tilde-z}.
Below we consider the case of $h\le 31H/32$, which gives
\begin{align*}
(1-R-\widetilde{p})^{H-h}\le \left(1-R-\widetilde{p}\right)^{\frac{H}{32}} \le (1-p)^{\frac{H}{32}} \overset{(i)}{\le} 1 - \frac{c}{32},
\end{align*}
where (i) holds because $p\ge \frac{c}{H}\ge 1-(1-\frac{c}{32})^{\frac{32}{H}}$ for $c\le 3/4$.
Thus we have
\begin{align}\label{eq:proof-thm-lower-temp-9}
\widetilde{V}_{h}(0) - \widetilde{V}_{h-1}(0) 
\ge \frac{c(1-R)(p-q)}{64(R+\widetilde{p})^2}\left(R+\widetilde{p}\left(1-R-\widetilde{p}\right)^{h-1}\right)\left\|\widehat{\pi}_h(\cdot|0)-{\pi}_h^{\star,\theta}(\cdot|0)\right\|_1.
\end{align}
We shall proceed with a case-by-case analysis as below.
\begin{itemize}
\item Consider the case $R+\widetilde{p}\ge \frac{7}{8(h-1)}$.
% Notice that if $R>\frac{1}{H}$, then we have
% \begin{align*}
% R+\widetilde{p} = R + c(1-R)\max\left\{\frac{1}{H},R\right\} = R + c(1-R)R \overset{(i)}{\le} \frac{7}{4}R,
% \end{align*}
% where (i) uses the fact that $c\le 3/4$.
% Thus we have
% \begin{align}\label{eq:proof-thm-lower-temp-7}
% R \ge \frac{4}{7}\left(R+\widetilde{p}\right) \ge \frac{49}{32(h-1)}\ge \frac{49}{32H}.
% \end{align}
Notice that if $R\le \frac{1}{H}$, then we have
\begin{align*}
R+\widetilde{p} = R + c(1-R)\max\left\{\frac{1}{H},R\right\} = R + \frac{c(1-R)}{H} \overset{(i)}{\le} R + \frac{3}{4H},
\end{align*}
where (i) uses the fact that $c\le 3/4$.
This can be rewritten as
\begin{align}\label{eq:proof-thm-lower-temp-8}
R \ge R+\widetilde{p} - \frac{3}{4H}\ge \frac{7}{8H} - \frac{3}{4H} = \frac{1}{8H}.
\end{align}
While if $R>\frac{1}{H}$, then $R\ge\frac{1}{8H}$ obviously.
% Combining \eqref{eq:proof-thm-lower-temp-7} and \eqref{eq:proof-thm-lower-temp-8}, we have
% $$
% R\ge \frac{1}{8H}.
% $$
Inserting into the definition of $p$, we have
$$
p=c\max\left\{\frac{1}{H},R\right\}\le c\max\left\{8R,R\right\} = 8cR.
$$
This upper bound on $p$ implies an upper bound on the following term.
\begin{align*}
&\quad \frac{(1-R)(p-q)}{(R+\widetilde{p})^2}\left(R+\widetilde{p}\left(1-R-\widetilde{p}\right)^{h-1}\right)\nonumber\\
&\ge \frac{(1-R)(p-q)R}{(R+\widetilde{p})^2}\ge\frac{(1-R)(p-q)R}{(R+8c(1-R)R)^2}\nonumber\\
&=\frac{(1-R)(p-q)}{(1+8c(1-R))^2R}\ge \frac{(1-R)(p-q)}{49R}\ge \frac{(1-R)(p-q)}{49\max\{\frac{1}{H},R\}}.
\end{align*}
\item Consider the case $R+\widetilde{p}< \frac{7}{8(h-1)}$.
In this case, we have
\begin{align*}
\left(1-R-\widetilde{p}\right)^{h-1} \ge 1-(R+\widetilde{p})(h-1) \ge 1-\frac{7}{8} = \frac{1}{8}.
\end{align*}
We further have
\begin{align*}
&\quad\frac{(1-R)(p-q)}{(R+\widetilde{p})^2}\left(R+\widetilde{p}\left(1-R-\widetilde{p}\right)^{h-1}\right)\nonumber\\
&\ge \frac{(1-R)(p-q)}{(R+\widetilde{p})^2}\left(R+\frac{\widetilde{p}}{8}\right)\ge \frac{(1-R)(p-q)(R+\widetilde{p})}{8(R+\widetilde{p})^2}\nonumber\\
&=\frac{(1-R)(p-q)}{8(R+\widetilde{p})}\ge\frac{(1-R)(p-q)}{8(1+c)\max\{\frac{1}{H},R\}}\ge\frac{(1-R)(p-q)}{14\max\{\frac{1}{H},R\}}.
\end{align*}
\end{itemize}
Summarizing the two cases above, we obtain the following inequality:
\begin{align*}
\frac{(1-R)(p-q)}{(R+\widetilde{p})^2}\left(R+\widetilde{p}\left(1-R-\widetilde{p}\right)^{h-1}\right)
\ge \frac{(1-R)(p-q)}{50}\min\left\{H,\frac{1}{R}\right\}.
\end{align*}
Inserting into \eqref{eq:proof-thm-lower-temp-9}, we have
\begin{align*}
\widetilde{V}_{h}(0) - \widetilde{V}_{h-1}(0) 
\ge\frac{c(1-R)(p-q)}{3200}\min\left\{H,\frac{1}{R}\right\} \left\|\widehat{\pi}_h(\cdot|0)-{\pi}_h^{\star,\theta}(\cdot|0)\right\|_1,\quad h\le \frac{31H}{32}.
\end{align*}
Combining with \eqref{eq:proof-thm-lower-temp-10}, we have
\begin{align}
V_1^{\star,R}(0) - V_1^{\widehat{\pi},R}(0) 
&\overset{(i)}{\ge} \sum_{h=1}^{\frac{31H}{32}}\left(\widetilde{V}_{h} - \widetilde{V}_{h-1}\right) 
\ge \frac{c(1-R)(p-q)}{3200}\min\left\{H,\frac{1}{R}\right\} \sum_{h=1}^{\frac{31H}{32}}\left\|\widehat{\pi}_h(\cdot|0)-{\pi}_h^{\star,\theta}(\cdot|0)\right\|_1\nonumber\\
&\overset{(ii)}{\ge}\frac{c(1-R)(p-q)H}{51200}\min\left\{H,\frac{1}{R}\right\}\overset{(ii)}{\ge} \varepsilon.
\end{align}
where (i) holds because $\widetilde{V}_{h} - \widetilde{V}_{h-1}\ge 0$ for any $h$ (c.f. \eqref{eq:proof-thm-lower-temp-11}), (ii) holds because $\sum_{h=1}^{\frac{31H}{32}}\left\|\widehat{\pi}_h(\cdot|0)-{\pi}_h^{\star,\theta}(\cdot|0)\right\|_1 \ge \left\|\widehat{\pi}_h(\cdot|0)-{\pi}_h^{\star,\theta}(\cdot|0)\right\|_1 - \frac{H}{16} \ge \frac{H}{16}$,
and (iii) holds as long as $c_1 \ge \frac{51200}{c(1-c_R)}$.

\paragraph{Proof of \eqref{eq:proof-thm-lower-temp-6}.}

It suffices to prove that
\begin{align*}
\prod_{j=1}^{h-1} A_j^{\star}\left(A_h^{\star}-A_h^{\widehat{\pi}}\right)\left(\sum_{h'=h+1}^{H}\prod_{j=h+1}^{h'-1}A_j^{\widehat{\pi}}r\right)
\ge \prod_{j=1}^{h-1} A_j^{\star}\left(A_h^{\star}-A_h^{\widehat{\pi}}\right)\left(\sum_{h'=h+1}^{H}\prod_{j=h+1}^{h'-1}A_j^{\star}r\right).
\end{align*}
We remind that all entries in matrix $A_j^{\star}$ is positive.
Thus it suffices to prove that
\begin{align}\label{eq:proof-thm-lower-temp-14}
\left(A_h^{\star}-A_h^{\widehat{\pi}}\right)\left(\sum_{h'=h+1}^{H}\prod_{j=h+1}^{h'-1}A_j^{\widehat{\pi}}r\right)
\ge \left(A_h^{\star}-A_h^{\widehat{\pi}}\right)\left(\sum_{h'=h+1}^{H}\prod_{j=h+1}^{h'-1}A_j^{\star}r\right).
\end{align}
To proceed, we remind that
\begin{align}\label{eq:proof-thm-lower-temp-15}
A_h^{\star}-A_h^{\widehat{\pi}}=
(\widetilde{p} - \widetilde{z}_h^{\widehat{\pi}})\left[
\begin{array}{cc}
-1 & 1\\
0 & 0
\end{array}
\right].
\end{align}
For convenience, we introduce two new notations:
\begin{align*}
x_{k} = \sum_{h'=k}^{H}\prod_{j=k}^{h'-1}A_j^{\star}r,\quad \widehat{x}_{k} = \sum_{h'=k}^{H}\prod_{j=k}^{h'-1}A_j^{\widehat{\pi}}r,\quad h+1\le k\le H.
\end{align*}
One can easily check that $x_k$ and $\widehat{x}_k$ satisfy the following recursive formula:
\begin{align}\label{eq:proof-thm-lower-temp-12}
x_k = r + A_k^{\star}x_{k+1},\quad \widehat{x}_k = r + A_k^{\widehat{\pi}}\widehat{x}_{k+1},\quad x_{H} = r, \quad \widehat{x}_H = r.
\end{align}
Combining with \eqref{eq:proof-thm-lower-temp-15}, proving \eqref{eq:proof-thm-lower-temp-14} is equivalent to showing that
\begin{align}\label{eq:proof-thm-lower-temp-13}
\widehat{x}_k(1) - \widehat{x}_k(0) \ge x_k(1) - x_k(0)\ge 0,\quad \forall h+1\le k\le H.
\end{align}
We shall establish this by mathematical induction.
Notice that \eqref{eq:proof-thm-lower-temp-13} holds for $k=H$ due to \eqref{eq:proof-thm-lower-temp-12}.
Assume that $\widehat{x}_{k+1}(1) - \widehat{x}_{k+1}(0) \ge x_{k+1}(1) - x_{k+1}(0)$. 
According to \eqref{eq:proof-thm-lower-temp-12}, we have
\begin{align*}
\widehat{x}_{k}(1) - \widehat{x}_{k}(0) &= 1 + R\widehat{x}_{k+1}(0) + (1-R)\widehat{x}_{k+1}(1) - (1-\widetilde{z}_k^{\widehat{\pi}})\widehat{x}_{k+1}(0) - \widetilde{z}_k^{\widehat{\pi}}\widehat{x}_{k+1}(1)\nonumber\\
&=1 + (1-R-\widetilde{z}_k^{\widehat{\pi}})(\widehat{x}_{k+1}(1) - \widehat{x}_{k+1}(0))\nonumber\\
&\ge 1 + (1-R-\widetilde{p})({x}_{k+1}(1) - {x}_{k+1}(0)).
\end{align*}
Moreover, by using \eqref{eq:proof-thm-lower-temp-13}, we have
\begin{align*}
{x}_{k}(1) - {x}_{k}(0) &= 1 + R{x}_{k+1}(0) + (1-R){x}_{k+1}(1) - (1-\widetilde{p}){x}_{k+1}(0) - \widetilde{p}{x}_{k+1}(1)\nonumber\\
&=1 + (1-R-\widetilde{p})({x}_{k+1}(1) - {x}_{k+1}(0))\ge 0.
\end{align*}
In comparison, it is obvious that
$
\widehat{x}_{k}(1) - \widehat{x}_{k}(0)\ge{x}_{k}(1) - {x}_{k}(0).
$
Thus we complete the proof.
% this connection established between the policy $\widehat{\pi}$ and its sub-optimality gap, we can now proceed to build an estimate for $\theta$. Here, we denote $\mathbb{P}_\theta$ as the probability distribution when the MDP is $\mathcal{M}_\theta$, where $\theta$ can take on values in the set $\Theta$.

\bibliographystyle{iclr2025_conference}
\bibliography{refs}

\end{document}